\documentclass{article} 
\def\printOption{final} 

\def\actualcodeurl{https://github.com/MetaDialog-Research/PBRC}
\def\printFinal{final}
\def\printReview{review}
\def\printArxiv{arxiv}
\usepackage{iclr2025_conference,times}

\ifx \printOption \printFinal
	\iclrfinalcopy
\else
	\ifx \printOption \printReview
            \relax
	\else
            \iclrfinalcopy
	\fi
\fi

\usepackage{bm}
\usepackage{array}
\usepackage{adjustbox}
\usepackage{hyperref}
\usepackage{url}
\usepackage{graphicx} 
\usepackage{subcaption}
\usepackage{float} 
\usepackage{tikz}
\usepackage{enumitem}
\usepackage{algorithm}
\usepackage{algpseudocode}
\usepackage{listings}
\usepackage{color}
\usepackage{wrapfig}
\definecolor{lightblue}{rgb}{0.8, 0.9, 1.0}
\definecolor{lightgray}{rgb}{0.9,0.9,0.9}
\usepackage{titletoc}
\newcommand\DoToC{%
  \startcontents
  \printcontents{}{2}{\textbf{APPENDIX}\vskip3pt\hrule\vskip5pt}
  \vskip3pt\hrule\vskip5pt
}

\ifx \printOption \printFinal
	\def\codeurl{\actualcodeurl} 
        \def\infoheader{Actual Paper Title}
\else
       \def\infoheader{A Preprint.}
	\ifx \printOption \printReview
		\def\codeurl{\anonymouscodeurl} 
	\else
		\def\codeurl{\actualcodeurl} 
	\fi
\fi

\usepackage{amsmath, amssymb, amsthm}
\newtheorem{theorem}{Theorem}
\newtheorem{lemma}[theorem]{Lemma}

\newtheorem{proposition}[theorem]{Proposition}

\def\eqref#1{equation~\ref{#1}}

\def\1{\bm{1}}

\def\vs{{\bm{s}}}

\def\vx{{\bm{x}}}
\def\vy{{\bm{y}}}
\def\vz{{\bm{z}}}

\def\mI{{\bm{I}}}

\def\mP{{\bm{P}}}
\def\mQ{{\bm{Q}}}

\def\mY{{\bm{Y}}}

\DeclareMathAlphabet{\mathsfit}{\encodingdefault}{\sfdefault}{m}{sl}
\SetMathAlphabet{\mathsfit}{bold}{\encodingdefault}{\sfdefault}{bx}{n}

\title{Efficient Perplexity Bound and Ratio Matching in Discrete Diffusion Language Models}

\def\etrit{Etrit~Haxholli}
\def\yeti{Yeti~Z.~Gurbuz}
\def\ogul{Ogul~Can}
\def\eli{Eli~Waxman}

\renewcommand\footnotemark{}

\def\instmetadialog{MetaDialog Research}

\author{\etrit\qquad \yeti \qquad \ogul \qquad  \eli  \thanks{Code is available at: \href{\codeurl}{\codeurl}}\\
\instmetadialog\\
\texttt{\{etrith, yeti, ogulc, elib\}@metadialog.com}
}

\begin{document}

\maketitle

\ifx \printOption \printArxiv

\fancyhead{}
\fancyhead[L]{\infoheader}   
\pagestyle{fancy}

\else
\relax

\fi

\begin{abstract}
While continuous diffusion models excel in modeling continuous distributions, their application to categorical data has been less effective. Recent work has shown that ratio-matching through \emph{score-entropy} within a continuous-time discrete Markov chain (CTMC) framework serves as a competitive alternative to autoregressive models in language modeling.
To enhance this framework, we first introduce three new theorems concerning the KL divergence between the data and learned distribution. Our results serve as the discrete counterpart to those established for continuous diffusion models and allow us to derive an improved upper bound of the perplexity. Second, we empirically show that ratio-matching performed by minimizing the \emph{denoising cross-entropy} between the clean and corrupted data enables models to outperform those utilizing score-entropy with up to 10\% lower perplexity/generative-perplexity, and 15\% faster training steps.
 To further support our findings, we introduce and evaluate a novel CTMC transition-rate matrix that allows prediction refinement, and derive the analytic expression for its matrix exponential which facilitates the computation of conditional ratios thus enabling efficient training and generation.
\end{abstract}

\section{Introduction}
Modeling data distributions is a fundamental task in machine learning. In the case of continuous data distributions, recent advancements in continuous diffusion models \citep{orig_diff, DBLP:journals/corr/abs-2006-11239, DBLP:journals/corr/abs-1905-07088} have demonstrated impressive capabilities in generating data samples and performing density estimation \citep{NEURIPS2021_0a9fdbb1, song2021scorebased, haxholli2023faster, NEURIPS2021_b578f2a5}. Despite these achievements, the application of such models to categorical data distributions, like language, remains limited, as continuous diffusion models generally underperform compared to autoregressive models in these scenarios \citep{chen2022analog, gulrajani2024likelihood, li2022diffusion, dieleman2022continuous, strudel2022self}. To address this, recent research has focused on the development of discrete diffusion models \citep{austin2021structured, campbell2022continuous, meng2022concrete, lou2023discrete, sahoo2024simple, shi2024simplified, ou2025your} which offer distinct advantages compared to autoregressive models, such as the ability to infill various parts of a sequence non-sequentially and have the potential to reduce computing time and expenses in generating lengthy sequences \citep{deschenaux2024promises, christopher2024speculative}.

Evaluating discrete diffusion models, however, presents a practical challenge due to the difficulty in calculating the perplexity, unlike in autoregressive models where this computation is straightforward. Although a recent perplexity bound has been proposed by \citet{lou2023discrete}, no tightness guarantees exist. In this paper, we present three theorems concerning the Kullback-Leibler (KL) divergence between the data distribution and the learned distribution in discrete diffusion models \citep{lou2023discrete}. These results serve as the discrete analogue of the continuous diffusion theorems provided in \citet{song2021maximum}. One of our key contributions is Theorem \ref{theorem4}, which provides an upper bound ($J_2$) on the cross-entropy between the data and learned distributions, offering a more direct way of bounding the perplexity. This bound is computationally more efficient than the existing bound in \citet{lou2023discrete}, and empirical results suggest that it is also slightly tighter.

In addition, inspired by model reparametrizations \citep{DBLP:journals/corr/abs-2006-11239, karras2022elucidating, lou2023discrete}, this paper examines the ratio-matching training objective SEDD in \citet{lou2023discrete}. We highlight that the sole unknowns implicitly learned by the model are the per-token marginal probabilities across the vocabulary, conditioned on the current perturbed sequence. Consequently, rather than modeling the ratios directly,  we employ a weighted version of the denoising cross-entropy loss $L_{ll}$ proposed in \citet{campbell2022continuous}, which also mirrors the cross-entropy loss utilized in continuous diffusion models, as described in \citet{dieleman2022continuous}. We show empirically that by modifying the reconstruction of the scores, training with cross-entropy outperforms direct ratio matching for all types of tested discrete diffusion dynamics. We name this strategy of using cross-entropy for training and the adjusted ratio reconstruction for generation, cross-entropy discrete diffusion (CEDD). Similar advantageous results when utilizing the cross-entropy loss have been reported in studies of \emph{absorb} discrete diffusion, as shown in concurrent research by \citet{sahoo2024simple, shi2024simplified, ou2025your}. These particular results, up to a scaling factor, are specific cases within the broader CEDD framework. CEDD improves upon SEDD by circumventing the learning of conditional ratios—which are analytically determinable— focusing instead on learning the mixing weights that constitute the necessary marginal ratios for generation. This focus not only conserves modeling resources but is also particularly advantageous when the distribution of conditional ratios is complex.

To illustrate our point, we design a new transition-rate matrix named \emph{roulette} diffusion, and derive its matrix exponential. The roulette diffusion is an interpolation between the absorb and the uniform diffusion. In the forward process, a token can transition to any state until it hits the absorb state, that is, until it is masked. In return, the reverse process begins with a sequence of masked tokens, which are gradually unmasked, and where the unmasked tokens can be refined. Intuitively. this capability should be important for discrete diffusion models \citep{deschenaux2024promises}, and can be useful for downstream tasks, as shown in our spelling correction experiments. Moreover, during the reverse process, the scores/ratios of unmasked tokens have much larger magnitudes than those of masked tokens, posing a significant learning challenge for the network due to output scale variability. Employing the CEDD strategy mitigates this challenge as demonstrated experimentally.

In summary, the main \textbf{contributions} of this paper include:
\begin{itemize}[leftmargin=10pt]
    \item Providing 3 new theorems concerning the KL divergence and cross entropy between the data and the learned distribution. Improving model evaluation through the bound provided in Theorem \ref{theorem4}.
    \item Introducing a new transition-rate matrix (roulette diffusion) that allows token correction after unmasking in the reverse process. Deriving its matrix exponential which enables efficient training using SEDD and generation when CEDD is employed.

    \item Comparing the performance of SEDD and CEDD experimentally in the task of language modelling on absorb, uniform and roulette diffusion models. Showing that CEDD outperforms SEDD in all cases in terms of perplexity.
\end{itemize}

\section{Preliminaries and Notation}
\subsection{Markov Chains Over Finite-State Spaces}
\label{gen_inst}

A discrete-time Markov Chain in a finite-state space is a stochastic process $X_1, X_2, \ldots, X_{\bar{T}}$, where each state $X_t$ depends solely on the preceding one. The states $X_t$ can take values from $\{1, 2, \ldots, S\}$, and $\bar{T}$ represents the number of time steps. The probability of being in state $x$ at time $t$ is
\begin{equation}\label{first}
    p_t(X_t=x)= \sum_{y=1}^S p_t(X_t=x, X_{t-1}=y)=\sum_{y=1}^Sp_{t|t-1}(X_t=x|X_{t-1}=y)p_{t-1}(X_{t-1}=y).
\end{equation}
Placing all such probabilities $p_t(X_t=x)$ in a vector $\vs_t$, such that $\vs_t(x)=p_t(X_t=x)$, gives
\begin{equation}\label{second}
    \vs_t= \mP\vs_{t-1}, \text{ where } \mP(x, y) = p_{t|t-1}(X_t=x|X_{t-1}=y).
\end{equation}

One can generalize such processes into Continuous Time Markov Chains (CTMCs) where $t\in[0, \bar{T}]$, \citep{anderson2012continuous}. For simplicity, we make the choice $\bar{T}=1$. To construct a CTMC, one first chooses a transition-rate matrix $\mQ_t$, which has the property that its non-diagonal elements are non-negative, and the elements in each of its columns add to zero \citep{suhov2008probability}.  Given an initial probability distribution $\vs_0$, the equation below fully determines the evolution of the probability with respect to time: 
\begin{equation}\label{ode}
    \frac{d\vs_t}{dt}=\mQ_t\vs_t.
\end{equation}
In addition, we choose $\mQ_t=\sigma^{'}(t) \mQ$, where $\mQ$ is itself a constant transition-rate matrix and where function $\sigma$ is monotonically increasing, and satisfies $\sigma(0)=0$ as well as $\lim_{t\rightarrow 1}\sigma(t)=T$. In this setting, the distribution over states at time $t$ is the solution of the linear ODE in Equation (\ref{ode}), that is, $\vs_t=e^{\sigma(t)\mQ}\vs_0$.
\newline
 Matrices $\mQ_t$ are chosen such that: a) the matrix exponential $e^{\sigma(t)\mQ}$ is easy to calculate, which is essential as $p_{t|0}(x|y)=e^{\sigma(t)\mQ}(x,y)$; and b) $\vs_1$ is an easy reference distribution to sample from \citep{austin2021structured, campbell2022continuous}.
 \newline
Finally, similar to diffusion processes in continuous spaces, the continuous-time Markov chain in Equation (\ref{ode}) also admits a reverse process \citep{kelly1979reversibility, sun2022score}:
\begin{equation}\label{back_ode}
    \frac{d\vs_{1-t}}{dt}=\bar{\mQ}_{1-t}\vs_{1-t},
\end{equation}
where $\bar{\mQ}_t(x,y)={\mQ}_t(y,x)\frac{p_t(x)}{p_t(y)}$ for $x\neq y$, and $\bar{\mQ}_t(x,x)=-\sum_{y\neq x}\bar{\mQ}_t(y,x)$. Since we can easily sample from the reference distribution, the only unknowns preventing us from being able to run backwards are the ratios $\frac{p_t(x)}{p_t(y)}$ also known as concrete scores \citep{meng2022concrete, lou2023discrete}. Once such ratios are modeled using a neural network, we can generate samples from the learned data distribution $p_0^\theta$ by discretizing Equation (\ref{back_ode}) as follows:
\begin{equation}\label{gen_euler}
    p(x_{t-\epsilon} = y \mid x_t = x) = \delta_{x}(y) + \bar{\mQ}_t(y, x) \epsilon + O(\epsilon^2).
\end{equation}
Additional details are provided in Appendix \ref{AppendixD}.
\subsection{SEDD: Estimating the Ratios via Score Entropy }
As pointed out in the previous subsection, we wish to model the ratios $\frac{p_t(y)}{p_t(x_t)}$ via a neural network $s_{\theta}(x_t, t)_y$, for example by minimizing the score entropy loss \citep{lou2023discrete}:
\begin{equation}\label{se_loss}
\mathbb{E}_{x_t \sim p_t} \sum_{y \neq x_t} w_{x_t, y}  \ell\left(\frac{p_{t}(y)}{p_{t}(x_t )}, s_\theta(x_t , t)_y\right), \text{ for } \ell(a, b) = 
 \left(
b-a \log b + K (a)
\right),
\end{equation}
and $K(a)= a(\log{a}-1)$. In \citet{lou2023discrete}, $w_{x_t, y}= \mQ_t(x_t , y) $, and furthermore they show that an equivalent loss is the following:
\begin{equation}\label{cse_loss}
\mathbb{E}_{x_0 \sim p_0, x_t  \sim p_{t|0}(\cdot | x_0)} \sum_{y \neq x_t } w_{x_t, y} \ell\left(\frac{p_{t|0}(y|x_0)}{p_{t|0}(x_t |x_0)}, s_\theta(x_t , t)_y\right),
\end{equation}
which side-steps the problem of not knowing the marginal ratios $\frac{p_t(y)}{p_t(x_t )}$, by employing \mbox{$p_{t|0}(i|j)=e^{\sigma(t)\mQ}(i,j)$}. A more detailed derivation of Equation (\ref{cse_loss}) can be found in Appendix \ref{rederive}.

\subsection{Discrete Diffusion for Language Modeling - Token Level Transitions}\label{difflanguage}

In the case of Language Modeling, we write a sequence of length $L$ from the data distribution as $\vx_0$, where $\vx_0=(x_0^1, x_0^2, ..., x_0^L)$ and $x_0^i \in \text{Vocab}=\{1,2,...,V\}$. The number of possible sequences, that is, the number of states $S$ is $V^L$. Unfortunately, this implies that it is not computationally feasible to model the ratios of probabilities between the current state $\vx_t$ and all other states $\vy$, since the output of our neural network would have to be $V^L$ dimensional \citep{campbell2022continuous}.

We follow the usual approach \citep{campbell2022continuous, lou2023discrete} to mitigate this issue, which is to select a sparse matrix $\mQ_t(S\times S)$, such that each entry $\mQ_t(\vx,\vy)$ for two sequences $\vx,\vy$ that differ in more than one token will be zero. The forward process that such a $\mQ_t$ defines, can equivalently described as follows: at each discretized step, only one uniformly randomly chosen token from the current sequence can be modified, according to a token level forward diffusion process \mbox{$\mQ^{tok}_t (V\times V)$}. More formally, for $\vx=(x_0^1, ...,x^i, ..., x_0^L)$ and $\vy=(x_0^1, ...,\hat{x}^i, ..., x_0^L)$, if $x^i\neq \hat{x}^i$ we have $\mQ_t(\vx, \vy)=\mQ^{tok}_t(\vx^i, \vy^i)$, otherwise $\mQ_t$ is zero in other non-diagonal entries. For such a sparse choice of $\mQ_t$, and $\vy$ which only differs from $\vx$ at a single position $i$, Expression (\ref{gen_euler})  becomes
\begin{equation}
p(x_{t-\epsilon}=\vy \mid x_t = \vx) =
\begin{cases} 
1-\sum_{z^i\in\text{Vocab}\setminus x^i}\mQ^{tok}_t(x^i, z^i)\frac{p_t(\vz)}{p_t(\vx)}\epsilon + O(\epsilon^2). & \text{if } \vy= \vx\\
\mQ^{tok}_t(x^i, y^i)\frac{p_t(\vy)}{p_t(\vx)} \epsilon + O(\epsilon^2). & \text{if } \vy\neq \vx\\
\end{cases}
\end{equation}

where $\vz$ denotes a sequence that is identical to $\vx$ everywhere, but position $i$. Thus, the usual approach entails only modeling the ratios between $\vx$ and `neighbours' $\vy$ which only differ from  $\vx$ by one token. The number of such neighbours is $L\times V$, that is $V$ per each of the $L$ positions, hence the output of the network is $L\times V$ coinciding with that of transformers in autoregressive language models. It should be pointed out that one can indeed use Expression (\ref{cse_loss}) for training, due to the fact that tokens are perturbed independently from one another in the forward process $p_{t|0}(\vx_t|\vx_0)=\prod_j p_{t|0}({x_t}^j|{x_0}^j)$, and thus
\begin{equation}
\frac{p_{t|0}(\vy|\vx_0)}{p_{t|0}(\vx_t|\vx_0)}=\frac{\prod p_{t|0}(y^j|{x_0}^j)}{\prod p_{t|0}({x_t}^j|{x_0}^j)}=\prod \frac{p_{t|0}(y^j|{x_0}^j)}{p_{t|0}({x_t}^j|{x_0}^j)}=\frac{p_{t|0}(y^i|{x_0}^i)}{p_{t|0}({x_t}^i|{x_0}^i)}=\frac{e^{\mQ^{tok}_t}(y^i, x^i_0)}{e^{\mQ^{tok}_t}(x^i_t, x^i_0)},
\end{equation}
Finally, the noise schedule $\sigma_t$ is typically loglinear $-\log(1-(1-\epsilon)t)$ or geometric $\sigma_{min}^{1-t}\cdot\sigma_{max}^{t}$.

\section{Methodology and Theoretical Results}
\label{headings}
In Subsection \ref{CEKL}, we provide results related to the cross entropy and the KL divergence between the data and the learned distribution in the CTMC (discrete diffusion) framework. The first three theorems therein can be considered as the discrete diffusion analog of the ones given in \citep{song2021maximum}. Importantly, Theorem \ref{theorem4} provides an upper bound ($J_2$) on the cross entropy between the data and learned distribution which can be used to bound the perplexity, and which does not depend on the function $K$ (Equation \ref{se_loss}). We emphasize that the results hold for general CTMCs, and not only in the special case of token-level transitions. From the second subsection onwards, we operate in the token-level transition framework. More precisely, in Subsection \ref{subsection3.2}, we introduce the roulette transition-rate matrix, and provide an expression for its exponential. In Subsection \ref{testingspeed}, we state Proposition \ref{proposition1}, which enables a more efficient estimation of $J_2$. In Subsection \ref{subsection3.4}, we highlight that similarly to the continuous case \citep{dieleman2022continuous}, the ratios can be modeled using $L_{ll}$ from \citet{campbell2022continuous}, and present how this approach is adapted in our experimental setup.

\subsection{Cross Entropy and KL Divergence Results}\label{CEKL}
 We begin by finding an upper bound for the KL divergence between the data and the learned distribution. The proofs are provided in Appendix \ref{threekl}.

\begin{theorem}\label{theorem1} Define a CTMC with transition matrix ${\mQ}_{t}$ that runs from time $0$ to $1$. The true reverse process defines a probability evolution $p_t$ from $p_1$ to the data distribution $p_0$, while the learned reverse process induces the evolution $p^{\theta}_t$ from the reference distribution $p^{\theta}_1=p_{r}$ to the approximation of the data distribution $p^{\theta}_0$. In this setting, the following KL divergence bound holds 
\begin{equation}
    D_{KL}(p_0||p_0^\theta)\leq \int_0^1\mathbb{E}_{x_t\sim p_t}\sum_{y\neq x_t}\mQ_t(x_t,y)\ell\left(\frac{p_{t}(y)}{p_{t}(x_t)}, s_\theta(x_t, t)_y\right)dt + D_{KL}(p_1||p_{r}).
\end{equation}
where $\ell(a, b) = 
 \left(
b-a \log b + K (a)
\right)$ and $K(a) = a(\log a - 1)$.
\end{theorem}
The following theorem provides an expression for the entropy of the data distribution. Furthermore, it provides sufficient conditions for the bound given above to become tight.

\begin{theorem}\label{theorem3} Denote the intermediate distributions at time $t$ determined by the true reverse process, and by the learned reverse process with $p_t$ and $p_t^\theta$, respectively. We can write the entropy of the data distribution $H(p_0)$ as
\begin{equation}
H(p_0) =  H(p_1) - \int_0^1  \mathbb{E}_{x_t \sim p_t} \sum_{y}\mQ_t(x_t,y) K\left(\frac{p_t(y)}{p_t(x_t)}\right)  dt.
\end{equation}
In addition, if the learned ratios ${s_\theta(x_t, t)_y}$ equal $\frac{p^\theta_t(y)}{p^\theta_t(x_t)}$ and $p_1=p_1^\theta:=p_r$, then the inequality in Theorem \ref{theorem1}, becomes an equality.
\end{theorem}

A particular case where the conditions of the theorem above hold is when $s_\theta(x_t, t)_y=\frac{p_t(y)}{p_t(x_t)}$ as then $\frac{p^\theta_t(y)}{p^\theta_t(x_t)}=\frac{p_t(y)}{p_t(x_t)}=s_\theta(x_t, t)_y$. The third theorem gives an upper bound of the negative log-likelihood at a single point. This is a central result in \citet[Theorem 3.6]{lou2023discrete}, but we restate it here for completeness, and provide an alternative, more detailed proof in Appendix \ref{threekl}. 

\begin{theorem}\label{theorem2} Let $p_0^\theta$ denote the learned distribution from which the reverse process samples. The negative log-probability of a state $x_0$ being sampled by the reverse process can be bounded from above as follows,
\begin{equation*}
    -\log p_0^\theta(x_0)\leq \int_0^1 \mathbb{E}_{x_t \sim p_{t|0}(\cdot|x_0)}  \sum_{y \neq x_t} \mQ_t(x_t, y) \ell\left(\frac{p_{t|0}(y|x_0)}{p_{t|0}(x_t|x_0)}, s_\theta(x_t, t)_y\right) dt
\end{equation*}
\begin{equation}\label{J2_pre}
+ D_{KL}(p_{1|0}(\cdot|x_0)||p_{r}).
\end{equation}
\end{theorem}
Since the noise schedule in \citep{lou2023discrete, ou2025your} is chosen such that $p_1\approx p_r$ and thus $D_{KL}(p_{1|0}||p_{r})\approx 0$, one can take the expectation $\mathbb{E}_{x_0}$ on both sides of Inequality (\ref{J2_pre}) to get a bound on the cross entropy $\frac{1}{L}H(p_0, p_0^\theta)=\mathbb{E}_{x_0}[-\frac{1}{L}\log(p_0^\theta(x_0))]$. One approach is to compute the RHS in Expression (\ref{J2_pre}) for each point $x_0$ and then average results (Appendix \ref{Naive-J_1}), which can be computationally expensive. Instead, we can divide by $L$ and take the expectation with regards to data distribution on both sides of Expression (\ref{J2_pre}) as in \citet{ou2025your}, and calculate
\begin{equation}
    J_1 = \frac{1}{L}\mathbb{E}_{t\sim U(0,1)} \mathbb{E}_{x_0 \sim p_{0}(x_0)}\mathbb{E}_{x_t \sim p_{t|0}(\cdot|x_0)} \sum_{y \neq x_t} \mQ_t(x_t, y) \ell\left(\frac{p_{t|0}(y|x_0)}{p_{t|0}(x_t|x_0)}, s_\theta(x_t, t)_y\right).
\end{equation}
Using Theorem \ref{theorem1} and \ref{theorem3}, we provide another, direct upper bound on the cross-entropy between the data and learned distributions, which evades the computation of $K$.
\begin{theorem}\label{theorem4} Under the conditions stated in Theorem \ref{theorem1}, the following inequality for the cross entropy between the data and the learned distribution holds:
\begin{equation*}
    H(p_0, p_0^\theta) \leq \int_0^1\mathbb{E}_{x_t\sim p_t}\sum_{y\neq x_t}\mQ_t(x_t,y)\bar{\ell}\left(\frac{p_t(y)}{p_t(x_t)}, s_\theta(x_t,t)_y \right) dt
\end{equation*}
\begin{equation}\label{altupperbound}
 -\int_0^1\mathbb{E}_{x_t\sim p_t}\sum_{y\neq x_t}\mQ_t(y,x_t)dt + H(p_1, p_{r}), \text{ where } \bar{\ell}(a, b) = 
 \left(
b-a \log b
\right).
\end{equation}

\end{theorem}
The second term $-\int_0^1\mathbb{E}_{x_t\sim p_t}\sum_{y\neq x_t}\mQ_t(y,x_t)dt$ and third one $H(p_1, p_r)\approx H(p_{r})$ can be analytically computed as shown in Section \ref{testingspeed}, Proposition \ref{proposition1}. \mbox{Finally, the first term can be rewritten as}
\begin{equation}
\mathbb{E}_{t\sim U(0,1)}\mathbb{E}_{x_0\sim p_0(x_0)}\mathbb{E}_{x_t\sim p_{t |0}(\cdot | x_0)} \sum_{y\neq x_t}\mQ_t(x_t,y)\bar{\ell}\left(\frac{p_{t |0}(y|x_0)}{p_{t |0}(x_t |x_0)},{s_\theta(x_t,t)_y}\right).
\end{equation}
Therefore, due to Theorem \ref{theorem4}, we instead propose to use
\begin{align}
    J_2 &= \frac{1}{L}\left[\mathbb{E}_{t\sim U(0,1)}\mathbb{E}_{x_0\sim p_0(x_0)}\mathbb{E}_{x_t\sim p_{t |0}(\cdot | x_0)}\sum_{y \neq x_t} \mQ_t(x_t, y) \bar{\ell}\left(\frac{p_{t|0}(y|x_0)}{p_{t|0}(x_t|x_0)}, s_\theta(x_t, t)_y\right) \right. \nonumber \\
    &\quad \left. + H(p_r)  -\int_0^1 \mathbb{E}_{x_t\sim p_t}\sum_{y\neq x_t}\mQ_t(y,x_t)dt\right].
\end{align}
In both cases $Perplexity=\exp(\mathbb{E}_{x_0}[-\frac{1}{L}\log(p_0^\theta(x_0))])=\exp(\frac{1}{L}H(p_0, p_0^\theta))\leq e^{J_1},\ e^{J_2}$.

\subsection{Roulette Discrete Diffusion}\label{subsection3.2}
Typically, matrices $\mQ^{tok}$ are defined as $\mQ^{tok}=\mP^{tok}-\mI$, where $\mP^{tok}$ is idempotent, since this implies that $(\mQ^{tok})^2=-\mQ^{tok}$. This last property of $\mQ^{tok}$ greatly simplifies the calculation of $e^{\sigma(t)\mQ^{tok}}$, as by using the Taylor series, 
$ e^{\sigma(t)\mQ^{tok}}=\mI+\mQ^{tok}(1-e^{-\sigma(t)})$.
Usually, the following two matrices $\mP^{tok}$ are chosen: a) $\mP^{tok}_{uniform}(V\times V)$ where each entry is set to $\frac{1}{V}$, and b) $\mP^{tok}_{absorb}(n\times n)$ in which the last row is full of ones while all other elements are $0$, where $n=V+1$.  While absorb diffusion often outperforms the uniform one in standard evaluations, the latter is more practical for some tasks like spelling correction, where refining tokens is crucial. We propose another transition-rate matrix whose exponential can be analytically calculated. To our knowledge this is the only matrix with such a property presented so far apart from the absorb and uniform ones. We refer to this new discrete diffusion process as \emph{roulette} diffusion. The corresponding $\mP^{tok}_{roulette}(n\times n)$ is a matrix, such that $\mP^{tok}_{roulette}(i\neq n,j\neq n)=\frac{1}{V}(1-p_m)$, $\mP^{tok}_{roulette}(n,j\neq n)=p_m$, $\mP^{tok}_{roulette}(i\neq n, n)=0$ and $\mP^{tok}_{roulette}(n,n)=1$. We notice that for $p_m=1$, roulette diffusion coincides  with absorb diffusion, while for $p_m=0$ it coincides with the uniform diffusion. Intuitively, a token can transit from a non-absorb state to a non-absorb state with probability $\frac{1}{V}(1-p_m)$, until it hits the absorb state (with probability $p_m$) and then remains there. While this matrix is not idempotent, one can still calculate its exponential as stated in the following proposition (proved in Appendix \ref{appendixroulettederiv}):
\begin{proposition}\label{proposition2}
    If we denote with $\mY_t$ the matrix exponential of $\sigma_t \mQ^{tok}_{roulette}=\sigma_t \left(\mP^{tok}_{roulette}-\mI\right)$, then $\mY_t(i\not\in\{j,n\}, j\neq n)=e^{-\sigma_t  p_m}\frac{1}{n-1}(1-e^{-(1-p_m)\sigma_t })$, $\mY_t(i\neq n,i\neq n)=e^{-\sigma_t  p_m}(1-\frac{n-2}{n-1}(1-e^{-(1-p_m)\sigma_t }))$, $\mY_t(n,j\neq n)=1-e^{-\sigma_t  p_m}$, $\mY_t(i\neq n, n)=0$, and $\mY_t(n, n)=1$.
\end{proposition}
The noise schedule used is the roulette-loglinear noise $-\frac{1}{p_m}\log(1-(1-\epsilon)t)$. In the reverse process, when a token is unmasked it can still be corrected with probability directly related to $p_m$ as shown in Appendix \ref{nrcorrectedtokens}. A generalization for time-evolving $p_m$ is given in Appendix \ref{EVroulette_appendix}, and the corresponding $\mQ^{tok}_{eroulette}(t)$ is named \emph{eroulette}. Therein (Proposition \ref{proposition3}), it is shown that the exponential matrix can be calculated as in the previous proposition, by substituting $p_m$ with $p_m(t)$.

\subsection{Efficient Estimation of $J_2$}\label{testingspeed}

In this subsection we provide Proposition \ref{proposition1}, which shows that the second and third term on the RHS of Expression (\ref{altupperbound}) can be computed efficiently:
\begin{proposition}\label{proposition1}
    In the case of the roulette diffusion with roulette-loglinear noise, $H(p_r)=0$ and $-\int_0^1\mathbb{E}_{\vx_t\sim p_t}\sum_{\vy\neq \vx_t}\mQ_t(\vy,\vx_t)dt=\left(1-\frac{1-p_m}{n-1} \right)\frac{L}{p_m}(\epsilon-1)$. For the absorb diffusion, we have $H(p_r)=0$ and $-\int_0^1\mathbb{E}_{\vx_t\sim p_t}\sum_{\vy\neq \vx_t}\mQ_t(\vy,\vx_t)dt=L(\epsilon-1)$. Finally, in the case of uniform diffusion, $H(p_r)=L\log(V)$ and $-\int_0^1\mathbb{E}_{\vx_t\sim p_t}\sum_{\vy\neq \vx_t}\mQ_t(\vy,\vx_t)dt=-\left(1-\frac{1}{V} \right)L\int_0^1 \sigma^{'}(t)dt$.
\end{proposition}

\subsection{Modeling Ratios via CEDD}\label{subsection3.4}

For sequences $\vx$, $\vy$ which only differ at some position $i$, we can write
 \begin{equation}
\frac{p_t(\vy)}{p_t(\vx_t)} = \Sigma_{h\in[V]}\frac{p_{t|0}(\vy^i|h)}{p_{t|0}({\vx_t}^i|h)}p^i_{0|t}(h|\vx_t),
\end{equation} where $p^i_{0|t}(\vx_0^i|\vx_t)=\sum_{\{\vx^1_0,... ,\vx^{L}_0\}\setminus \vx^i_0} p_{0|t}(\vx_0|\vx_t)$, as shown in Equation (\ref{ratioexplicitformula}), Appendix \ref{ceddappendix}. Since conditional ratios $\frac{p_{t|0}(\vy^i|h)}{p_{t|0}({\vx_t}^i|h)}$ are known, we can choose to reparametrize the score as
 \begin{equation}\label{reparam_c_to_s}
s^i_\theta(\vx_t, t)_{\vy} = \Sigma_{h\in[V]}\frac{p_{t|0}(\vy^i|h)}{p_{t|0}({\vx_t}^i|h)}f_\theta^i(\vx_t, t)[h],
\end{equation}
where $f_\theta(\vx_t, t)$, is a neural network, with a matrix output of shape $L\times V$, whose elements in each row $i$ add to $1$, that is $\sum_h f_\theta^i(\vx_t, t)[h]=1$. Intuitively, given the current perturbed sequence $\vx_t$, prediciton $f_\theta^i(\vx_t, t)[h]$ gives the probability that the pre-perturbation token at position $i$ used to be $h$.
In Appendix \ref{ceddappendix}, we explain how the loss in Expression (\ref{cse_loss}) can be optimized by minimizing the following cross-entropy loss:
 \begin{equation}
L_{ll}=-\mathbb{E}_{t \sim U(0,1)} \mathbb{E}_{\vx_0\sim p_0(\vx_0)} \mathbb{E}_{\vx_t \sim p_{t|0}(\cdot |\vx_0)} \sum_{i=1}^L w(t)\log f_\theta^i(\vx_t, t)[\vx_0^i].
\end{equation}
Thus, we can learn the ratios via $L_{ll}$ \citep{campbell2022continuous} for any type of diffusion model. This approach is analogous to the one in continuous diffusion models for language modeling used in \citet{dieleman2022continuous}, as shown in Appendix \ref{continousdiffconnection}.  In addition, in Appendix \ref{sedd_scaling_appendix}, we provide our original motivation for the reparametrization given in (\ref{reparam_c_to_s}). 

 In the uniform case, models using the direct reparametrization in (\ref{reparam_c_to_s}) underperform SEDD in terms of perplexity. Indeed, the model is too confident in its predictions when $t\rightarrow 0$  as it does not benefit form the neural network regularization due to the incorporation of the true conditional ratios. Thus, inside the ratios $\frac{p_{t|0}(\vy^i|h)}{p_{t|0}({\vx_t}^i|h)}$ we rescale $\sigma^\theta_t<0.0015$, by setting $\sigma^\theta_t=0.0015$. We also perform rescaling in the case of roulette diffusion dynamics, as when time is close to $0$ most tokens are unmasked. The same strategy is applied, as before, only to ratios corresponding to unmasked token, by rescaling $\sigma^\theta_t$ when $\sigma^\theta_t<0.5$ as follows: $\sigma^{scaled, \theta}_t = \log(1.1\sigma^\theta_t+1.1)$. Naturally, for the sake of rigor, these are also the models we employ to generate samples, whose quality is measured in terms of generative perplexity. Further, in order to evade metric hacking, we take care to not modify the $\sigma_t$ in the metrics $J_1$ and $J_2$ used to evaluate the models. Additional details and motivations for such rescalings are provided in the last paragraph of Appendix \ref{prob_to_score_appendix}. When $w(t)=1$ we refer to the method as CEDD, while when $w(t)=\log{(e+\frac{0.3}{t})}$ we use CEDD*.

\section{Experiments}
We now empirically validate the approaches and theoretical contributions presented in the previous section. In Subsection \ref{subsection4.1}, we compare the generative perplexities of models trained on OpenWebText \citep{Gokaslan2019OpenWeb} with SEDD and SEDD scaled (SEDDs, see Appendix \ref{sedd_scaling_appendix} for details) versus those trained with CEDD and CEDD*. We keep all other variables unchanged for a fair comparison, finding that CEDD outperforms SEDD in all cases. The tests are conducted for the absorb, uniform and roulette diffusion dynamics. In Subsection \ref{subsection4.2}, we evaluate the perplexity of the models, by calculating the upper bound on 5 different datasets, namely: 1BW, LAMBADA, PTB, Wikitext2 and Wikitext103 \citep{chelba2013one, paperno2016lambada, marcus_building_1993, merity2016pointer}. For the sake of reproducing the results of previous work we use $J_1$ as a metric. Finally, having computed the results using $J_1$, in Subsection \ref{subsection4.3} we also re-evaluate the models via $J_2$. Our findings suggest that the bound provided by $J_2$ is slightly tighter. In the last subsection we compare SEDD and CEDD* on a spell-checking task.
\newline
Our model employs the transformer architecture as described by \citet{lou2023discrete}, with no modifications; more details can be found in Appendix \ref{nahp}. 
The algorithms for training via SEDD, and sampling unconditionally and conditionally can be found in the Appendix of \citet{lou2023discrete}. On the other hand, we provide the algorithm for training using CEDD in Algorithm \ref{alg:entropy_training}, Appendix \ref{algorithms_appendix}. In all cases, samples are generated using tau-leaping \citep{gillespie-tau, campbell2022continuous}, which performs an update at each position simultaneously for each reverse time step. All models were trained for 400k parameter updates unless stated otherwise.

\subsection{Generative Perplexity Comparisons}\label{subsection4.1}
We compare the generative perplexities (GenPerp) of identical networks trained using SEDD, SEDDs, CEDD and CEDD*. To evaluate the generative perplexity of a model, we generate samples from that model, and use a GPT-2 large model in order to assess the likelihood of the generated samples. However, this metric can be unreliable, as models such as GPT-2 large are not perfect themselves, and they tend to assign high probability to some unlikely sequences, such as those that contain repetitive tokens. Such biased samples can be generated by increasing the step size, while maintaining the number of reverse steps. To ensure a fair comparison, we evade such approaches. Furthermore, recently \citet{zheng2024masked} showed that the sampling procedure in \citet{lou2023discrete} suffers from numerical precision issues. To address this, they proposed fixing the categorical sampling to 64-bit floating-point precision—a strategy that we have also adopted. In Appendix \ref{additional_res_appendix}, we also provide perplexity results as evaluated by LLama 3.1 8B \citep{dubey2024llama}. 
\newline
Initially, we perform a grid search to find the best $p_m$ for the roulette case. Out of the 4 values, $0.95,\ 0.65,\ 0.35,\ 0.05$, we found $p_m=0.95$ performs best (Appendix \ref{additional_res_appendix}, Table \ref{roulette_gen_tab}), but the optimum is likely reached when $p_m=1$. However, $p_m = 0.95$, enables equipping unmasking dynamics with the correction mechanism, which can be useful for some tasks like spelling correction. Therefore, in what follows, when the roulette case is concerned, it is implied that $p_m=0.95$. In Table \ref{small_perp_tab1}, we provide the results when the sequence length is 128 and the number of reverse steps is 1024, when using the analytic sampler. Generated samples can be found in Appendix \ref{gentext_samples_appendix}.

\subsection{Perplexity Comparisons}\label{subsection4.2}
In this subsection, we compare the performance of identical networks trained using SEDD, SEDDs, CEDD and CEDD*, in terms of $J_1$. We do not shuffle the test set. Results for sequence length $L=128$ are provided in Table \ref{small_perp_tab1}. Plots illustrating cummulative performance trajectories across the testing sets are provided in Appendix \ref{additional_res_appendix}.

\begin{figure}[H]
    \centering
    \begin{minipage}{0.54\textwidth}
        \centering
        \includegraphics[width=\textwidth]{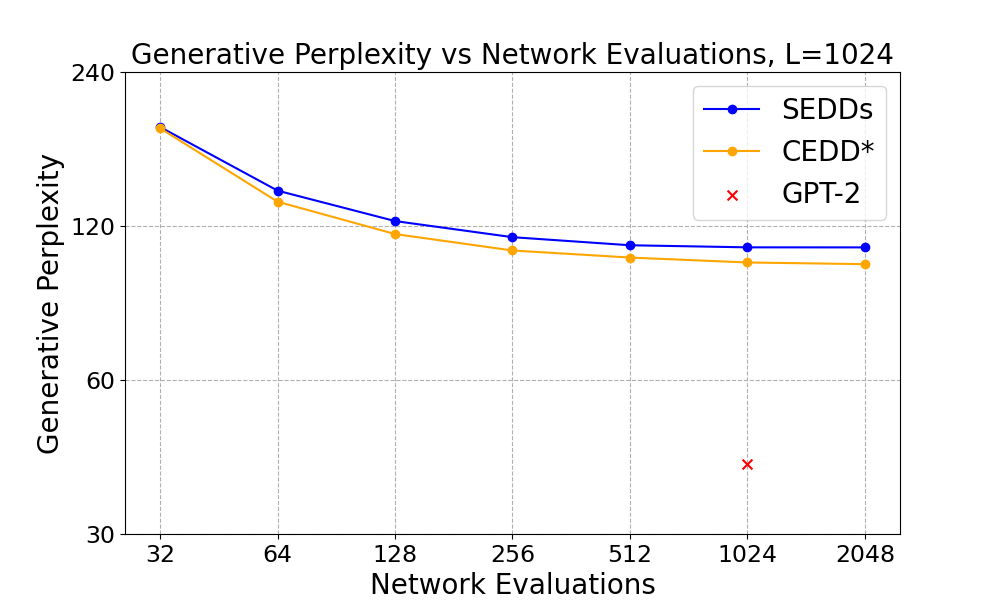}
        \caption{Scaling of Generative Perplexity vs sampling steps for SEDDs (loaded) and CEDD* absorb.}
    \end{minipage}%
    \hfill
    \begin{minipage}{0.44\textwidth}
        \centering
        \begin{adjustbox}{width=1\textwidth}
        \begin{tabular}{|>{\centering\arraybackslash}m{0.5cm}|m{7cm}|}
            \hline
            \rotatebox[origin=c]{90}{\ \ SEDDs} & \textbf{Seeing them is a treat,} because I’m coaching. It’s a full week for me. "On Saturday night, Gonzalez spoke to media and said he thinks he could work out the fine fit to make the made \\ 
            \hline
            \rotatebox[origin=c]{90}{\ \ CEDD*} & \textbf{Seeing them is a treat,} though it was only in her presence. Just being able to think about all the soldiers out here helped make the place even look like their countries and the rest of\\
            \hline
            \rotatebox[origin=c]{90}{\ \ GPT-2} & \textbf{Seeing them is a treat,} and ultimately all the practical work should be well underway for everyone to enjoy. All Milan, Italy-based staff members will be given access to the dev portal\\ 
            \hline
        \end{tabular}
        \end{adjustbox}
        \captionof{figure}{Filtered samples from SEDDs and CEDD* absorb, L=1024. The conditional part is highlighted in bold.}
    \end{minipage}
\end{figure}

\begin{table}[H]
\caption{Results comparing SEDD, SEDD scaled (SEDDs), CEDD and CEDD*. Lower is better.}
\label{small_perp_tab1}
\begin{center}
\begin{tabular}{lllllll} 
\textbf{Model (L=128)} & \textbf{GenPerp} & \textbf{LAMBADA} & \textbf{WikiText2} & \textbf{PTB}  & \textbf{WikiText103} & \textbf{1BW}\\
\hline
SEDD Absorb &  172.35  & 70.07 & 75.20 & 240.43 & 74.79 & 88.99\\
SEDDs Absorb & 166.35 & 67.05 & 69.37 & 208.69& 69.17& 83.87\\
CEDD Absorb  &  148.21 & 65.18 & 65.66 & 199.69 & 65.62& 79.83\\
CEDD* Absorb & \textbf{143.86}  & \textbf{64.60} & \textbf{65.04} & \textbf{192.99} & \textbf{64.69}& \textbf{79.81}\\
\hline
SEDD Roulette & 178.94  & 72.07 & 80.13 & 230.74& 79.68& 93.45\\
SEDDs Roulette & 172.93 & 69.10 & 74.38 & \textbf{209.12}& 74.16& 88.02\\
CEDD Roulette & 167.67 & 69.77 & 72.91 & 227.16 & 72.49 & \textbf{86.55}\\
CEDD* Roulette & \textbf{158.56} & \textbf{67.84} & \textbf{70.54} & 216.91 & \textbf{70.18}& 86.76\\
\hline
SEDD Uniform & 169.66   & 80.74 & 91.79 & 252.81 & 91.40 & 102.75\\
SEDDs Uniform & 163.88 & 81.13 & 89.21 & \textbf{228.37} & 88.56 & 100.80\\
CEDD Uniform  & \textbf{161.84}& \textbf{80.27} &  \textbf{87.91} & 279.65 & \textbf{87.46} & \textbf{99.34}\\
CEDD* Uniform & 175.42 & 82.54 & 89.68 & 289.09 & 88.90 & 106.32\\
\hline
DFM $k_t=t$ & \textbf{145.48} & \textbf{71.90} & \textbf{71.20} & 221.15 & \textbf{70.84} & \textbf{82.63}\\
DFM $k_t=t^2$ & 152.70 & 72.31 & 72.87 & \textbf{215.30} & 72.55 & 85.82 \\
\hline
\end{tabular}
\end{center}
\end{table}


We also compare our approach against models trained utilizing Discrete Flow matching (DFM) \citep{gat2024discrete}. In Table \ref{small_perp_tab1}, we present results when comparing against flows with convex interpolants, where we chose schedules $k_t=t$, as in \citet{campbell2024generative}, as well as $k_t=t^2$. The perplexity bound for these models is calculated using Expression (24) in \citet{haxholli2024minibatch}.

It can be seen that CEDD* absorb performs best overall, thus we compare this model, against CEDDT, that is, CEDD with the scaling loss used in the SOTA discrete diffusion model \citep{sahoo2024simple}. Interestingly, our scaling CEDD* outperforms that of CEDDT, despite its theoretical support with regards to the score entropy loss. The results can be found in Table \ref{ceddstar_vs_ceddt}.

We also train 3 absorb models, namely SEDDs, CEDD*, CEDDT, as well as GPT2, with a sequence length of $1024$. Results are provided in Table \ref{ceddstar_vs_ceddt}, where it can be seen that overall GPT-2 performs best. The gaps between SEDDs, CEDD*, CEDDT are reduced, likely since by seeing more tokens they all approach their optimal performance. However, models trained with CEDD/CEDD* converge faster to the optimum in terms of number of parameter updates. In Appendix \ref{additional_res_appendix}, Table \ref{seddvscedd_20k} and Figure \ref{cvs20k}, we show the difference in performance between CEDD* and SEDDs absorb during and after training for 20k parameter updates. In addition, training with CEDD (and its variants) is roughly $15\%$ faster per iteration, due to the simplified loss function. Furthermore, by incorporating the $f_\theta^i(\vx_t, t)$ to $s^i_\theta(\vx_t, t)$ scaling in the timestep, absorb models trained with CEDD/CEDD* generate sequences $2\%$ faster than those trained with SEDDs.
\begin{table}[H]
\caption{Results comparing SEDDs (retrained), CEDD*, CEDDT and GPT-2 (retrained) in terms of generative perplexity, and perplexity on 5 test sets. Number of generation steps is 1024.}
\label{ceddstar_vs_ceddt}
\begin{center}
\begin{tabular}{lllllll} 
\textbf{Model (Absorb)}  & \textbf{GenPerp} & \textbf{LAMBADA} & \textbf{WikiText2} & \textbf{PTB}  & \textbf{WikiText103} & \textbf{1BW}\\
\hline
SEDDs L=128 & 166.35 & 67.05 & 69.37 & 208.69& 69.17& 83.87\\
CEDD* L=128 & \textbf{143.86} & \textbf{64.60} & \textbf{65.04} & \textbf{192.99} & \textbf{64.69}& \textbf{79.81}\\
CEDDT L=128 & 154.04 & 68.24 & 68.61 & 204.76 & 68.10& 81.81\\
\hline
SEDDs L=1024 & 105.27 & \textbf{52.18}& 42.02 & 117.00 & 41.83 & 80.79 \\
CEDD* L=1024 & \textbf{101.83} & 52.70 & \textbf{41.57} & \textbf{115.99*} & \textbf{41.31} & \textbf{77.96}\\
CEDDT L=1024 & 108.88 & 53.20 & 42.24 & 121.05 & 42.07& 78.10\\
D3PM L=1024& - & 93.47 & 77.28 & 200.82 & 75.16 & 138.92\\
PLAID L=1024& - & 57.28 & 51.80 & 142.60 & 50.86 & 91.12\\
\hline
GPT-2 L=1024& \textbf{41.02}* & \textbf{49.02*} & \textbf{37.68*} & 134.13 & \textbf{37.55*} & \textbf{58.92*}\\
\hline
\end{tabular}
\end{center}
\end{table}

\subsection{Comparing the Two Upper Bounds}\label{subsection4.3}

Lastly, we compare the two upper bounds $J_1$ and $J_2$. The bound $J_2$ is shown empirically to be slightly tighter than $J_1$, supporting the importance of Theorem \ref{theorem4} and Proposition \ref{proposition1}. We estimate the integral with respect to time in both $J_1$ and $J_2$ by randomly sampling time points from a uniform distribution in the interval $(e^{-4}, 1-e^{-4})$, and the averaging the loss. The procedure of comparing these bounds is explained in detail in Appendix \ref{compare_proc_appendix}. We present the results for several models in Table \ref{main_bound_tab}, while we show the testing plots for CEDD* absorb (L=1024) in Figure \ref{cumulative_eval_fig}.  The proposed bound $J_2$ gives a lower bound in every single case, as it can also be seen in Table \ref{small_perp_tab1_J2}, Appendix \ref{additional_res_appendix}, where we provide the equivalent of Table \ref{small_perp_tab1} when $J_2$ is utilized. 

\begin{table}[H]
\caption{Results comparing $J_1$ and $J_2$ for the best performing models of each category.}
\label{main_bound_tab}
\begin{center}
\begin{tabular}{llllll} 
\textbf{Model/L/Perplexity-Bound}\ \ & \textbf{LAMBADA} & \textbf{WikiText2} & \textbf{PTB}  & \textbf{WikiText103} & \textbf{1BW}\\
\hline
SEDDs absorb/1024/ $\exp(J_1)$ & 52.18& 42.02 & 117.00 & 41.83 & 80.79\\
SEDDs absorb/1024/ $\exp(J_2)$ & \textbf{51.78} & \textbf{41.76} & \textbf{115.97} & \textbf{41.51} & \textbf{80.53}\\
\hline
CEDD* absorb/1024/ $\exp(J_1)$ & 52.70 & 41.57 & 115.99 & 41.31 & 77.96\\
CEDD* absorb/1024/ $\exp(J_2)$ & \textbf{52.10} & \textbf{41.18} & \textbf{115.03} & \textbf{40.98}& \textbf{77.28}\\
\hline
CEDD* absorb/128/ $\exp(J_1)$ & 64.60 & 65.04 & 192.99 & 64.69 & 79.81\\
CEDD* absorb/128/ $\exp(J_2)$ & \textbf{64.11} & \textbf{64.54} & \textbf{191.38} & \textbf{64.30}& \textbf{79.17}\\
\hline
CEDD* roulette/128/ $\exp(J_1)$ & 67.84 & 70.54 & 216.91 & 70.18 & 86.76\\
CEDD* roulette/128/ $\exp(J_2)$ & \textbf{67.27} & \textbf{69.61} & \textbf{213.90} & \textbf{69.45}& \textbf{85.64}\\
\hline
CEDD uniform/128/ $\exp(J_1)$ & 80.27 &  87.91 & 279.65 & 87.46 & 99.34\\
CEDD uniform/128/ $\exp(J_2)$ & \textbf{79.46} &  \textbf{86.82} & \textbf{276.61} & \textbf{86.52} & \textbf{98.44}\\
\hline
\end{tabular}
\end{center}
\end{table}
\begin{figure}[H]
\centering
\begin{subfigure}[b]{0.25\textwidth}
    \centering
    \includegraphics[width=\textwidth]{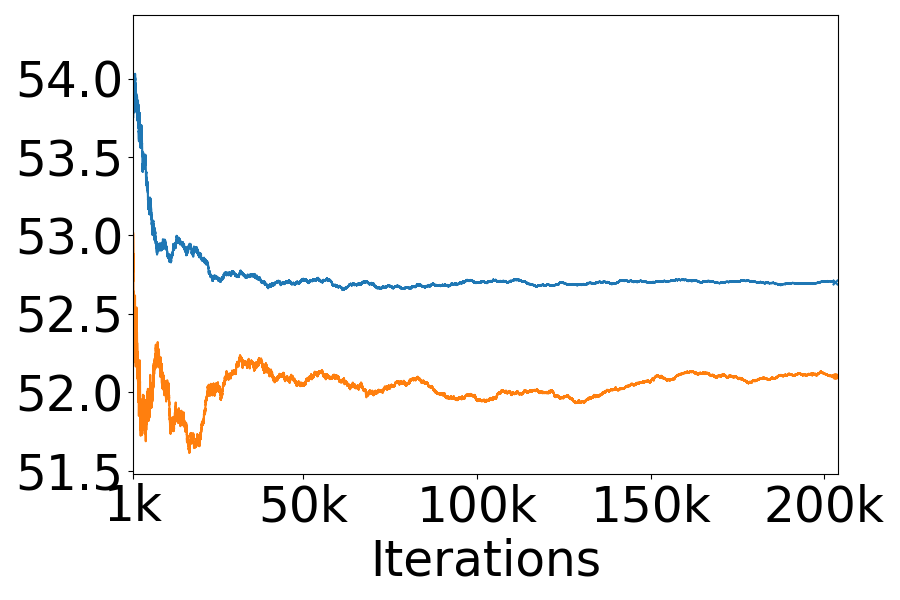}
    \caption{LAMBADA}
    \label{fig:image1}
\end{subfigure}%
\hfill%
\begin{subfigure}[b]{0.25\textwidth}
    \centering
    \includegraphics[width=\textwidth]{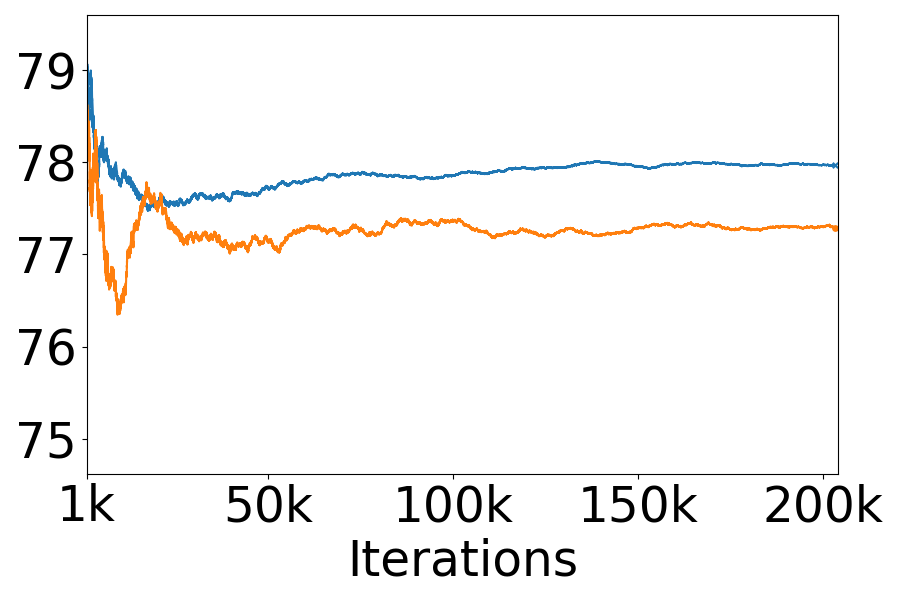}
    \caption{1BW}
    \label{fig:image2}
\end{subfigure}%
\hfill%
\begin{subfigure}[b]{0.25\textwidth}
    \centering
    \includegraphics[width=\textwidth]{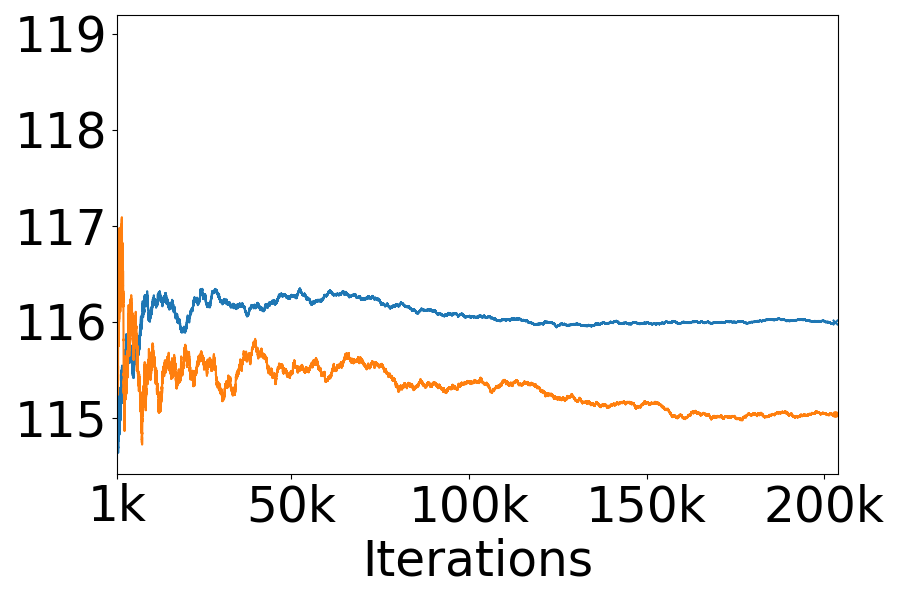}
    \caption{PTB}
    \label{fig:image3}
\end{subfigure}%
\hfill%
\begin{subfigure}[b]{0.25\textwidth}
    \centering
    \includegraphics[width=\textwidth]{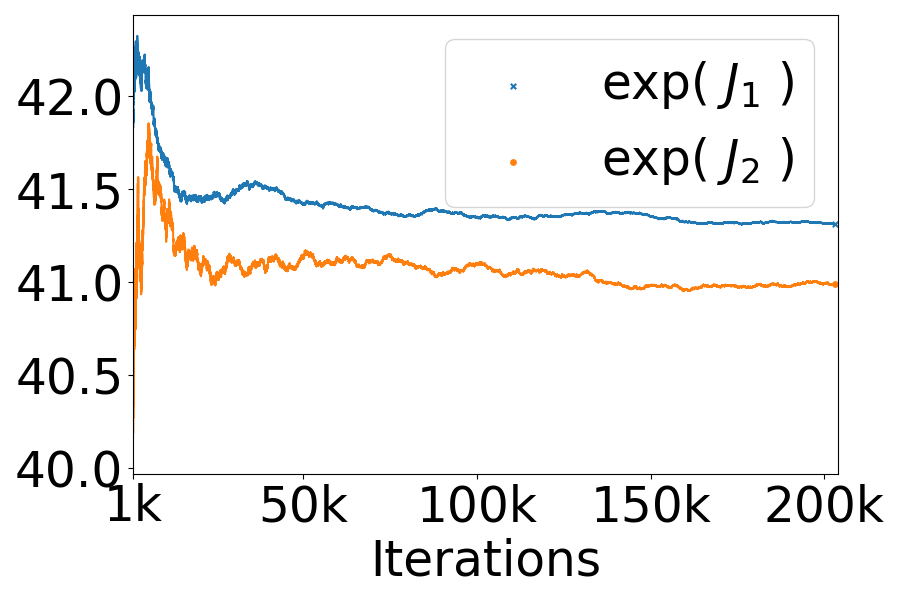}
    \caption{WikiText103}
    \label{fig:image4}
\end{subfigure}
\caption{Upper bounds $J_1$ and $J_2$ of CEDD* absorb L=1024 for different testing sets.}
\label{cumulative_eval_fig}
\end{figure}

\subsection{Spelling Correction}

We evaluate our uniform and roulette diffusion models on a character-level unsupervised spell-checking task \citep{hoogeboom2021argmax}.  We train CEDD* and SEDD models on 'War and Peace' (3.3M tokens), and test on 'Crime and Punishment' (CAP) and 'Pride and Prejudice' (PAP). The test set is contaminated with mistakes (5\% of characters), and the model chooses the most likely correction. Results when training for 25k/50k iterations are provided in Table \ref{spc} and Appendix \ref{additional_res_appendix}.

\begin{table}[H]
\caption{Correction accuracy percentages. 25k training iterations, batch size of 32 and $L=128$.}
\label{spc}
\begin{center}
\begin{tabular}{lllll} 
\textbf{Model (L=128)} & CEDD* Uniform & SEDD Uniform & CEDD* Roulette & SEDD Roulette\\
\hline
PAP & 89.5 & 86.9 & $\boldsymbol{89.7}$ & 85.1\\
CAP & 89.5 & 87.5 & $\boldsymbol{90.3}$ & 85.8\\
\hline
\end{tabular}
\end{center}
\end{table}

\section{Related Work and Future Outlook}

\textbf{The continuous diffusion}  approach has demonstrated excellent performance in modeling continuous data distributions \citep{song2020denoising, song2021scorebased, DBLP:journals/corr/abs-2107-00630, pmlr-v139-nichol21a, saharia2022image, ramesh2022hierarchical}, leading to numerous efforts to adapt it for language modeling tasks  \citep{chen2022analog, gulrajani2024likelihood, li2022diffusion, dieleman2022continuous, strudel2022self, gong2022diffuseq, mahabadi2023tess}. Although initial results were not competitive, recent advancements have reduced the performance gap with autoregressive models.
\textbf{Discrete diffusion models} offer an alternative approach for modeling categorical distributions like language data. Originally proposed by \citet{hoogeboom2021argmax, austin2021structured}, the framework was extended to continuous time by \citet{campbell2022continuous}. Both strands of work employ a combination of the variational lower bound and cross entropy loss, the latter being central to our training approach and corresponding to the strategy employed by \citet{dieleman2022continuous} in the continuous case. The cross-entropy loss is also derived in\citet{sahoo2024simple, shi2024simplified, ou2025your}, but only for the absorb transition-rate matrix.
 In addition, the cross entropy loss is similar to the loss employed in \citet{sun2022score}, however in their case, one conditions on the current sequence $\vx_t^{-i}$ without the token at position $i$, and attempts to maximize the likelihood of $\vx^i_t$. In contrast, \citet{lou2023discrete} proposed modeling ratios of probabilities \citep{meng2022concrete} directly using score entropy. 
 \newline
\textbf{The evaluation} of such models can be performed by using Theorem \ref{theorem2}, as originally derived and used in \citet{lou2023discrete}. Inspired by this result, we formulate and prove the discrete version of the rest of the theorems in \citet{song2021maximum}, which provide important information regarding the KL divergence between the data and learned distributions in CTMCs, and which justify the usage of $J_2$.
\newline
\textbf{Transition-rate matrices} are an essential component in CTMCs, as demonstrated by the performance discrepancy between the absorb and uniform matrices \citep{lou2023discrete, campbell2022continuous}. We introduced roulette matrix, an interpolation between the two and derived its exponential.
\newline
\textbf{Regarding future work}, numerous avenues remain open for optimization within the diffusion model framework such as the exploration of the Eroulette matrix. Additionally, the exploration of noise schedules in discrete diffusion models remains relatively nascent, as similar to \citet{lou2023discrete}, we did not systemically explore noise schedules. In our experiments, cross entropy loss is modulated by heuristically chosen time-dependent weighting. Investigating and establishing a general optimal weighting schedule could further refine performance metrics. Finally, studying scaling laws for discrete diffusion models and establishing their practical utility in downstream tasks is crucial.
\section{Conclusion}
In this work, we provided three new theorems concerning the KL divergence between the data and the learned distribution, improving model evaluation through the bound presented in Theorem 4. We also introduced a new transition-rate matrix that allows for token correction after unmasking in the reverse process, and derived its exponential matrix to enable efficient training/sampling.  Finally, we proposed favoring denoising cross entropy loss over score entropy for training discrete diffusion models due to the findings in the experiments we conducted.
\section*{Ethics Statement}
This paper introduces research that progresses the field of natural language generation. Beyond the pre-existing ethical concerns associated with this domain, such as e.g. bias, toxicity, and the generation of deceptive content, our methodology poses no unique risks. This is because our work is primarily theoretical and not of a magnitude that could cause any distinct issues.

\bibliography{library}
\bibliographystyle{iclr2025_conference}

\newpage
\appendix
\DoToC
\setcounter{theorem}{0}
\newpage
\section{Theoretical Results}
\subsection{KL Divergence Theorems}\label{threekl}
\begin{lemma}\label{lemma1}Let $q^{*}_{1}$ and $\hat{q}_{1}$ denote two initial distributions and define two Continuous-Time Markov Chains (CTMCs) running from time $1$ to $0$, with transition-rate matrices $\mQ^{*}_{t}(y,x)=\mQ_{t} (x,y)r(x,y)$ and $\hat{\mQ}_{t}(y,x)=\mQ_{t} (x,y)\hat{r}(x,y)$, where $r(x,y)$ and $\hat{r}(x,y)$ are chosen such that they ensure $\mQ^{*}_{t}$ and $\hat{\mQ}_{t}$ are indeed transition-rate matrices. When the first process (with matrix $\mQ^{*}_{t}(y,x)$) is applied to initial distribution $q^{*}_{1}$ and the second (with matrix $\hat{\mQ}_{t}(y,x)$) to initial distribution $\hat{q}_{1}$, they define distributions $q_t^{*}$ and $\hat{q}_t$ at time $t$, and $q_0^{*}$, $\hat{q}_0$ at the end of the diffusion, for which
\begin{equation*}
    D_{KL}(q_0^{*}||\hat{q}_0)
\end{equation*}
\begin{equation}\label{lemmaineq}
\leq \int_0^1\mathbb{E}_{x_t\sim q^{*}_t}\sum_{y\neq x_t}\mQ_{t} (x_t,y)\ell\left(r(x_t,y), \hat{r}(x_t,y)\right)dt + D_{KL}(q_1^{*}||\hat{q}_1),
\end{equation}
where $\ell(a, b) = 
 \left(
b-a \log b + K (a)
\right)$ and $K(a) = a(\log a - 1)$.
\end{lemma}
\begin{proof}
From the data processing inequality applied to the reverse of the processes defined in the Lemma, we have $D_{KL}(q_0^{*}||\hat{q}_0) \leq D_{KL}(q^{*}||\hat{q})  $ where $q^{*}$ is the measure over the space of paths generated by the first process, while $\hat{q}$ is the measure over the space of paths generated by second one. Using the expression for the KL divergence between path measures of two CTMCs found in \citep[Section 2.1]{NIPS2007_735b90b4}, and by substituting $q_t(x):=q^{*}(x_t)$, $g(x'|x)=\mQ^{*}_{t}(y,x_t)$ and $f(x'|x)=\hat{\mQ}_{t}(y,x_t)$, we have:
\begin{equation*}
D_{KL}(q^{*}||\hat{q}) = \int_0^T\mathbb{E}_{x_t\sim q^{*}_t}\sum_{y\neq x_t}\left(\mQ^{*}_{t}(y,x_t)\log{\frac{\mQ^{*}_{t}(y,x_t)}{\hat{\mQ}_{t}(y,x_t)}}+\hat{\mQ}_{t}(y,x_t)-\mQ^{*}_{t}(y,x_t)\right)dt
\end{equation*}
\begin{equation}
+ D_{KL}(q_1^{*}||\hat{q}_1).
\end{equation}
Now we focus on the expression inside the sum:
\begin{equation}
\mQ^{*}_{t}(y,x_t)\log{\frac{\mQ^{*}_{t}(y,x_t)}{\hat{\mQ}_{t}(y,x_t)}}+\hat{\mQ}_{t}(y,x_t)-\mQ^{*}_{t}(y,x_t)=
\end{equation}
\begin{equation}
\mQ_{t} (x_t,y)r(x_t,y)\log{\frac{\mQ_{t} (x_t,y)r(x_t,y)}{\mQ_{t} (x_t,y)\hat{r}(x_t,y)}}+\mQ_{t} (x_t,y)\hat{r}(x_t,y)-\mQ_{t} (x_t,y)r(x_t,y)=
\end{equation}
\begin{equation}
\mQ_{t} (x_t,y)r(x_t,y)\left(\log{r(x_t,y)}-\log{\hat{r}(x_t,y)}\right)+\mQ_{t} (x_t,y)\hat{r}(x_t,y)-\mQ_{t} (x_t,y)r(x_t,y)=
\end{equation}
\begin{equation}
\mQ_{t} (x_t,y)r(x_t,y)\left(\log{r(x_t,y)}-1\right)-\mQ_{t} (x_t,y)r(x_t,y) \log{\hat{r}(x_t,y)}+\mQ_{t}(x_t,y)\hat{r}(x_t,y)=
\end{equation}
\begin{equation}
\mQ_{t}(x_t,y)\left[K\left(r(x_t,y)\right)+\left(\hat{r}(x_t,y)-r(x_t,y)\log{\hat{r}(x_t,y)}\right)\right]=\mQ_{t}(x_t,y)\ell\left(r(x_t,y), \hat{r}(x_t,y)\right),
\end{equation}
which concludes the proof. We highlight that the CTMCs defined in the Lemma are completely arbitrary and not necessarily related to one-another, as the choices of  $r(x,y)$ and $\hat{r}(x,y)$ can completely overwrite the matrix $\mQ_{t} (x,y)$.
\end{proof}
\textbf{Theorem \ref{theorem1}.} \emph{Define a CTMC with transition matrix ${\mQ}_{t}$ that runs from time $0$ to $1$. The true reverse process defines a probability evolution $p_t$ from $p_1$ to the data distribution $p_0$, while the learned reverse process induces the evolution $p^{\theta}_t$ from the reference distribution $p^{\theta}_1=p_{r}$ to the approximation of the data distribution $p^{\theta}_0$. In this setting, the following KL divergence bound holds 
\begin{equation}
    D_{KL}(p_0||p_0^\theta)\leq \int_0^1\mathbb{E}_{x_t\sim p_t}\sum_{y\neq x_t}\mQ_t(x_t,y)\ell\left(\frac{p_{t}(y)}{p_{t}(x_t)}, s_\theta(x_t, t)_y\right)dt + D_{KL}(p_1||p_{r}).
\end{equation}}
\begin{proof}
In Lemma \ref{lemma1}, we set $r(x, y)=\frac{p_t(y)}{p_t(x)}$ and $\hat{r}(x_t,y) = s_\theta(x_t,t)_y$. This implies that $q_t^{*}=p_t$ and $\hat{q}_t=p_t^\theta$ . In particular, $q_0^{*}=p_0$, $\hat{q}_0=p_0^\theta$ and $q_1^{*}=p_1$, $\hat{q}_1=p_1^\theta:=p_{r}$. Plugging everything in Expression (\ref{lemmaineq}) gives the desired result.
\end{proof}

\textbf{Theorem \ref{theorem2}.} \emph{Let $p_0^\theta$ denote the learned distribution from which the reverse process samples. The negative log-probability of a state $x_0$ being sampled by the reverse process can be bounded from above as follows,}
\begin{equation*}
    -\log p_0^\theta(x_0)\leq \int_0^1 \mathbb{E}_{x_t \sim p_{t|0}(\cdot|x_0)}  \sum_{y \neq x_t} \mQ_t(x_t, y) \ell\left(\frac{p_{t|0}(y|x_0)}{p_{t|0}(x_t|x_0)}, s_\theta(x_t, t)_y\right) dt
\end{equation*}
\begin{equation}
+ D_{KL}(p_{1|0}(\cdot|x_0)||p_{r}).
\end{equation}

\begin{proof}
In Lemma \ref{lemma1} we set $r(x,y)=\frac{p_{t|0}(y|x_0)}{p_{t|0}(x_t|x_0)}$ and $\hat{r}(x_t,y) = s_\theta(x_t,t)_y$, where $p_{t|0}$ is the probability over states at time $t$ determined by a CTMC with transition-rate matrix $\mQ_{t}$ applied to initial distribution $\delta(x=x_0)$. 

As such, we have $q_t^{*}=p_{t|0}$ and $\hat{q}_t=p_t^\theta$. In particular, $q_0^{*}=\delta(x=x_0)$, $\hat{q}_0=p_0^\theta$ and $q_1^{*}=p_{1|0}$ , $\hat{q}_1=p_1^\theta:=p_r$. Plugging everything in Expression (\ref{lemmaineq}):
\begin{align*}
&D_{KL}(\delta(x=x_0)||p_0^\theta)\leq \int_0^1 \mathbb{E}_{x_t \sim p_{t|0}(\cdot|x_0)} \sum_{y \neq x_t} \mQ_t(x_t, y) \left[ K\left( \frac{p_{t|0}(y|x_0)}{p_{t|0}(x_t|x_0)} \right) \right. \quad\quad\quad\quad\quad\quad\quad\quad\quad\quad \\
&\quad\quad\quad\quad\quad\quad\quad\quad\quad\quad\quad+ \left. \left( s_\theta(x_t, t)_y - \frac{p_{t|0}(y|x_0)}{p_{t|0}(x_t|x_0)} \log{s_\theta(x_t, t)_y} \right) \right] dt+ D_{KL}(p_{1|0}(x_0)||p_{r}).
\end{align*}
This concludes the proof as $D_{KL}(\delta(x=x_0)||p_0^\theta) = -\log(p_0^\theta(x_0))$.
\end{proof}

\textbf{Theorem \ref{theorem3}.} \emph{Denote the intermediate distributions at time $t$ determined by the true reverse process, and by the learned reverse process with $p_t$ and $p_t^\theta$, respectively. We can write the entropy of the data distribution $H(p_0)$ as
\begin{equation}
H(p_0) =  H(p_1) - \int_0^1  \mathbb{E}_{x_t \sim p_t} \sum_{y}\mQ_t(x_t,y) K\left(\frac{p_t(y)}{p_t(x_t)}\right)  dt.
\end{equation}
In addition, if the learned ratios ${s_\theta(x_t, t)_y}$ equal $\frac{p^\theta_t(y)}{p^\theta_t(x_t)}$ and $p_1=p_1^\theta:=p_r$, then the inequality in Theorem \ref{theorem1}, becomes an equality.}
\begin{proof}
The cross entropy between the true data distribution, and the modeled data distribution satisfies the following: 
\begin{equation}
H(p_0, p_0^\theta) - H(p_1, p_1^\theta) =  \int_1^0 \frac{\partial}{\partial t} H(p_t, p_t^\theta)dt
\end{equation}
To keep things clear we focus on $\frac{\partial}{\partial t} H(p_t, p_t^\theta)$.
\begin{equation}
\frac{\partial}{\partial t} H(p_t, p_t^\theta)=\frac{\partial}{\partial t} \int -p_t(x_t)\log(p^\theta_t(x_t))dx_t= \int -\frac{\partial}{\partial t} [p_t(x_t)\log(p^\theta_t(x_t))]dx_t
\end{equation}
\begin{equation}
 = - \left( \int \frac{\partial p_t(x_t)}{\partial t} \log(p^\theta_t(x_t))dx_t+\int \frac{ p_t(x_t)}{p^\theta_t(x_t)} \frac{\partial p^\theta_t(x_t)}{\partial t}dx_t \right).
\end{equation}
We can use the Kolmogorov forward equations (Equation \ref{ode}), that is,  $\sum_{y}\mQ_t(x_t,y)p_t(y)=\frac{\partial p_{t}(x_t)}{\partial t}$ and $\frac{\partial p^{\theta}_{t}(x_t)}{\partial t} = \sum_{y}\mQ_t(x_t,y)p^{\theta}_t(y)$ to get
\begin{equation}
 = - \left( \int \sum_{y}\mQ_t(x_t,y)p_t(y)\log(p^\theta_t(x_t))dx_t+\int \frac{ p_t(x_t)}{p^\theta_t(x_t)} \sum_{y}\mQ_t(x_t,y)p^{\theta}_t(y)dx_t \right)
\end{equation}
\begin{equation}
 = - \left( \int \sum_{y}\mQ_t(x_t,y)\frac{p_t(y)}{p_t(x_t)}\log(p^\theta_t(x_t))p(x_t)dx_t+\int p_t(x_t)\sum_{y}\mQ_t(x_t,y)\frac{p^{\theta}_t(y)}{p^\theta_t(x_t)} dx_t \right)
\end{equation}
\begin{equation}
 = - \mathbb{E}_{x_t \sim p_t} \sum_{y}\mQ_t(x_t,y) \left( \frac{p_t(y)}{p_t(x_t)}\log(p^\theta_t(x_t))+\frac{p^{\theta}_t(y)}{p^\theta_t(x_t)} \right)
\end{equation}
\begin{equation}
 = - \mathbb{E}_{x_t \sim p_t} \sum_{y}\mQ_t(x_t,y) \left( \frac{p^{\theta}_t(y)}{p^\theta_t(x_t)} - \frac{p_t(y)}{p_t(x_t)}\log\frac{p^\theta_t(y)}{p^\theta_t(x_t)}\right)
\end{equation}
\begin{equation}
+ \mathbb{E}_{x_t \sim p_t} \sum_{y}\mQ_t(x_t,y)[\frac{p_t(y)}{p_t(x_t)}\log p^\theta_t(y)].
\end{equation}
The second term above $\mathbb{E}_{x_t \sim p_t} \sum_{y}\mQ_t(x_t,y)[\frac{p_t(y)}{p_t(x_t)}\log p^\theta_t(y)]$ is zero. Indeed,
\begin{equation}
\mathbb{E}_{x_t \sim p_t} \sum_{y}\mQ_t(x_t,y)[\frac{p_t(y)}{p_t(x_t)}\log(p^\theta_t(y))] = \sum_{x_t}\sum_{y}\mQ_t(x_t,y)p_t(x_t)[\frac{p_t(y)}{p_t(x_t)}\log(p^\theta_t(y))]
\end{equation}
\begin{equation}
= \sum_{x_t}\sum_{y}\mQ_t(x_t,y)p_t(y)\log(p^\theta_t(y))=\sum_{y}\left(\sum_{x_t}\mQ_t(x_t,y)\right)p_t(y)\log(p^\theta_t(y))=0,
\end{equation}
since for any column of a transition-rate matrix the sum of that columns elements is 0, therefore $\sum_{x}\mQ_t(x_t,y)=0$. In total we have:
\begin{equation}
H(p_0, p_0^\theta) - H(p_1, p_1^\theta) =  \int_1^0  - \mathbb{E}_{x_t \sim p_t} \sum_{y}\mQ_t(x_t,y) \left( \frac{p^{\theta}_t(y)}{p^\theta_t(x_t)} - \frac{p_t(y)}{p_t(x_t)}\log\frac{p^\theta_t(y)}{p^\theta_t(x_t)}\right)  dt
\end{equation}
\begin{equation}\label{general_crossentropy_int}
=  \int_0^1  \mathbb{E}_{x_t \sim p_t} \sum_{y}\mQ_t(x_t,y) \left( \frac{p^{\theta}_t(y)}{p^\theta_t(x_t)} - \frac{p_t(y)}{p_t(x_t)}\log\frac{p^\theta_t(y)}{p^\theta_t(x_t)}\right)  dt.
\end{equation}
This allows the derivation of the expression for entropy of the data distribution, by taking $p_t^\theta=p_t$
\begin{equation}
H(p_0)  =  H(p_1)+\int_0^1  \mathbb{E}_{x_t \sim p_t} \sum_{y}\mQ_t(x_t,y) \left( \frac{p_t(y)}{p_t(x_t)}- \frac{p_t(y)}{p_t(x_t)}\log\frac{p_t(y)}{p_t(x_t)}\right)  dt,
\end{equation}
thus
\begin{equation}\label{data_entropy}
H(p_0)  =  H(p_1)-\int_0^1  \mathbb{E}_{x_t \sim p_t} \sum_{y}\mQ_t(x_t,y) K\left(\frac{p_t(y)}{p_t(x_t)}\right)  dt.
\end{equation}
If we assume $\frac{p^\theta_t(y)}{p^\theta_t(x_t)}=s_{\theta}(x_t,t)_y$ and $p_1=p_1^\theta:=p_b$, from Equation (\ref{general_crossentropy_int}) we get 
\begin{equation}\label{perfect_cross}
H(p_0, p_0^\theta) - H(p_1) = \int_0^1  \mathbb{E}_{x_t \sim p_t} \sum_{y}\mQ_t(x_t,y) \left( s_{\theta}(x_t,t)_y- \frac{p_t(y)}{p_t(x_t)}\log(s_{\theta}(x_t,t)_y)\right)  dt,
\end{equation}
and furthermore, from Equation (\ref{data_entropy}) it is trivial that
\begin{equation}
-H(p_0) + H(p_1) =  \int_0^1\mathbb{E}_{x_t \sim p_t} \sum_{y}\mQ_t(x_t,y) K\left(\frac{p_t(y)}{p_t(x_t)}\right)  dt,
\end{equation}
thus adding this last equation and Equation (\ref{perfect_cross}) we get 
\begin{equation}
    D_{KL}(p_0||p_0^\theta)
\end{equation}
\begin{equation}
= \int_0^1\mathbb{E}_{x_t\sim p_t}\sum_{y}\mQ_t(x_t,y)\left[K\left(\frac{p_t(y)}{p_t(x_t)}\right)+\left(s_\theta(x_t,t)_y-\frac{p_t(y)}{p_t(x_t)}\log{s_\theta(x_t,t)_y}\right)\right]dt.
\end{equation}
Since we know that $\frac{p_t(x_t)}{p_t(x_t)}=1$, and $s_\theta(x_t,t)_{x_t}=\frac{p^\theta_t(x_t)}{p^\theta_t(x_t)}=1$, we have
\begin{equation}
\left[K\left(\frac{p_t(x_t)}{p_t(x_t)}\right)+\left(s_\theta(x_t,t)_{x_t}-\frac{p_t({x_t})}{p_t(x_t)}\log{s_\theta(x_t,t)_{x_t}}\right)\right]=-1+1=0
\end{equation}
so we can change the sums $\sum_{y}$ to simply $\sum_{y\neq x_t}$, as follows 
\begin{equation}
    D_{KL}(p_0||p_0^\theta)
\end{equation}
\begin{equation}
= \int_0^1\mathbb{E}_{x_t\sim p_t}\sum_{y\neq x_t}\mQ_t(x_t,y)\left[K\left(\frac{p_t(y)}{p_t(x_t)}\right)+\left(s_\theta(x_t,t)_y-\frac{p_t(y)}{p_t(x_t)}\log{s_\theta(x_t,t)_y}\right)\right]dt.
\end{equation}
Finally, since $p_1=p_1^\theta:=p_b$, then $D_{KL}(p_1||p_{r})=0$, therefore we can add it to the right side, finishing the proof.
\end{proof}
We remark that as $s_\theta(x_t,t)_y\rightarrow \frac{p_t(y)}{p_t(x_t)}$, then $p_t(i)^\theta\rightarrow p_t(i)$, and thus $\frac{p_t^\theta (y)}{p_t^\theta (x_t)}\rightarrow \frac{p_t(y)}{p_t(x_t)}$, as such $s_\theta(x_t,t)_y\rightarrow \frac{p_t^\theta (y)}{p_t^\theta (x_t)}$. Therefore, the bound in Theorem \ref{theorem1} becomes tighter as the model improves.

\textbf{Theorem \ref{theorem4}.}\emph{ Under the conditions stated in Theorem \ref{theorem1}, the following inequality for the cross entropy between the data and the learned distribution holds:
\begin{equation*}
    H(p_0, p_0^\theta) \leq \int_0^1\mathbb{E}_{x_t\sim p_t}\sum_{y\neq x_t}\mQ_t(x_t,y)\bar{\ell}\left(\frac{p_t(y)}{p_t(x_t)}, s_\theta(x_t,t)_y \right) dt
\end{equation*}
\begin{equation}
-\int_0^1\mathbb{E}_{x_t\sim p_t}\sum_{y\neq x_t}\mQ_t(y,x_t)dt+ H(p_1, p_{r}),
\end{equation}
where $\bar{\ell}(a, b) = 
 \left(
b-a \log b
\right)$.
}
\begin{proof}
    From Theorem \ref{theorem1}, we have 
\begin{equation*}
    D_{KL}(p_0||p_0^\theta)-D_{KL}(p_1||p_{r})
\end{equation*}
\begin{equation}
\leq \int_0^1\mathbb{E}_{x_t\sim p_t}\sum_{y\neq x_t}\mQ_t(x_t,y)\left[K\left(\frac{p_t(y)}{p_t(x_t)}\right)+\left(s_\theta(x_t,t)_y-\frac{p_t(y)}{p_t(x_t)}\log{s_\theta(x_t,t)_y}\right)\right]dt.
\end{equation}

Since we know that $\frac{p_t(x_t)}{p_t(x_t)}=1$, we can manually set $s_\theta(x_t,t)_{x_t}=1$, and we get 
\begin{equation}
\left[K\left(\frac{p_t(x_t)}{p_t(x_t)}\right)+\left(s_\theta(x_t,t)_{x_t}-\frac{p_t({x_t})}{p_t(x_t)}\log{s_\theta(x_t,t)_{x_t}}\right)\right]=-1+1=0,
\end{equation}
therefore
\begin{equation*}
    H(p_0, p_0^\theta) - H(p_0) - D_{KL}(p_1||p_{r})
\end{equation*}
\begin{equation}
\leq \int_0^1\mathbb{E}_{x_t\sim p_t}\sum_{y}\mQ_t(x_t,y)\left[K\left(\frac{p_t(y)}{p_t(x_t)}\right)+\left(s_\theta(x_t,t)_y-\frac{p_t(y)}{p_t(x_t)}\log{s_\theta(x_t,t)_y}\right)\right]dt.
\end{equation} from which we deduce,
\begin{equation*}
    H(p_0, p_0^\theta)
\end{equation*}
\begin{equation*}
\leq  H(p_0) + \int_0^1\mathbb{E}_{x_t\sim p_t}\sum_{y}\mQ_t(x_t,y) K\left(\frac{p_t(y)}{p_t(x_t)}\right)dt
\end{equation*}
\begin{equation}
+ \int_0^1\mathbb{E}_{x_t\sim p_t}\sum_{y}\mQ_t(x_t,y)\left(s_\theta(x_t,t)_y-\frac{p_t(y)}{p_t(x_t)}\log{s_\theta(x_t,t)_y}\right)dt + D_{KL}(p_1||p_{r}).
\end{equation}
Since from Theorem \ref{theorem3} we have that 
\begin{equation}
H(p_0) +\int_0^1  \mathbb{E}_{x_t \sim p_t} \sum_{y}Q_t(x,y) K\left(\frac{p_t(y)}{p_t(x_t)}\right)  dt =  H(p_1),
\end{equation}
hence we can write
\begin{equation*}
    H(p_0, p_0^\theta)
\end{equation*}
\begin{equation*}
\leq \int_0^1\mathbb{E}_{x_t\sim p_t}\sum_{y}\mQ_t(x_t,y)\left(s_\theta(x_t,t)_y-\frac{p_t(y)}{p_t(x_t)}\log{s_\theta(x_t,t)_y}\right)dt + H(p_1)+ D_{KL}(p_1||p_{r}),
\end{equation*}
and therefore 
\begin{equation*}
H(p_0, p_0^\theta) \leq \int_0^1\mathbb{E}_{x_t\sim p_t}\sum_{y\neq x_t}\mQ_t(x_t,y)\left(s_\theta(x_t,t)_y-\frac{p_t(y)}{p_t(x_t)}\log{s_\theta(x_t,t)_y}\right)dt
\end{equation*}
\begin{equation}
+ \int_0^1\mathbb{E}_{x_t\sim p_t}\mQ_t(x_t,x_t)dt+ H(p_1, p_{r}).
\end{equation}
The fact that $\mQ_t(x_t,x_t) = - \sum_{y\neq x_t}\mQ_t(y, x_t)$ concludes the proof.
\end{proof}

\textbf{Proposition \ref{proposition1}.}
    \emph{In the case of the roulette diffusion with roulette-loglinear noise, $H(p_r)=0$ and $-\int_0^1\mathbb{E}_{\vx_t\sim p_t}\sum_{\vy\neq \vx_t}\mQ_t(\vy,\vx_t)dt=\left(1-\frac{1-p_m}{n-1} \right)\frac{L}{p_m}(\epsilon-1)$. For the absorb diffusion, that is when $p_m\rightarrow1$, we have $-\int_0^1\mathbb{E}_{\vx_t\sim p_t}\sum_{\vy\neq \vx_t}\mQ_t(\vy,\vx_t)dt=L(\epsilon-1)$. Finally, for the uniform diffusion, $H(p_r)=L\log(V)$ and $-\int_0^1\mathbb{E}_{\vx_t\sim p_t}\sum_{\vy\neq \vx_t}\mQ_t(\vy,\vx_t)dt=-\left(1-\frac{1}{V} \right)L\int_0^1 \sigma^{'}(t)dt$.}
\begin{proof}
In all cases, 
\begin{equation}
    -\int_0^1\mathbb{E}_{\vx_t\sim p_t}\sum_{\vy\neq \vx_t}\mQ_t(\vy,\vx_t)dt  = -\mathbb{E}_{t\sim U(0,1)}\mathbb{E}_{\vx_t\sim p_t(\vx_t)} \sum_{i=1}^L \sum_{\vy^i\neq \vx^i_t}\mQ_t^{tok}(\vy^i,\vx^i_t)
\end{equation}    
\begin{equation}
   = -\mathbb{E}_{t\sim U(0,1)}\sum_{i=1}^L \mathbb{E}_{\vx_t\sim p_t(\vx_t)} \sum_{\vy^i\neq \vx^i_t}\mQ_t^{tok}(\vy^i,\vx^i_t)
\end{equation}  
\begin{equation}
   = -\mathbb{E}_{t\sim U(0,1)}\sum_{i=1}^L \sum_{\vx^i_t} p_{t}(\vx^i_t) \sum_{\vy^i\neq \vx^i_t}\mQ_t^{tok}(\vy^i,\vx^i_t)
\end{equation}  
\begin{equation}\label{proposition6_base}
   =-\mathbb{E}_{t\sim U(0,1)}\sum_{i=1}^L \sum_{\vx^i_0}p_{0}(\vx^i_0) \sum_{\vx^i_t}p_{t|0}(\vx^i_t|\vx^i_0)\sum_{\vy^i\neq \vx^i_t}\mQ_t^{tok}(\vy^i,\vx^i_t)
\end{equation}  
For the roulette case, if $\vx^i_t$ is the masked token, that is, $\vx^i_t=n$, then $\sum_{\vy^i\neq \vx^i_t}\mQ_t^{tok}(\vy^i,\vx^i_t)=0$. Otherwise, $\sum_{\vy^i\neq \vx^i_t}\mQ_t^{tok}(\vy^i,\vx^i_t)=\sigma_t^{'}\left(1-\frac{1-p_m}{n-1} \right)$. Thus, 
\begin{equation}
    -\int_0^1\mathbb{E}_{\vx_t\sim p_t}\sum_{\vy\neq \vx_t}\mQ_t(\vy,\vx_t) = 
\end{equation}  
\begin{equation}
    =-\left(1-\frac{1-p_m}{n-1} \right)\mathbb{E}_{t\sim U(0,1)}\sum_{i=1}^L \sum_{\vx^i_0}p_{0}(\vx^i_0)  \sum_{\vx^i_t\neq n}p_{t|0}(\vx^i_t|\vx^i_0)\sigma_t^{'}=
\end{equation}  
\begin{equation}
 =-\left(1-\frac{1-p_m}{n-1} \right)\mathbb{E}_{t\sim U(0,1)}\sum_{i=1}^L \sum_{\vx^i_0}p_{0}(\vx^i_0)  \sigma_t^{'}\sum_{\vx^i_t\neq n}p_{t|0}(\vx^i_t|\vx^i_0)
\end{equation}  
\begin{equation}
 =-\left(1-\frac{1-p_m}{n-1} \right)\mathbb{E}_{t\sim U(0,1)}\sum_{i=1}^L \sum_{\vx^i_0}p_{0}(\vx^i_0)  \sigma_t^{'}\left(1-p(\vx^i_t=n|\vx^i_0)\right)
\end{equation}  
\begin{equation}
 =-\left(1-\frac{1-p_m}{n-1} \right)\mathbb{E}_{t\sim U(0,1)}\sum_{i=1}^L \sum_{\vx^i_0}p_{0}(\vx^i_0)  \sigma_t^{'}e^{-\sigma_t p_m}=-\left(1-\frac{1-p_m}{n-1} \right)L\int_0^1 \sigma_t^{'}e^{-\sigma_t p_m}dt
\end{equation} 
\begin{equation}
=\left(1-\frac{1-p_m}{n-1} \right)\frac{L}{p_m}(e^{-\sigma_1 p_m}-e^{-\sigma_0 p_m})
\end{equation} 
Since $\sigma_t$ is roulette-loglinear, that is $\sigma_t=-\frac{1}{p_m}\log{\left(1-(1-\epsilon)t\right)}$, we get the result stated in the Proposition. 

In the Uniform case, we notice that Equation (\ref{proposition6_base}) can be rewritten as 
\begin{equation}
   \mathbb{E}_{t\sim U(0,1)}\sum_{i=1}^L \sum_{\vx^i_0}\sum_{\vx^i_t}p(\vx^i_t,\vx^i_0)\mQ_t^{tok}(\vx^i_t,\vx^i_t)
\end{equation}  
\begin{equation}
=\mathbb{E}_{t\sim U(0,1)}\sum_{i=1}^L \sum_{\vx^i_0}\sum_{\vx^i_t}p(\vx^i_t,\vx^i_0)\sigma^{'}(t)\left(\frac{1}{V}-1\right)=-L\left(1-\frac{1}{V}\right)\int_0^1 \sigma^{'}(t) dt.
\end{equation}  
\end{proof}

\subsection{Cross Entropy Discrete Diffusion (CEDD)}\label{rederive}
\subsubsection{Rederiving the Denosing Score Entropy Loss}
In what follows, we rederive the Denosing Score Entropy Loss for completeness \citep{lou2023discrete}. Training a model using ratio matching is performed by minimizing the error of a network that receives as input $\vx_t$ and tries to predict the ratios of $\frac{p_t(\vy)}{p_t(\vx_t)}$, where $\vy$ is a neighbour of $\vx_t$ with Hamming distance of 1. In other words, one tries to minimize
\begin{equation}
    \bar{\ell} \left(\frac{p_t(\vy)}{p_t(\vx_t)}, s_\theta(\vx_t, t)_\vy\right),
\end{equation}
for all $\vx_t$ following distribution $p_t(\vx_t)$, that is: 
\begin{equation}
    \mathbb{E}_{\vx_t\sim p_t(\vx_t)} \bar{\ell}\left(\frac{p_t(\vy)}{p_t(\vx_t)}, s_\theta(\vx_t, t)_\vy\right).
\end{equation}
The main bottleneck is that the ratios $\frac{p_t(\vy)}{p_t(\vx_t)}$ are unknown, as $p_t(\vx_t)=\int p_{t|0}(\vx_t|\vx_0)p_0(\vx_0)d \vx_0$ and $p_0(\vx_0)$ is what we are trying to model in the first place.
Luckily, the denoising trick can be employed: 
\begin{equation}\label{eqghj}
    \mathbb{E}_{\vx_t\sim p_t(\vx_t)} \bar{\ell}\left(\frac{p_t(\vy)}{p_t(\vx_t)}, s_\theta(\vx_t, t)_\vy\right)=\mathbb{E}_{\vx_t\sim p_t(\vx_t)} \bar{\ell}\left(\Sigma_{\vx_0}p_{t|0}(\vy|\vx_0)\frac{p_0(\vx_0)}{p_t(\vx_t)}, s_\theta(\vx_t, t)_\vy\right)=
\end{equation}
\begin{equation}
=\mathbb{E}_{\vx_t\sim p_t(\vx_t)} \bar{\ell}\left(\Sigma_{\vx_0}\frac{p_{t|0}(\vy|\vx_0)}{p_{t|0}(\vx_t|\vx_0)}\frac{p_{t|0}(\vx_t|\vx_0)p_0(\vx_0)}{p_t(\vx_t)}, s_\theta(\vx_t, t)_\vy\right)
\end{equation}
\begin{equation}\label{eqgdj}
=\mathbb{E}_{\vx_t\sim p_t(\vx_t)} \bar{\ell}\left(\Sigma_{\vx_0}\frac{p_{t|0}(\vy|\vx_0)}{p_{t|0}(\vx_t|\vx_0)}p_{0|t}(\vx_0|\vx_t), s_\theta(\vx_t, t)_\vy\right)
\end{equation}
\begin{equation}
=\mathbb{E}_{\vx_t\sim p_t(\vx_t)} \bar{\ell}\left(\mathbb{E}_{\vx_0 \sim p_{0|t}(\vx_0|\vx_t)}\frac{p_{t|0}(\vy|\vx_0)}{p_{t|0}(\vx_t|\vx_0)}, s_\theta(\vx_t, t)_\vy\right).
\end{equation}
When $\bar{\ell}$ is linear with respect to ratios, like in the case of Score Entropy \citep{lou2023discrete}, then we can pull the inner expectation (sum) outside and have 
\begin{equation}
    \mathbb{E}_{\vx_t\sim p_t(\vx_t)} \bar{\ell}\left(\frac{p_t(\vy)}{p_t(\vx_t)}, s_\theta(\vx_t, t)_\vy\right) = \mathbb{E}_{\vx_t\sim p_t(\vx_t)} \mathbb{E}_{\vx_0 \sim p_{0|t}(\vx_0|\vx_t)} \bar{\ell}\left(\frac{p_{t|0}(\vy|\vx_0)}{p_{t|0}(\vx_t|\vx_0)}, s_\theta(\vx_t, t)_\vy\right),
\end{equation}
which of course is equal to 
 \begin{equation}\label{denoising}
\mathbb{E}_{\vx_0\sim p_0(\vx_0)} \mathbb{E}_{\vx_t \sim p_{t|0}(\vx_t|\vx_0)} \bar{\ell}\left(\frac{p_{t|0}(\vy|\vx_0)}{p_{t|0}(\vx_t|\vx_0)}, s_\theta(\vx_t, t)_\vy\right).
\end{equation}

\subsubsection{Deriving Cross-Entropy from Ratio Matching}\label{ceddappendix}

In order to go from Equation (\ref{eqghj}) to (\ref{eqgdj}) we used the fact that 
 \begin{equation}
\frac{p_t(\vy)}{p_t(\vx_t)} = \Sigma_{\vx_0}\frac{p_{t|0}(\vy|\vx_0)}{p_{t|0}(\vx_t|\vx_0)}p_{0|t}(\vx_0|\vx_t).
\end{equation}
Now, we select position $i$ on the sequence with length $L$. We want to find an expression of the ratios $\frac{p_t(\vy)}{p_t(\vx_t)}$ for all sequences $\vy$, which differ with $\vx_t$ only on position $i$. From above we have 
 \begin{equation}
\frac{p_t(\vy)}{p_t(\vx_t)} = \Sigma_{\vx_0}\frac{p_{t|0}(\vy|\vx_0)}{p_{t|0}(\vx_t|\vx_0)}p_{0|t}(\vx_0|\vx_t),
\end{equation} where

\begin{equation}
\frac{p_{t|0}(\vy|\vx_0)}{p_{t|0}(\vx_t|\vx_0)}=\frac{p_{t|0}\left((\vy^{(0)}, \vy^{(1)}, ..., \vy^{(i)}, ..., \vy^{(V-1)})|({\vx_0}^{(0)}, {\vx_0}^{(1)}, ..., {\vx_0}^{(i)}, ..., {\vx_0}^{(V-1)})\right)}{p_{t|0}\left(({\vx_t}^{(0)}, {\vx_t}^{(1)}, ..., {\vx_t}^{(i)}, ..., {\vx_t}^{(V-1)})|({\vx_0}^{(0)}, {\vx_0}^{(1)}, ..., {\vx_0}^{(i)}, ..., {\vx_0}^{(V-1)})\right)}
\end{equation}
and due to independence between entries in the forward process we have
\begin{equation}
\frac{p_{t|0}(\vy|\vx_0)}{p_{t|0}(\vx_t|\vx_0)}=\frac{\prod p_{t|0}(\vy^j|{\vx_0}^j)}{\prod p_{t|0}({\vx_t}^j|{\vx_0}^j)}=\prod \frac{p_{t|0}(\vy^j|{\vx_0}^j)}{p_{t|0}({\vx_t}^j|{\vx_0}^j)}=\frac{p_{t|0}(\vy^i|{\vx_0}^i)}{p_{t|0}({\vx_t}^i|{\vx_0}^i)},
\end{equation}
as the rest of these ratios are 1, since $\vy$ differs with $\vx_t$ only on position $i$. Therefore we have 
 \begin{equation}
\frac{p_t(\vy)}{p_t(\vx_t)} = \Sigma_{\vx_0}\frac{p_{t|0}(\vy^i|{\vx_0}^i)}{p_{t|0}({\vx_t}^i|{\vx_0}^i)}p_{0|t}(\vx_0|\vx_t),
\end{equation} 
where $p_{0|t}(\vx_0|\vx_t)=p_{0|t}({\vx_0}^{(0)}, {\vx_0}^{(1)}, ..., {\vx_0}^{(i)}, ..., {\vx_0}^{(V-1)}|\vx_t)$ and where the expression inside the sum $\frac{p_{t|0}(\vy^i|{\vx_0}^i)}{p_{t|0}({\vx_t}^i|{\vx_0}^i)}$, depends only on ${\vx_0}^i$ and not on the rest of ${\vx_0}^j$ for $j\neq i$. Thus,
 \begin{equation}
\Sigma_{\vx_0}\frac{p_{t|0}(\vy^i|{\vx_0}^i)}{p_{t|0}({\vx_t}^i|{\vx_0}^i)}p_{0|t}(\vx_0|\vx_t)=\Sigma_{{\vx_0}^i}\frac{p_{t|0}(\vy^i|{\vx_0}^i)}{p_{t|0}({\vx_t}^i|{\vx_0}^i)}\Sigma_{{\vx_0}^{(0)}, ..., {\vx_0}^{(i-1)}, {\vx_0}^{(i+1)}, ..., {\vx_0}^{(V-1)}}p_{0|t}(\vx_0|\vx_t),
\end{equation} 
which implies that 
 \begin{equation}
\frac{p_t(\vy)}{p_t(\vx_t)} = \Sigma_{{\vx_0}^i}\frac{p_{t|0}(\vy^i|{\vx_0}^i)}{p_{t|0}({\vx_t}^i|{\vx_0}^i)}p^i_{0|t}({\vx_0}|\vx_t),
\end{equation}
or more clearly 
 \begin{equation}\label{ratioexplicitformula}
\frac{p_t(\vy)}{p_t(\vx_t)} = \Sigma_{h\in[V]}\frac{p_{t|0}(\vy^i|h)}{p_{t|0}({\vx_t}^i|h)}p^i_{0|t}(h|\vx_t).
\end{equation}
If we learn the $V$ probabilities $[p^i_{0|t}({0}|\vx_t), p^i_{0|t}({1}|\vx_t), ..., p^i_{0|t}({(V-1)}|\vx_t)]$, since we analytically have $\frac{p_{t|0}(\vy^i|{h}^i)}{p_{t|0}({\vx_t}^i|{h}^i)}$ (from the matrix exponential), then we we will have the ratio  $\frac{p_t(\vy)}{p_t(\vx_t)}$. We can choose another sequence $\vz$ which also differs only at position $i$ from $\vx_t$. Then we will have 
 \begin{equation}\label{eq:ratio_distilled}
\frac{p_t(\vz)}{p_t(\vx_t)} = \Sigma_{h\in[V]}\frac{p_{t|0}(\vz^i|h)}{p_{t|0}({\vx_t}^i|h)}p^i_{0|t}(h|\vx_t),
\end{equation}
which highlights the fact that even though the ratios might change, the same $[p^i_{0|t}({0}|\vx_t), p^i_{0|t}({1}|\vx_t), ..., p^i_{0|t}({(V-1)}|\vx_t)]$ appear as long as the selected position $i$ in the sequence remains unchanged. Therefore if we learn these probabilities for a position, we will have the ratios of all neighbours that differ only in that position, that is, we have $V$ ratios, by modeling these $V$ probabilities.
\newline
There is nothing special about position $i$ and we can choose to model $V$ probabilities for each $L$ positions. In order to do so, we define a neural network $f_\theta(\vx_t, t)$ whose output is $L\times V$, where entry $(i,h)$ of the output (that is $f_\theta^i(\vx_t, t)[h]$), predicts $p^i_{0|t}(h|\vx_t)$ the probability of the entry $i$ of $\vx_t$ (position $i$ of the current sequence), having being perturbed from token with id $h$ in the original sequence $\vx_0$, given that we are at sequence $\vx_t$ right now. That is, the model is directly trying to predict from where each of the tokens in the current sequence came from. Therefore we will reparametrize our ratio model as follows 
 \begin{equation}
s^i_\theta(\vx_t, t)[h]=s^i_\theta(\vx_t, t)_h = \Sigma_{h\in[V]}\frac{p_{t|0}(\vz^i|h)}{p_{t|0}({\vx_t}^i|h)}f_\theta^i(\vx_t, t)[h],
\end{equation}
where $f_\theta^i(\vx_t, t)$ are the outputs of the softmax at the end, and where for each $i$, distribution $f_\theta^i(\vx_t, t)$ should match $p_{0|t}^i({\cdot|\vx_t)}$. This happens when $D_{KL} (p_{0|t}^i({\cdot|\vx_t)}|f_\theta^i(\vx_t, t))=0$, thus we minimize,
 \begin{equation}
\sum_{i=1}^L \mathbb{E}_{t \sim U(0,1)}w(t)\mathbb{E}_{\vx_t\sim p_t(\vx_t)}  D_{KL} (p_{0|t}^i({\cdot|\vx_t)}||f_\theta^i(\vx_t, t)).
\end{equation}
Since $D_{KL} (p_{0|t}^i({\cdot|\vx_t)}||f_\theta^i(\vx_t, t)) = H (p_{0|t}^i({\cdot|\vx_t)}, f_\theta^i(\vx_t, t))-C=\mathbb{E}_{p_{0|t}^i({h|\vx_t)}}\log {f_\theta^i(\vx_t, t)[h]}-C$ the loss function above has the same gradients with regards to network parameters as 
 \begin{equation}
-\sum_{i=1}^L \mathbb{E}_{t \sim U(0,1)}w(t)\mathbb{E}_{\vx_t\sim  p_t(\vx_t)} \mathbb{E}_{h\sim p_{0|t}^i({h|\vx_t)}}\log {f_\theta^i(\vx_t, t)[h]},
\end{equation}
which is equal to 
 \begin{equation}
-\sum_{i=1}^L \mathbb{E}_{t \sim U(0,1)}\mathbb{E}_{\vx_t\sim  p_t(\vx_t)} \mathbb{E}_{\vx_0\sim p_{0|t}({\vx_0|\vx_t)}}w(t)\log {f_\theta^i(\vx_t, t)[\vx_0^i]}.
\end{equation}
Thus, as in  \citep{campbell2022continuous}, the loss we utilize in our training is $L_{ll}$,
 \begin{equation}
L_{ll}=-\mathbb{E}_{t \sim U(0,1)} \mathbb{E}_{\vx_0\sim p_0(\vx_0)} \mathbb{E}_{\vx_t \sim p_{t|0}(\vx_t|\vx_0)} \sum_{i=1}^L w(t)\log f_\theta^i(\vx_t, t)[\vx_0^i].
\end{equation}

We can also motivate this loss from the score entropy in Equation (\ref{denoising}). Indeed, since $\ell$ and $\bar{\ell}$ differ by a constant, its gradients are the same as those of
 \begin{equation}
\mathbb{E}_{\vx_0\sim p_0(\vx_0)} \mathbb{E}_{\vx_t \sim p_{t|0}(\vx_t|\vx_0)} \ell\left(\frac{p_{t|0}(\vy|\vx_0)}{p_{t|0}(\vx_t|\vx_0)}, \Sigma_{h\in[V]}\frac{p_{t|0}(\vy^i|h)}{p_{t|0}({\vx_t}^i|h)}f_\theta^i(\vx_t, t)[h] \right),
\end{equation} and as before since $\vy$ and $\vx_t$ differ at only one token we get 
 \begin{equation}
=\mathbb{E}_{\vx_0\sim p_0(\vx_0)} \mathbb{E}_{\vx_t \sim p_{t|0}(\vx_t|\vx_0)} \ell\left(\frac{p_{t|0}(\vy^i|{\vx_0}^i)}{p_{t|0}({\vx_t}^i|{\vx_0}^i)}, \Sigma_{h\in[V]}\frac{p_{t|0}(\vy^i|h)}{p_{t|0}({\vx_t}^i|h)}f_\theta^i(\vx_t, t)[h] \right).
\end{equation}
The function $\ell(a,b)$ is clearly minimized when $a=b$, in which case $\ell(a,b)=0$. In our case,
 \begin{equation}
\ell\left(\frac{p_{t|0}(\vy^i|{\vx_0}^i)}{p_{t|0}({\vx_t}^i|{\vx_0}^i)}, \Sigma_{h\in[V]}\frac{p_{t|0}(\vy^i|h)}{p_{t|0}({\vx_t}^i|h)}f_\theta^i(\vx_t, t)[h] \right)=0,
\end{equation}
for
 \begin{equation}
\frac{p_{t|0}(\vy^i|{\vx_0}^i)}{p_{t|0}({\vx_t}^i|{\vx_0}^i)}=\Sigma_{h\in[V]}\frac{p_{t|0}(\vy^i|h)}{p_{t|0}({\vx_t}^i|h)}f_\theta^i(\vx_t, t)[h],
\end{equation}
which happens when $f_t^i(\vx_t, \theta)[h\neq {\vx_0}^i]=0$ and $f^i_t(\vx_t, \theta)[{\vx_0}^i]=1$, that is when our probability prediction matches the one hot encoding used in cross entropy. Thus we train our model with cross entropy, which simplifies the job of the network of learning complex ratios, as the conditional ratios now do not participate in the loss. We only add them during sampling to the weighted sum in order to predict the marginal ratios, since we have analytic expressions for the conditional ratios. Therefore, we train the model by minimizing
 \begin{equation}
-\mathbb{E}_{t \sim U(0,1)} \mathbb{E}_{\vx_0\sim p_0(\vx_0)} \mathbb{E}_{\vx_t \sim p_{t|0}(\vx_t|\vx_0)} \sum_{i=1}^L w(t)\log f_\theta^i(\vx_t, t)[\vx_0^i].
\end{equation}
\subsubsection{Relation to Cross Entropy in Continuous Diffusion}\label{continousdiffconnection}
Cross entropy has also been used in continuous diffusion models applied to Language Modelling \citep{dieleman2022continuous}. Here, we draw parallels between the two approaches. We denote a sequence of $L$ tokens at time $t$ as $\vx_t = (\vx_t^1, ..., \vx_t^L)$, where each token is embedded in a $D-$dimensional space $\vx_t^i\in\mathbb{R}^D$. That is, each sequence can be seen as an $L\times D$ vector, created from concatenating the embeddings of each token. The score that generates the reverse process, is therefore a $L\times D$ vector $s_\theta(x_t, t)$, which approximates:
\begin{equation}
    \nabla_{\vx_t}\log p_t(\vx_t)=\int \nabla_{\vx_t}\log p_t(\vx_t|\vx_0) p_t(\vx_0|\vx_t)d\vx_0=
\end{equation}
\begin{equation}
    \int \nabla_{\vx_t}\log e^\frac{-||\vx_t-\vx_0||^2}{2\sigma_t^2} p_t(\vx_0|\vx_t)d\vx_0=\int \frac{\vx_0-\vx_t}{\sigma_t^2} p_t(\vx_0|\vx_t)d\vx_0=\int \frac{\vx_0}{\sigma_t^2} p_t(\vx_0|\vx_t)d\vx_0-\frac{\vx_t}{\sigma_t^2}.
\end{equation}
From above, it is clear that
\begin{equation}
    \nabla_{\vx_t^i}\log p_t(\vx_t)=\int \frac{\vx_0^i}{\sigma_t^2} p_t(\vx_0|\vx_t)d\vx_0-\frac{\vx_t^i}{\sigma_t^2}=\int \frac{\vx_0^i}{\sigma_t^2} \left(\int p_t(\vx_0|\vx_t)d\vx_0^{-i}\right)d\vx_0^i  -\frac{\vx_t^i}{\sigma_t^2},
\end{equation}
where $\vx_0^{-i}$ denotes all entries of $\vx_0$ without the ones of $\vx_0^{i}$. Therefore
\begin{equation}
    s_\theta(\vx_t, t)_{(i1, i2, .., iD)} \approx \nabla_{\vx_t^i}\log p_t(\vx_t) =\int \frac{\vx_0^i}{\sigma_t^2} p_t(\vx_0^i|\vx_t) d\vx_0^i  -\frac{\vx_t^i}{\sigma_t^2}.
\end{equation}
Clearly, all that need to be learned in order to model the score for the dimensions of the token at position $i$ are the $V$ probabilities $p_t(\vx_0^i|\vx_t)$ . For $L$ such tokens (to model the score of the entire sequence) one simply needs to model $L\times V$ probabilities using cross entropy as in the discrete case.

\newpage
\subsection{Roulette Discrete Diffusion}
The roulette transition-rate matrix is 
\begin{equation}\label{eqQt4}
\mQ_{roulette}^{tok} = \begin{pmatrix}
\frac{1}{n-1}(1-p_m)-1 & \frac{1}{n-1}(1-p_m) &... &\frac{1}{n-1}(1-p_m)& 0\\
\frac{1}{n-1}(1-p_m) & \frac{1}{n-1}(1-p_m)-1 &... &\frac{1}{n-1}(1-p_m)& 0\\
... & ... &... &...&...\\
\frac{1}{n-1}(1-p_m) & \frac{1}{n-1}(1-p_m) &... &\frac{1}{n-1}(1-p_m)-1& 0\\
p_m & p_m &... &p_m& 1-1\
\end{pmatrix}_{n\times n},
\end{equation}
where $n-1=V$ is the number of tokens in our vocabulary. We add a special token for the absorb (mask) state, therefore increasing the number of total states to $n$, with the token id $n$ corresponding to the absorbed state. Unfortunately, we cannot derive the matrix exponential as in the case of absorb and uniform diffusion, since the corresponding matrix $\mP$ of $\mQ$ is not idempotent. Therefore, we will first derive the exponential matrix of the uniform and absorb diffusion manually, in order to motivate the manual derivation of the exponential matrix in the roulette case. 
\subsubsection{Deriving the exponential matrix of the absorb diffusion}\label{appendixabsorbexp}
The transition matrix in the absorb case is 
\begin{equation}
\mQ_{abs}^{tok} = \begin{pmatrix}
0-1 & 0 &... &0& 0\\
0 & 0-1 &... &0& 0\\
... & ... &... &...&...\\
0 & 0 &... &0-1& 0\\
1 & 1 &... &1& 1-1\
\end{pmatrix}_{n\times n}.
\end{equation}
Clearly this is $\mQ_{roulette}^{tok}$ when $p_m=1$.
\newline
If we are at state $i$, the probability of moving at state $n$ (the absorb/mask state) for a time-step of size $\epsilon$ is $\epsilon$.
We discretize the time interval into discrete steps that are multiples of $\epsilon$, that is $[0,\tau]\rightarrow{\{\epsilon,...,j\epsilon,...\lfloor \frac{\tau}{\epsilon} \rfloor\epsilon\}}$ and denote the event of jumping from $i\neq n$ to $n$ at time $j\epsilon$ as $a_j$. The intersection of any of these two events is empty and their union is the space of all possibilities, thus the probability of being at $n$ at time $\tau$ is
\begin{equation}
    p(x_\tau=n)=p(\{x_\tau=n\}\cap \Omega)=p(\{x_\tau=n\}\cap \{\cup a_j\})=p(\cup_{j\epsilon<\tau}\{ a_j\})
\end{equation}
\begin{equation}
    p(x_\tau=n)=\Sigma_{j<\frac{\tau}{\epsilon} }p( a_j)=\epsilon+\epsilon(1-\epsilon)+\epsilon(1-\epsilon)^2+...+\epsilon(1-\epsilon)^{\lfloor \frac{\tau}{\epsilon} \rfloor-1}
\end{equation}
\begin{equation}
    p(x_\tau=n)=\epsilon\frac{1-(1-\epsilon)^{\lfloor \frac{\tau}{\epsilon} \rfloor}}{1-(1-\epsilon)}=1-(1-\epsilon)^{\lfloor \frac{\tau}{\epsilon} \rfloor}
\end{equation}
Taking the limit $\epsilon \rightarrow{0}$, we get $p(x_\tau=n)=1-e^{-\tau}$.
Thus, assuming that we start at state $i$, we have $p_{\tau|0}(j\not\in \{n,i\}|i\neq n)=0$, $p_{\tau|0}(j=i|i\neq n)=e^{-\tau}$ and $p_{\tau|0}(j=n|i\neq n)=1-e^{-\tau}$. This defines the column $i\neq n$ of the exponential $e^{\tau\mQ^{tok}}$. Iterating though different $i$ we construct the entire exponential matrix, except its last column which is a simple one hot encoding at $n$. Setting $\tau=\sigma_t$ finishes the derivation.
\subsubsection{Deriving the exponential of the uniform diffusion}\label{appendixuniformexp}
\begin{equation}
\mQ_{unif}^{tok} = \begin{pmatrix}
\frac{1}{n-1}-1 & \frac{1}{n-1} &... &\frac{1}{n-1}\\
\frac{1}{n-1}& \frac{1}{n-1}-1 &... &\frac{1}{n-1}\\
... & ... &... &...\\
\frac{1}{n-1} & \frac{1}{n-1} &... &\frac{1}{n-1}-1\\
\end{pmatrix}_{V\times V}
\end{equation}
Clearly this is $\mQ_{roulette}^{tok}$ when $p_m= 0$, without the absorption row and column, as our vocabulary size is $V=n-1$ since we do not have the special token.
As before we discretize the time interval into discrete steps that are multiples of $\epsilon$, that is $[0,\tau]\rightarrow{\{\epsilon,...,j\epsilon,...\lfloor \frac{\tau}{\epsilon} \rfloor\epsilon\}}$. We assume that initially (at time $0$) we start at position $i$. We wish to find the probability of being at state $k\neq i$ at time $\lfloor \frac{\tau}{\epsilon} \rfloor\epsilon$. For the sake of simplicity, we abuse notation by denoting $p_{(\lfloor \frac{\tau}{\epsilon} \rfloor\epsilon-m\epsilon)|0}(j| i)$ as $\bar{p}_{\tau-m\epsilon}(j|i)$. By symmetry we know that for every $j\neq i$ the probability $\bar{p}_{\tau-\epsilon}(j|i)$ is the same.  Thus the probability of $\bar{p}_{\tau}(j|i)$ is the sum of the following components:
\begin{itemize}
    \item probability of being at $j\not\in\{i, k\}$ at time $\tau-\epsilon$, i.e, ($\bar{p}_{\tau-\epsilon}(j|i)$), times $n-3=V-2$ such states, times probability of moving ($\epsilon$), times probability of hitting $k$ on that move ($\frac{1}{n-1}$).
    \item probability of being at $j=i$ at time $\tau-\epsilon$, i.e, ($1-(n-2)\cdot \bar{p}_{\tau-\epsilon}(j|i)$), times probability of moving ($\epsilon$), times probability of hitting $k$ on that move ($\frac{1}{n-1}$).
    \item probability of being at $j=k$ at time $\tau-\epsilon$, i.e, ($\bar{p}_{\tau-\epsilon}(j|i)$), times of staying there ($1-\epsilon$).
    \item probability of being at $j=k$ at time $\tau-\epsilon$, i.e, ($\bar{p}_{\tau-\epsilon}(j|i)$), times of moving ($\epsilon$), times of hitting itself on this move ($\frac{1}{n-1}$).
\end{itemize}
We write them down mathematically and get:
\begin{equation}
   \bar{p}_{\tau}(k|i)= \bar{p}_{\tau-\epsilon}(j|i)\epsilon\frac{1}{n-1}(n-3)+(1-(n-2)\bar{p}_{\tau-\epsilon}(j|i))\epsilon\frac{1}{n-1}+
\end{equation}
\begin{equation}
   +\bar{p}_{\tau-\epsilon}(j|i)(1-\epsilon)+\bar{p}_{\tau-\epsilon}(j|i)\frac{1}{n-1}\epsilon,
\end{equation}
thus 
\begin{equation}
   \bar{p}_{\tau}(k|i)= \frac{1}{n-1}\epsilon+\bar{p}_{\tau-\epsilon}(j|i)(1-\epsilon).
\end{equation}
Since by definition $k\neq i$ and $j\neq i$, then by symmetry we have $\bar{p}_{\tau}(k|i)=\bar{p}_{\tau}(j|i)$.
Therefore we get the recursion: 
\begin{equation}
   p_{\lfloor \frac{\tau}{\epsilon} \rfloor\epsilon|0}(k| i)= \frac{1}{n-1}\epsilon+p_{(\lfloor \frac{\tau}{\epsilon} \rfloor\epsilon-\epsilon)|0}(k| i),
\end{equation}
which if we fully develop becomes:
\begin{equation}
   p_{\lfloor \frac{\tau}{\epsilon} \rfloor\epsilon|0}(k| i)= \frac{1}{n-1}\epsilon(1+(1-\epsilon)+(1-\epsilon)^2+...+(1-\epsilon)^{\lfloor \frac{\tau}{\epsilon} \rfloor-1}),
\end{equation}
which is equal to:
\begin{equation}
   p_{\lfloor \frac{\tau}{\epsilon} \rfloor\epsilon|0}(k| i)= \frac{1}{n-1}\epsilon\frac{1-(1-\epsilon)^{\lfloor \frac{\tau}{\epsilon} \rfloor}}{1-(1-\epsilon)}= \frac{1}{n-1}(1-(1-\epsilon)^{\lfloor \frac{\tau}{\epsilon} \rfloor}).
\end{equation}
Taking the limit $\epsilon\rightarrow 0$ gives $p_{\tau|0}(k\neq i|i)=\frac{1}{n-1}(1-e^{-\tau})$ and $p_{\tau|0}(i|i)=1-\frac{n-2}{n-1}(1-e^{-\tau})$, which gives the $i$th column of the exponential matrix. Setting $\tau=\sigma_t$ finishes the derivation.
\subsubsection{Deriving the exponential of the roulette diffusion}\label{appendixroulettederiv}
First, we notice that we must start at a state that is different from the absorb state (there are no masked tokens in the training data). We can split the states into two groups, the non-absorbing states, and the absorbing state. We consider the non-absorbing states as a single super-state and the absorbing state as the other super-state. From our derivations in Appendix \ref{appendixabsorbexp}, we know how to derive the probability of being at the absorb super-state at time $\tau$. The only difference is that previously, if we moved, it would be certain we would move to the absorption super-state, while in our case, we can move to a different state in our non-absorb super-state. The possibility of moving to the absorb super-state in a $\epsilon$ time step, changes from $\epsilon$ to $\epsilon p_m$, and the formula becomes
\begin{equation}
    p(x_\tau=n)=\Sigma_{j<\frac{\tau}{\epsilon} }p( a_j)=\epsilon p_m+\epsilon p_m(1-\epsilon p_m)+\epsilon p_m(1-\epsilon p_m)^2+...+\epsilon p_m(1-\epsilon p_m)^{\lfloor \frac{\tau}{\epsilon} \rfloor-1}
\end{equation}
\begin{equation}
    p(x_\tau=n)=\epsilon p_m\frac{1-(1-\epsilon p_m)^{\frac{1}{\epsilon p_m} \lfloor \frac{\tau}{\epsilon} \rfloor \epsilon p_m}}{1-(1-\epsilon p_m)}=1-(1-\epsilon p_m)^{\frac{1}{\epsilon p_m} \lfloor \frac{\tau}{\epsilon} \rfloor \epsilon p_m}
\end{equation}
which implies that $p(x_\tau=n)=1-e^{-\tau p_m}$. This means that at time $\tau$ we are in one of the non-absorption states with a probability of $e^{-\tau p_m}$. Now, as before (Appendix \ref{appendixuniformexp}), we want to find the probability of being at each non-absorption state at time $\tau$ given that we stated a state $i$. We can construct these probabilities using a similar approach as before, but with two key differences. First, because we are conditioning on being in the non-absorption super-state, all probabilities must be multiplied by  $e^{-\tau p_m}$. Second, this conditioning implicitly provides some information about the probability of moving when going from time $\tau-\epsilon$ to time $\tau$. In the uniform case, for a time step $\epsilon$ this probability was $\epsilon$, but now conditioning on the fact that we are not at the absorption state at time $\tau$ provides extra information on whether we moved when going from $\tau-\epsilon$ to $\tau$. Indeed, given this extra information, we expect the probability of having moved when going form $\tau-\epsilon$ to $\tau$ to be reduced. This becomes obvious when $p_m=1$, as then we are in the case of the absorb diffusion, and saying that at time $\tau$ we are not at the absorption state, immediately implies that we did not move, as there are no other possibilites. As $p_m$ decreases, the magnitude of this information decreases. So instead of writing the probability of moving (during an $\epsilon$ time step) with $\epsilon$ as we did previously, we write this probability with $\delta(\epsilon)$, and as before we can derive that 
\begin{equation}
   \bar{p}_{\tau}(k|i)= \bar{p}_{\tau-\epsilon}(j|i)\delta\frac{1}{n-1}(n-3)+(1-(n-2)\bar{p}_{\tau-\epsilon}(j|i))\delta\frac{1}{n-1}+
\end{equation}
\begin{equation}
   +\bar{p}_{\tau-\epsilon}(j|i)(1-\delta)+\bar{p}_{\tau-\epsilon}(j|i)\frac{1}{n-1}\delta
\end{equation}
thus 
\begin{equation}
   \bar{p}_{\tau}(k|i)= \frac{1}{n-1}\delta+\bar{p}_{\tau-\epsilon}(k|i)(1-\delta)
\end{equation}
and by the same recursion trick as before,
\begin{equation}
   p_{\lfloor \frac{\tau}{\epsilon} \rfloor\epsilon|0}(k| i)= \frac{1}{n-1}\delta\frac{1-(1-\delta)^{\lfloor \frac{\tau}{\epsilon} \rfloor}}{1-(1-\delta)}= \frac{1}{n-1}(1-(1-\delta)^{\lfloor \frac{\tau}{\epsilon} \rfloor}).
\end{equation}
Now we derive the expression of $\delta$. We write the move event when going from $\tau-\epsilon$ to $\tau$ with $A$ and not going at the absorbed state at time $\tau$ with $B$
\begin{equation}
   \delta=p(A|B)=\frac{p(A)}{p(B)}p(B|A)=\frac{p(B|A)p(A)}{p(B|A)p(A)+p(B|A^C)p(A^C)}=
\end{equation}
\begin{equation}
   =\frac{(1-p_m)\epsilon}{(1-p_m)\epsilon+1(1-\epsilon)}=\alpha\epsilon
\end{equation}
where $\alpha=\frac{(1-p_m)}{(1-p_m)\epsilon+1(1-\epsilon)}$ goes to $1-p_m$ when $\epsilon$ goes to 0,
therefore finally
\begin{equation}
   p_{\lfloor \frac{\tau}{\epsilon} \rfloor\epsilon|0}(k| i)=\frac{1}{n-1}(1-(1-\delta)^{\lfloor \frac{\tau}{\epsilon} \rfloor})=\frac{1}{n-1}(1-(1-\delta)^{\frac{1}{\delta}\delta {\lfloor \frac{\tau}{\epsilon} \rfloor}})=\frac{1}{n-1}(1-(1-\delta)^{\frac{1}{\delta}\alpha\epsilon {\lfloor \frac{\tau}{\epsilon} \rfloor}}).
\end{equation}
Therefore $p_{\tau|0}(j\not\in\{i,n\}|i)=e^{-\tau p_m}\frac{1}{n-1}(1-e^{(1-p_m)\tau})$, and $p_{\tau|0}(n|i)=1-e^{-\tau p_m}$, and $p_{\tau|0}(i|i)=e^{-\tau p_m}(1-\frac{n-2}{n-1}(1-e^{(1-p_m)\tau}))$, which gives the $i$th column of the exponential matrix, when $i\neq n$. In case $i=n$, the column of the exponential matrix is simply the one hot encoding at $n$. Setting $\tau=\sigma_t$ finishes the derivation.

\textbf{Proposition 5}\emph{
    If we denote with $\mY_t$ the matrix exponential of $\sigma_t \mQ^{tok}_{roulette}=\sigma_t \left(\mI-\mP^{tok}_{roulette}\right)$, then $\mY_t(i\not\in\{j,n\}, j\neq n)=e^{-\sigma_t  p_m}\frac{1}{n-1}(1-e^{-(1-p_m)\sigma_t })$, $\mY_t(i\neq n,i\neq n)=e^{-\sigma_t  p_m}(1-\frac{n-2}{n-1}(1-e^{-(1-p_m)\sigma_t }))$, $\mY_t (n,j\neq n)=1-e^{-\sigma_t  p_m}$, $\mY_t (i\neq n, n)=0$, and $\mY_t (n, n)=1$.}

\subsubsection{Controlling the number of the corrected tokens in the reverse process}\label{nrcorrectedtokens}
We wish to see how the choice of $p_m$ corresponds to the probability of a token moving uniformly at least once before being masked, given some diffusion interval from $[0, T]$. This should correspond to the probability of a token being corrected after it is unmasked in the reverse process, therefore enabling us to control the expected number of tokens to be corrected. 
\newline
As before, we discretize the time interval into subintervals of length $\epsilon$. The event of a token moving before getting masked will be denoted by $X$, and the event of a token being masked at time $j\epsilon$ is denoted with $A_j$. Therefore
\begin{equation}
    p(X)=p(X\cap \Omega)=p(X\cap \{\cup A_j\})=p(\cup_{j\epsilon<T}X\cap A_j)
\end{equation}
\begin{equation}
    p(X)=\Sigma_{j<\lfloor \frac{T}{\epsilon} \rfloor}p(X|A_j)p(A_j).
\end{equation}
$p(A_j)$ is the probability of being at the absorbtion state at $j\epsilon$ and not being there at $(j-1)\epsilon$. This probability is $p(A_j)=(1-e^{-j\epsilon p_m})-(1-e^{-(j-1)\epsilon p_m})=e^{-(j-1)\epsilon p_m}-e^{-j\epsilon p_m}$
therefore $p(A_j)=e^{-(j-1)\epsilon p_m}(1-e^{-\epsilon p_m})\approx e^{-(j-1)\epsilon p_m}\epsilon p_m$.
We write $\tau_i=i\epsilon$, and thus we have 
\begin{equation}
    p(X)=\Sigma_{j<\lfloor \frac{T}{\epsilon} \rfloor}p(X|A_j)p(A_j)=\Sigma_{j<\lfloor \frac{T}{\epsilon} \rfloor}\frac{n-2}{n-1}(1-e^{-(1-p_m)\tau_{j-1}}) \cdot e^{-\tau_{j-1}p_m}\epsilon p_m
\end{equation}
which is a Riemann sum. Taking the limit we get the following integral
\begin{equation}
    p(X)=p_m\frac{n-2}{n-1}\int_0^T (1-e^{-(1-p_m)x}) \cdot e^{-x p_m}dx ,
\end{equation}
whose solution is 
\begin{equation}
    p(X)=p_m\frac{n-2}{n-1}(e^{-T}-\frac{1}{p_m}e^{-T p_m}+\frac{1}{p_m}-1).
\end{equation}
Setting $t=\sigma_\tau$, we get 
\begin{equation}
    p(X)=\frac{n-2}{n-1}(e^{-\sigma_1}p_m-e^{-\sigma_1 p_m}+1-p_m).
\end{equation}
We can see that when $\sigma_1$ and $\sigma_1 p_m$ are relatively large then $p(X)\approx 1-p_m$. That is, the ratio of tokens moving before absorption is $p(X)\approx 1-p_m$, or in other words the probability of tokens being masked without ever moving is $p_m$ which is precisely the probability of a token being masked in our transition-rate matrix $\mQ^{tok}$. Thus the last row of our matrix directly controls the percentage of corrected tokens in the reverse process for large enough diffusion times.

\subsubsection{Time Evolving Roulette}\label{EVroulette_appendix}

Here, we generalize the results of the previous subsections for the case that $p_m$ varies with respect to time. We start by defining 
\begin{equation}\label{eroulette_def}
    \mQ_{eroulette}^{tok}(t)= [(1-p_m(t))\sigma(t)]^{'}_{t}\mQ_{uniform}^{tok}+[p_m(t)\sigma(t)]^{'}_{t} \mQ_{absorb}^{tok}.
\end{equation}

First, we point out that since the size of $\mQ_{eroulette}$ is $(n\times n)$ and that of $\mQ_{uniform}$ is $(V\times V)$, where $V=n-1$, in order for them to have the same size, we add a row of zeros and a column of zeros to the bottom and right of $\mQ_{uniform}$ respectively. Clearly, by definition, these are transition rate matrices, and therefore so is $\mQ_{eroulette}^{tok}(t)$. Indeed, the elements in each of its columns add to $0$, and the only negative elements are in the diagonals. If we define, $p_m(t)$ such that $p_m(0)=0$ and $p_m(1)=1$, then the limiting distribution will be the one-hot encoding at the absorb state. Furthermore, we can compute the exponential matrix of $\mQ_{eroulette}^{tok}(t)$. It is easy to prove that 
$(1-p_m(t))\sigma(t)\mQ_{uniform}^{tok}$ and $p_m(t)\sigma(t) \mQ_{absorb}^{tok}(t)$ commute with each other, thus
\begin{equation}
    e^{\mQ_{eroulette}^{tok}(t)}= e^{(1-p_m(t))\sigma(t)\mQ_{uniform}^{tok}+p_m(t)\sigma(t) \mQ_{absorb}^{tok}}.
\end{equation}
\begin{equation}
 = e^{p_m(t) \sigma_t \mQ_{absorb}^{tok}}e^{(1-p_m(t))\sigma_t\mQ_{uniform}^{tok}}.
\end{equation}
Writing $\alpha_t = p_m(t) \sigma_t$ and $\beta_t = (1-p_m(t))\sigma_t$, shows that we can calculate each of these exponential matrices using the strategy below
\begin{equation}\label{good_exponential}
    e^{c(t)\mQ^{tok}}=\mI+\sum_{k=1}^\infty \frac{(\mQ^{tok})^k c(t)^k}{k!}=\mI+\sum_{k=1}^\infty \frac{(-1)^{k+1} c(t)^k\mQ^{tok}}{k!}=\mI+\mQ^{tok}(1-e^{-c(t)}).
\end{equation}
Indeed,
\begin{equation}
    e^{\alpha_t \mQ_{absorb}^{tok}}=\mI+\mQ_{absorb}^{tok}(1-e^{-\alpha_t}) \text{ and } e^{\beta_t \mQ^{tok}}=\mI+\mQ_{uniform}^{tok}(1-e^{-\beta_t}).
\end{equation}
Multiplying them together, we get the following proposition:
\setcounter{theorem}{6} 
\begin{proposition}\label{proposition3}
    If we denote with $\mY_t$ the matrix exponential of $\sigma_t \mQ^{tok}_{eroulette}$, then $\mY_t(i\not\in\{j,n\}, j\neq n)=e^{-\sigma_t  p_m(t)}\frac{1}{n-1}(1-e^{-(1-p_m(t))\sigma_t })$, $\mY_t(i\neq n,i\neq n)=e^{-\sigma_t  p_m(t)}(1-\frac{n-2}{n-1}(1-e^{-(1-p_m(t))\sigma_t }))$, $\mY_t(n,j\neq n)=1-e^{-\sigma_t  p_m(t)}$, $\mY_t(i\neq n, n)=0$, and $\mY_t(n, n)=1$.
\end{proposition}
One can see that this is almost identical to Proposition \ref{proposition2}, with the only difference being that $p_m$ varies with $t$.
Indeed, if we fix $p_m$ so that it is constant with respect to time, Equation (\ref{eroulette_def}) becomes 
\begin{equation}
    \mQ_{eroulette}^{tok}(t)= p_m \mQ_{absorb}^{tok}(t) + (1-p_m)\mQ_{uniform}^{tok}(t)=\mQ_{roulette}^{tok}(t),
\end{equation}
meaning that $\mQ_{eroulette}^{tok}(t)$, coincides with $\mQ_{roulette}^{tok}(t)$ roulette in this case. This highlights that the roulette diffusion is an interpolation between the roulette and uniform diffusion.

One possible choice of $p_m(t)$ is $t^{\frac{1}{a t}}$ for a positive constant $a$.

\section{Experimental Details}

\subsection{Algorithms}\label{algorithms_appendix}
\begin{algorithm}[H]
\caption{Cross Entropy Training Algorithm}
\label{alg:entropy_training}
\begin{algorithmic}
\Require Network $f_\theta$, (total) noise schedule $\sigma_t$, data distribution $p_{\text{data}}$, token transition matrix $\mQ^{tok}$, and time $t\in [0, 1]$.
\State \textbf{Sample} $\vx_0 \sim p_0$, $t \sim \mathcal{U}([0, 1])$.
\State \textbf{Construct} $\vx_t$ from $\vx_0$. In particular, $\vx_t^i \sim p_{t|0}(\cdot|\vx_0^i) = \exp(\sigma_t \mQ^{tok})_{[:, \vx_0^i]}$.
\State \textbf{Compute} $L_{ll}=-\sum_{i=1}^L \log f_\theta^i(\vx_t, t)[\vx_0^i].$
\State \textbf{Backpropagate} $\nabla_\theta L_{ll}$. 
\State \textbf{Run} optimizer.
\end{algorithmic}
\end{algorithm}

\subsection{Network Architecture and Hyper-parameters}\label{nahp}
This core model is grounded in the diffusion transformer architecture introduced by \cite{Peebles_2023_ICCV}, which integrates time conditioning into the conventional encoder-only transformer framework as established by \citet{vaswani2017attention, devlin-etal-2019-bert}. However, it includes minor adjustments, such as the use of rotary positional encoding \citep{su2024roformer}. The model contains approximately 5\% more parameters than a standard transformer (utilized in the case of GPT-2), attributed to the incorporation of time conditioning. Additionally, the same tokenizers and data splits as in prior work are employed to avoid introducing artifacts. 

The network is configured with 12 transformer blocks, each featuring 12 attention heads and a hidden size of 768, aligning with the "small" variant of GPT-2. It includes conditioning dimensions set at 128 to facilitate the diffusion process by encoding time-dependent features. Notably, the architecture excludes masking, typical of generative models that generate all tokens simultaneously rather than sequentially. It uses standard scaled dot-product attention mechanisms and incorporates a dropout rate of 0.1 to mitigate overfitting.

In terms of hyperparameters, the model was trained on a single H100 when the sequence length is set at 128, while in the case of sequence lengths of 1024 the model is trained using 8$\times$H100 with a vocabulary size of 50,257 tokens. Training involves a batch size of 512. The training regime is designed for a total of 400,000 iterations.

Training utilizes the OpenWebText dataset, while evaluation is conducted on WikiText-103, with data managed locally to speed up access times. The noise schedule for the diffusion process is log-linear (uniform, absorb), and roulette log-linear (roulette) controlling the variance of noise added incrementally. In both cases we set $\epsilon=0.001$ as in \citep{lou2023discrete}. Sampling for evaluation during training employs an Euler predictor over 128 (and 1024 when $L=1024$) steps, with noise removal enabled.

For optimization, the model uses the AdamW optimizer with a learning rate of 0.0003, beta parameters of 0.9 and 0.999, and epsilon set to 1e-8. It features no weight decay, focusing on adapting learning without additional regularization. The optimizer includes a warm-up phase of 2,500 steps to stabilize learning dynamics, and employs gradient clipping at a threshold of 1 to prevent gradients from exploding during training. The log-linear noise schedule was used in the absorb and uniform case, while the roulette log-linear one was used in the roulette case.

\subsection{Efficient Implementation in Practice}

\subsubsection{Efficient Estimation of the Score in Practice}\label{efficient_score_matching}

For our choice of sparse matrices $\mQ_t$ which can only modify one position at each step, the loss function
\begin{equation}
    \mathbb{E}_{t\sim U(0,1)} \mathbb{E}_{\vx_0 \sim p_{0}(\vx_0)}\mathbb{E}_{\vx_t \sim p_{t|0}(\cdot|\vx_0)} \sum_{\vy \neq \vx_t} \mQ_t(\vx_t, \vy) \ell\left(\frac{p_{t|0}(\vy|\vx_0)}{p_{t|0}(\vx_t|\vx_0)}, s_\theta(\vx_t, t)_\vy\right),
\end{equation}
becomes 
\begin{equation}
    \mathbb{E}_{t\sim U(0,1)} \mathbb{E}_{\vx_0 \sim p_{0}(\vx_0)}\mathbb{E}_{\vx_t \sim p_{t|0}(\cdot|\vx_0)} \sum_{i=1}^L \sum_{\vy^i\neq \vx^i_t}\mQ_t^{tok}(\vx^i_t,\vy^i)\ell\left(\frac{p_{t|0}(\vy^i|\vx^i_0)}{p_{t|0}(\vx^i_t|\vx^i_0)}, s^i_\theta(\vx_t, t)[\vy^i]\right).
\end{equation}
Now we focus on calculating the loss at position $i$, as the total loss across the sequence, will be simply the sum of such individual losses at each position. 
We notice that the term $\sum_{\vy^i\neq \vx^i_t}\mQ_t^{tok}(\vx^i_t,\vy^i) s^i_\theta(\vx_t, t)[\vy^i]   $ in the expression above is trivial to calculate. We simply add the exponentiated outputs of the neural network across the last dimension weighted by $\mQ_t^{tok}(\vx^i_t,\vy^i)$, and then substract $s_\theta(\vx_t, t)_{\vx_t^i}^i \cdot \mQ_t^{tok}(\vx^i_t,\vx^i_t)$. Since in all cases the terms of $\mQ_t^{tok}$ are either mostly 0 or mostly the same, this can be done efficiently. The expression
\begin{equation}
    \sum_{\vy^i\neq \vx^i_t}\mQ_t^{tok}(\vx^i_t,\vy^i) \frac{p_{t|0}(\vy^i|\vx^i_0)}{p_{t|0}(\vx^i_t|\vx^i_0)}\log s^i_\theta(\vx_t, t)[\vy^i]
\end{equation}
is more challenging to be calculated efficiently. We follow the approach of \citep{lou2023discrete}, when using SEDD for training. To reiterate, for each position $i$, if we use SEDD for training, we need to calculate the sum of the product $\frac{p_{t|0}(\vy^i|\vx^i_0)}{p_{t|0}(\vx^i_t|\vx^i_0)}\log s^i_\theta(\vx_t, t)[\vy^i]$ over $n$ ratios, where $n=V$ if $\mQ^{tok}=\mQ^{tok}_{uniform}$ and $n=V+1$ otherwise. Luckily, for the choices of $\mQ^{tok}$ in this paper (uniform, absorb and roulette), such ratios have relatively simple form and they are almost all the same. For example, in the case of the \textbf{absorb diffusion}, there are two cases, either $\vx_t^i$ has not moved ($\vx^i_t=\vx_0^i$) or it has been masked ($\vx^i_t=n\neq \vx^i_0$), which can be seen in Figure \ref{absorb_ratios_fig}.

\begin{figure}[H]
    \centering
    \includegraphics[width=1\textwidth]{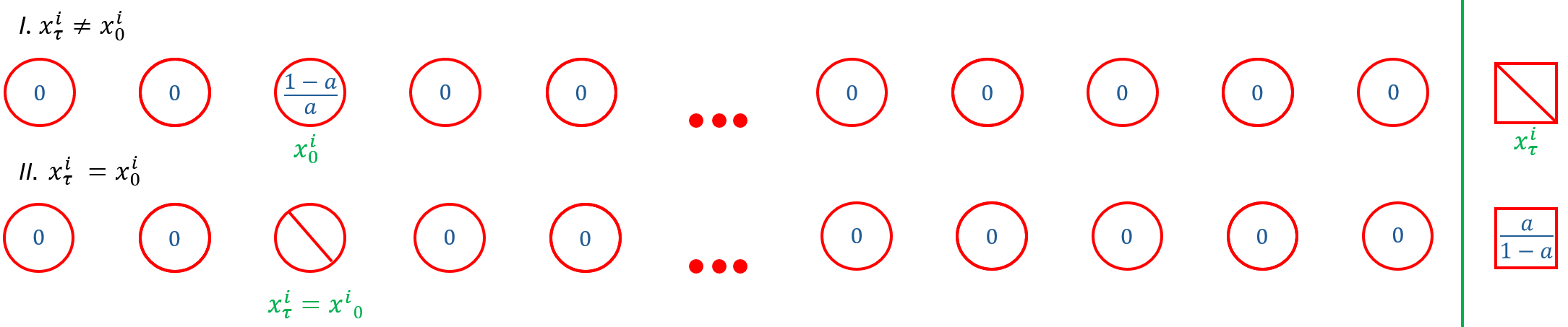}
    \caption{The conditional ratios at position $i$ over the vocabulary in the two cases. The square represents the absorb state. The value of $a$ is   $1-e^{-\sigma_t}$.}
    \label{absorb_ratios_fig}
\end{figure}

To simplify notation, when iterating through different tokens in the vocabulary we write $\vx^i_0=k$, $\vy^i=j$ and $\log s^i_\theta(\vx_t, t)[\vy^i]=s_j$. 
In the first case, when $\vx^i_t=n\neq \vx^i_0$ we get 
\begin{equation}
    \sum_{\vy^i\neq \vx^i_t}\mQ_t^{tok}(\vx^i_t,\vy^i) \frac{p_{t|0}(\vy^i|\vx^i_0)}{p_{t|0}(\vx^i_t|\vx^i_0)} \log s^i_\theta(\vx_t, t)[\vy^i] = \frac{1-a}{a}s_k.
\end{equation}
On the other hand, in the second case ($\vx_t^i= \vx_0^i$) one has
\begin{equation}
    \sum_{\vy^i\neq \vx^i_t}\mQ_t^{tok}(\vx^i_t,\vy^i) \frac{p_{t|0}(\vy^i|\vx^i_0)}{p_{t|0}(\vx^i_t|\vx^i_0)} \log  s^i_\theta(\vx_t, t)[\vy^i] =
\end{equation}
\begin{equation}
    \sum_{j\not\in \{\vx^i_t, n\}} \mQ_t^{tok}(\vx^i_t, j)0s_j+\mQ_t^{tok}(\vx^i_t, n)\frac{a}{1-a}s_n=0,
\end{equation}
as $\mQ_t^{tok}(\vx^i_t, n)=0$. Thus when $\vx_t^i=\vx_0^i\neq n$, we can simply not calculate the loss. 
\newline
For the \textbf{uniform diffusion}, there are also two cases, namely $\vx^i_t=\vx_0^i$ and $\vx^i_t\neq \vx^i_0$ as illustrated in Figure \ref{uniform_ratios_fig}.

\begin{figure}[H]
    \centering
    \includegraphics[width=1\textwidth]{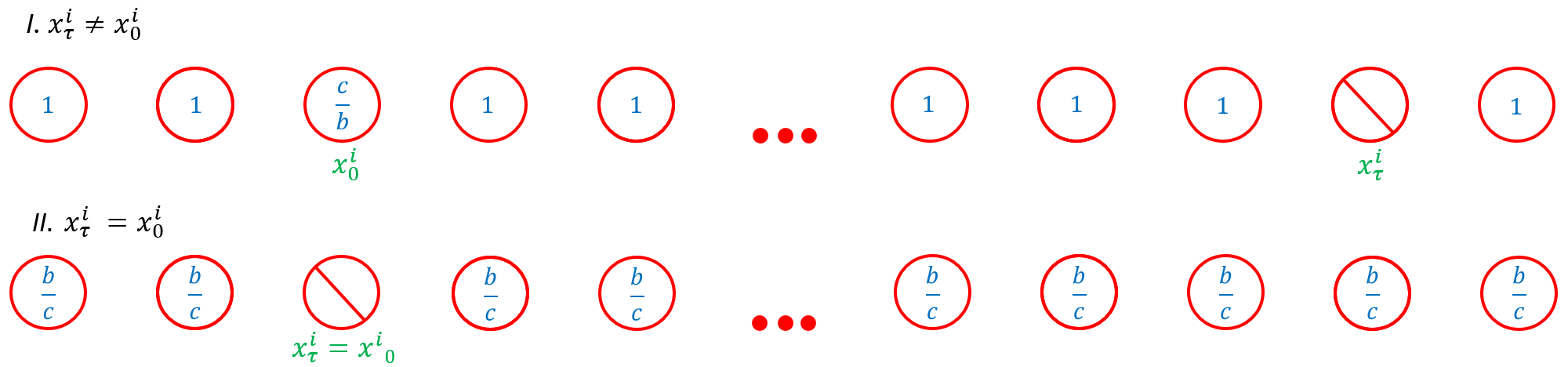}
    \caption{The conditional ratios at position $i$ over the vocabulary in the two cases. We define $b=\frac{1}{n=V}(1-e^{-\sigma_t})$ and $c=1-(n-1)b$.}
    \label{uniform_ratios_fig}
\end{figure}

 If $\vx^i_t\neq \vx^i_0$, then 
 \begin{equation}
     \sum_{\vy^i\neq \vx^i_t}\mQ_t^{tok}(\vx^i_t,\vy^i) \frac{p_{t|0}(\vy^i|\vx^i_0)}{p_{t|0}(\vx^i_t|\vx^i_0)} s^i_\theta(\vx_t, t)[\vy^i] = \sum_{j\not\in \{\vx^i_t,k\}} \mQ_t^{tok}(\vx^i_t, j)s_j+\mQ_t^{tok}(\vx^i_t, k)\frac{c}{b}s_{k}
 \end{equation}
 \begin{equation}
  =\sum_{j\neq \vx^i_t} \mQ_t^{tok}(\vx^i_t, j)s_j+\mQ_t^{tok}(\vx^i_t, k)\left(\frac{c}{b}-1\right)s_k=\sum_{j\neq \vx^i_t} \frac{1}{n}s_j+\frac{1}{n}\left(\frac{c}{b}-1\right)s_k.
 \end{equation}
 Otherwise for $\vx^i_t= \vx^i_0$, we have 
 \begin{equation}
     \sum_{\vy^i\neq \vx^i_t}\mQ_t^{tok}(\vx^i_t,\vy^i) \frac{p_{t|0}(\vy^i|\vx^i_0)}{p_{t|0}(\vx^i_t|\vx^i_0)} s^i_\theta(\vx_t, t)[\vy^i] =\frac{b}{c} \sum_{j\not\in \{\vx^i_t\}} \mQ_t^{tok}(\vx^i_t, j)s_j=\frac{1}{n}\frac{b}{c}\sum_{j\neq \vx^i_t} s_j.
 \end{equation}
 Finally writing $S^i=\sum_j s_j$, we get $ \frac{1}{n}S^i +\frac{1}{n}\left(\frac{c}{b}-1\right)s_k-\frac{1}{n}s_{\vx^i_t}$ in the first case, and $\frac{1}{n}\frac{b}{c}S^i-\frac{1}{n}\frac{b}{c}s_{\vx^i_t}$ in the second one.
 \newline
For the \textbf{roulette diffusion} we proceed similarly. Figure \ref{roulette_ratios_fig} shows the ratios in all three possible cases.
\begin{figure}[H]
    \centering
    \includegraphics[width=1\textwidth]{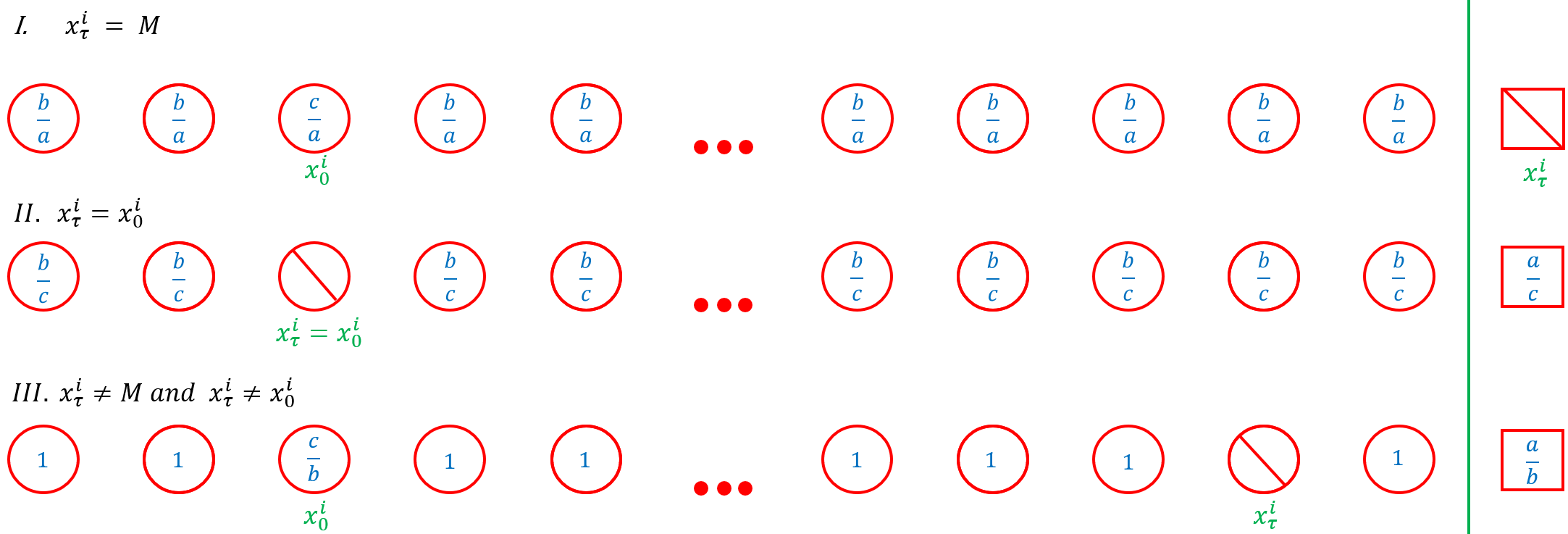}
    \caption{The conditional ratios at position $i$ over the vocabulary in the three cases. Since $\vx^i_0$ is a token from the data it cannot be masked. The square represents the absorb state. We define $a=1-e^{-p_m\sigma_t}$, $b =e^{-p_m\sigma_t} \cdot \frac{1-e^{-\sigma_t(1-p_m)}}{n-1}$ and $c=e^{-p_m\sigma_t}\left[\frac{1}{n-1}+\left(1-\frac{1}{n-1}\right) e^{-\sigma_t(1-p_m)} \right]$.} 
    \label{roulette_ratios_fig}
\end{figure}
In the first case, the sum at position $i$ is $p_m\left[\frac{b}{a}\left(S^i -s_k-s_{\vx^i_t}\right)+\frac{c}{a}s_k\right]$. In the second it is $\frac{1-p_m}{V}\frac{b}{c} \left(S^i-s_k-s_n \right),$ while in the third case we have $\frac{1-p_m}{V}\left(S^i -s_k -s_n-s_{\vx^i_t}\right)+\frac{1-p_m}{V}s_k \frac{c}{b}$.
\newline
If one chooses to incorporate the constant term $K(a)$, in the loss, then a similar approach is followed. However, considering Theorem \ref{theorem4} for evaluation, and the fact that a constant term does not affect the gradient during training, implementing this term is not necessary.
\subsubsection{SEDD Scaling}\label{sedd_scaling_appendix}

During training, some scores $\frac{p_{t|0}(\vy^i|\vx^i_0)}{p_{t|0}(\vx^i_t|\vx^i_0)}$, on expectation with respect to $\vx^i_0$, can vary significantly for different $\vy^i$. This makes the task of the network more complex and reduces performance. In order to alleviate this issue, in the case of the \textbf{absorb diffusion}, \citep{lou2023discrete} add $\log \sum_{\vx^i_0}\frac{1}{n}\frac{p_{t|0}(\vy^i|\vx^i_0)}{p_{t|0}(\vx^i_t|\vx^i_0)}$ to the output of the network $\log s^i_\theta(\vx_t, t)[\vy^i]$. In this case, since we only need to calculate scores when $\vx^i_t=n$, this sum equals 
\begin{equation}
    \sum_{\vx^i_0}\frac{1}{n-1}\frac{p_{t|0}(\vy^i|\vx^i_0)}{p_{t|0}(\vx^i_t|\vx^i_0)}= \sum_{\vx^i_0}\frac{1}{n-1}\frac{p_{t|0}(\vy^i\neq n|\vx^i_0)}{p_{t|0}(n|\vx^i_0)}=0+\frac{1}{n-1}\frac{p_{t|0}(\vy^i\neq n|\vy^i\neq n)}{p_{t|0}(n|\vy^i\neq n)}
\end{equation}
\begin{equation}
    =\frac{1}{n-1}\frac{e^{-\sigma_t}}{1-e^{-\sigma_t}}=\frac{1}{n-1}\frac{1}{e^{\sigma_t}-1}.
\end{equation}
We notice that in effect, this changes the prediction of the score, from  $\log s^i_\theta(\vx_t, t)[\vy^i]$ to $\log s^i_\theta(\vx_t, t)[\vy^i]-\log (n-1) -\log{(e^{\sigma_t}-1)}$.

We extend this strategy for the uniform and the roulette diffusion. In the case of the \textbf{uniform diffusion}, we get  
\begin{equation}
    \sum_{\vx^i_0}\frac{1}{V}\frac{p_{t|0}(\vy^i|\vx^i_0)}{p_{t|0}(\vx^i_t|\vx^i_0)} = \sum_{\vx^i_0\not\in\{\vy^i, \vx^i_t\}}\frac{1}{V}\frac{p_{t|0}(\vy^i|\vx^i_0)}{p_{t|0}(\vx^i_t|\vx^i_0)}+\frac{1}{V}\frac{p_{t|0}(\vy^i|\vy^i)}{p_{t|0}(\vx^i_t|\vy^i)}+\frac{1}{V}\frac{p_{t|0}(\vy^i|\vx^i_t)}{p_{t|0}(\vx^i_t|\vx^i_t)}.
\end{equation}

Now,
\begin{equation}
    \frac{p_{t|0}(\vy^i|\vy^i)}{p_{t|0}(\vx^i_t|\vy^i)}=\frac{1-\frac{V-1}{V}(1-e^{-\sigma_t})}{\frac{1}{V}(1-e^{-\sigma_t})}=1+\frac{V}{e^{\sigma_t}-1},
\end{equation}
and 
\begin{equation}
    \frac{1}{V}\frac{p_{t|0}(\vy^i|\vx^i_t)}{p_{t|0}(\vx^i_t|\vx^i_t)}=\frac{\frac{1}{V}(1-e^{-\sigma_t})}{1-\frac{V-1}{V}(1-e^{-\sigma_t})}=1-\frac{V}{e^{\sigma_t}-1+V}. 
\end{equation}
Since $\frac{p_{t|0}(\vy^i|\vx^i_0\neq \vy^i)}{p_{t|0}(\vx^i_t|\vx^i_0\neq \vx^i_t )}=1$, we conclude
\begin{equation*}
    \sum_{\vx^i_0}\frac{1}{V}\frac{p_{t|0}(\vy^i|\vx^i_0)}{p_{t|0}(\vx^i_t|\vx^i_0)}  = \frac{V-2}{V}+\frac{1}{V}+\frac{1}{e^{\sigma_t}-1}+\frac{1}{V}-\frac{1}{e^{\sigma_t}-1+V}=1+\frac{1}{e^{\sigma_t}-1}-\frac{1}{e^{\sigma_t}-1+V}.
\end{equation*}

We then modify the output of the network $\log s^i_\theta(\vx_t, t)[\vy^i]$ accordingly 
\begin{equation}
    \log s^i_\theta(\vx_t, t)[\vy^i]+\log{(\sum_{\vx^i_0}\frac{1}{V}\frac{p_{t|0}(\vy^i|\vx^i_0)}{p_{t|0}(\vx^i_t|\vx^i_0)})},\text{ where } n:=V.
\end{equation}
\newline
Similarly, we can calculate this scaling factor for the \textbf{roulette diffusion}. There are two cases, either $\vx^i_t=n$ or $\vx^i_t\neq n$. In the first case, with little modifications from before we can show that 
\begin{equation}
    \sum_{\vx^i_0}\frac{1}{n-1}\frac{p_{t|0}(\vy^i\neq n|\vx^i_0)}{p_{t|0}(n|\vx^i_0)}=\frac{1}{n-1}\frac{1-(1-e^{-\sigma_t p_m})}{1-e^{-\sigma_t p_m}}=\frac{1}{n-1}\frac{1}{e^{\sigma_t p_m}-1}.
\end{equation}
In the second case,
\begin{equation}
    \sum_{\vx^i_0}\frac{1}{n-1}\frac{p_{t|0}(\vy^i\neq n|\vx^i_0)}{p_{t|0}(\vx^i_t\neq n|\vx^i_0)}=1+\frac{1}{e^{(1-p_m)\sigma_t}-1}-\frac{1}{e^{(1-p_m)\sigma_t}-1+n-1}.
\end{equation}
We then modify the output of the network $\log s^i_\theta(\vx_t, t)[\vy^i]$ accordingly 
\begin{equation}
    \log s^i_\theta(\vx_t, t)[\vy^i]+\log{(\sum_{\vx^i_0}\frac{1}{n}\frac{p_{t|0}(\vy^i|\vx^i_0)}{p_{t|0}(\vx^i_t|\vx^i_0)})}.
\end{equation}
We take a moment to recall that if $\vy$ and $\vx_t$ differ only at one position $i$ then 
\begin{equation}
    \frac{p_t(\vy)}{p_t(\vx_t)}=\sum_{\vx^i_0}p^i(\vx_0^i|\vx_t)\frac{p_{t|0}(\vy^i|\vx^i_0)}{p_{t|0}(\vx^i_t|\vx^i_0)}.
\end{equation}
Thus, above we are scaling the output by a naive estimation of this expectation where $p^i(\vx_0^i|\vx_t)$ is assumed to be $\frac{1}{n-1}$. In fact, the study of the scaling approach and the realization that the entity $\frac{1}{n-1}$ can be more properly estimated, is what originally motivated the derivation and usage of CEDD in this paper.

\subsubsection{Score Estimation Through CEDD in Practice}\label{prob_to_score_appendix}

In this subsection, similarly to Section \ref{efficient_score_matching}, we show how to convert the learned probabilities into learned ratios in practice. Suppose that the perturbed token at position $i$ at time $t$ is $\vx^i_t$. If $\vy$ is a sequence that differs from $\vx_t$ only at position $i$ then from before one can write:
\begin{equation}
    \frac{p_t(\vy)}{p_t(\vx_t)}=\sum_{\vx^i_0}p^i(\vx_0^i|\vx_t)\frac{p_{t|0}(\vy^i|\vx^i_0)}{p_{t|0}(\vx^i_t|\vx^i_0)}.
\end{equation}
We point out that since $\vx_t$ and $\vy$ only differ at position $i$, then $\vx^i_t\neq \vy^i$. Thus, 
\begin{equation}\label{cedd_efficient}
    \frac{p_t(\vy)}{p_t(\vx_t)}=\sum_{\vx^i_0\not\in\{\vx^i_t, \vy^i\}}p^i(\vx_0^i|\vx_t)\frac{p_{t|0}(\vy^i|\vx^i_0)}{p_{t|0}(\vx^i_t|\vx^i_0)}+p^i(\vy^i|\vx_t)\frac{p_{t|0}(\vy^i|\vy^i)}{p_{t|0}(\vx^i_t|\vy^i)}+p^i(\vx_t^i|\vx_t)\frac{p_{t|0}(\vy^i|\vx^i_t)}{p_{t|0}(\vx^i_t|\vx^i_t)} .
\end{equation}

In the case of the \textbf{absorb diffusion}, from Section \ref{efficient_score_matching}, we can see that we do not need the ratio when $\vy^i=n$, and furthermore, we only need the ratios when $\vx_t^i=n$. Since also the data is not masked, then $\vx^i_0\neq n$. Therefore, the first term (the sum) in Equation (\ref{cedd_efficient}) is $0$, since $\frac{p_{t|0}(\vy^i\neq n|\vx^i_0\not\in\{n, \vy^i\})}{p_{t|0}(n|\vx^i_0\not\in\{n, \vy^i\})}=0$. Hence 
\begin{equation}
    \frac{p_t(\vy)}{p_t(\vx_t)}=p^i(\vy^i|\vx_t)\frac{p_{t|0}(\vy^i\neq n|\vy^i\neq n)}{p_{t|0}(n|\vy^i\neq n)}+p^i(\vx_t^i|\vx_t)\frac{p_{t|0}(\vy^i\neq n|n)}{p_{t|0}(n|n)},
\end{equation}
where $p_{t|0}(\vy^i\neq n|n)=0$ since a masked token cannot be unmasked in the forward process. Thus our approximation of the score is simply 
\begin{equation}
    s^i_\theta(\vx_t, t)[\vy^i]=f^i_\theta(\vx_t, t)[\vy^i]\frac{e^{-\sigma_t}}{1-e^{-\sigma_t}}.
\end{equation}
To conclude $s_\theta(\vx_t, t)=f_\theta(\vx_t, t)\frac{1}{e^{\sigma_t}-1}$.
\newline
\newline
In the case of the \textbf{uniform diffusion}, we write $b=\frac{1}{n=V}(1-e^{-\sigma_t})$ and $c=1-(n-1)b$ as in Section \ref{efficient_score_matching}, and Equation (\ref{cedd_efficient}) becomes
\begin{equation}
    \frac{p_t(\vy)}{p_t(\vx_t)}=\sum_{\vx^i_0\not\in\{\vx^i_t, \vy^i\}}p^i(\vx_0^i|\vx_t)+p^i(\vy^i|\vx_t)\frac{c}{b}+p^i(\vx_t^i|\vx_t)\frac{b}{c} .
\end{equation}
\begin{equation}
    \frac{p_t(\vy)}{p_t(\vx_t)}=\sum_{\vx^i_0}p^i(\vx_0^i|\vx_t)+p^i(\vy^i|\vx_t)\left(\frac{c}{b}-1\right)+p^i(\vx_t^i|\vx_t)\left(\frac{b}{c}-1\right) .
\end{equation}Therefore, our approximation of the score is simply 
\begin{equation}
    s^i_\theta(\vx_t, t)[\vy^i]=1+f^i_\theta(\vx_t, t)[\vy^i]\left(\frac{c}{b}-1\right)+f^i_\theta(\vx_t, t)[\vx_t^i]\left(\frac{b}{c}-1\right),
\end{equation} or 
\begin{equation}
    s^i_\theta(\vx_t, t)=\1+f^i_\theta(\vx_t, t)\left(\frac{c}{b}-1\right)+f^i_\theta(\vx_t, t)[\vx_t^i]\left(\frac{b}{c}-1\right)\1,
\end{equation} in vector form.

Finally, in the case of \textbf{roulette diffusion}, we can discern two possibilities, namely $\vx^i_t=n$ and $\vx^i_t\neq n$.
In the first one, writing as in Section \ref{efficient_score_matching},  $a=1-e^{-p_m\sigma_t}$, $b =e^{-p_m\sigma_t} \cdot \frac{1-e^{-\sigma_t(1-p_m)}}{n-1}$ and $c=e^{-p_m\sigma_t}\left[\frac{1}{n-1}+\left(1-\frac{1}{n-1}\right) e^{-\sigma_t(1-p_m)} \right]$, Equation \ref{cedd_efficient} becomes
\begin{equation}
    \frac{p_t(\vy)}{p_t(\vx_t)}=\sum_{\vx^i_0\not\in\{\vx^i_t, \vy^i\}}p^i(\vx_0^i|\vx_t)\frac{b}{a}+p^i(\vy^i|\vx_t)\frac{c}{b}=\frac{b}{a}+p^i(\vy^i|\vx_t)\left(\frac{c}{b}-\frac{b}{a}\right).
\end{equation}
Hence, if the perturbed token $\vx_t^i$ at position $i$ is masked, the predicted ratios over the vocabulary are:
\begin{equation}
    s^i_\theta(\vx_t, t)=\1\frac{b}{a}+f^i_\theta(\vx_t, t)\left(\frac{c}{b}-\frac{b}{a}\right).
\end{equation}
Very similarly, if the perturbed token $\vx_t^i$ at position $i$ is not masked, the predicted ratios over the vocabulary are:
\begin{equation}
    s^i_\theta(\vx_t, t)=\1+f^i_\theta(\vx_t, t)\left(\frac{c}{b}-1\right)+f^i_\theta(\vx_t, t)[\vx_t^i]\left(\frac{b}{c}-1\right)\1,
\end{equation}
This shows that we can calculate the ratios by scaling and translating the learned probabilities through CEDD.

\textbf{Modification of the reparametrization in (\ref{reparam_c_to_s}): Re-scaling conditional ratios when $t\rightarrow 0$.}

In the case of the uniform diffusion dynamics, when $t\rightarrow 0$ then $\frac{b}{c}\rightarrow 0$ and $\frac{c}{b}\rightarrow \infty$. This does not cause any issues when sampling, however, it does negatively impact the perplexity bound. We explain informally the issue below. We fix position $i$ and assume we are at $\vx_t^i$ and that originally we were at $\vx_0^i$. When the time $t$ is close to $0$, the model $f^i_\theta(\vx_t, t)[\vy^i]$ is typically very confident and can predict that some token $\vy^i$ is $\vx_0^i$, that is, $f^i_\theta(\vx_t, t)[\vy^i]\approx 1-\epsilon$ and $f^i$ is $\approx \frac{\epsilon}{V-1}$ everywhere else. Furthermore, most of the time $\vx_t^i = \vx_0^i$. This means the score $s^i_\theta(\vx_t, t)[\vy^i]\approx \frac{c}{b}$ and is $\approx \frac{b}{c}$ everywhere else. If the prediction is correct then the loss becomes $\approx \frac{1}{V-1}(V-2)[-\frac{b}{c}\log(\frac{b}{c})+\frac{b}{c}+\frac{b}{c}(\log(\frac{b}{c})-1)]+[-\frac{c}{b}\log(\frac{c}{b})+\frac{c}{b}+\frac{c}{b}(\log(\frac{c}{b})-1)]\approx 0$. If it misses however, then  we have $\approx \frac{1}{V-1}(V-3)[-\frac{b}{c}\log(\frac{b}{c})+\frac{b}{c}+\frac{b}{c}(\log(\frac{b}{c})-1)]+[-\frac{b}{c}\log(\frac{c}{b})+\frac{c}{b}+\frac{b}{c}(\log(\frac{b}{c})-1)]+[-\frac{c}{b}\log(\frac{b}{c})+\frac{b}{c}+\frac{c}{b}(\log(\frac{c}{b})-1)]\approx 2\frac{c}{b}\log \frac{c}{b}$. Considering the large magnitude of $\frac{c}{b}$ when $t\approx 0$, this loss is extremely punitive. SEDD does not suffer from this issue, as the conditional ratios $\frac{c}{b}$ are implicitly learned by the network, which due to its limited flexibility likely acts as a regularizer, not allowing the values of the ratios to rise steeply as $t\rightarrow 0$. Motivated by this, in the uniform case, we rescale $\sigma_t<0.0015$, by setting $\sigma_t=0.0015$. This significantly reduces the magnitude of $\frac{c}{b}$ as $t\rightarrow 0$.  We highlight that only the $\sigma_t$ that is used to calculate $\frac{c}{b}$ and $\frac{b}{c}$ are scaled while the $\sigma_t$ that is fed to the model is not touched. Naturally, for the sake of rigor, this is also the model we employ to generate samples. Finally, we note that this problem also appears in the case of the roulette diffusion dynamics, as when time is close to $0$ most tokens are unmasked. The same strategy is applied, as before, only to ratios $\frac{c}{b}$ and $\frac{b}{c}$, by rescaling $\sigma_t$ when $\sigma_t<0.5$ as follows: $\sigma^{scaled}_t = \log(1.1\sigma_t+1.1)$.

\subsubsection{Analytic Sampling in Practice}

In \citep{lou2023discrete} an alternative sampling scheme (Equation (18)) is provided. This method is called the analytic method and it performs better than Euler sampling, in particular when the number of sampling steps is small:

\begin{equation}
    p_{t-\epsilon|t}(\vx_{t-\epsilon}^i|\vx^i_t) = \left(e^{\sigma_t^{\Delta t} \mQ^{tok}}(\vx_t^i, \vx_{t-\epsilon}^i)\right)\sum_{\vy^i=1}^n \left(e^{-\sigma_t^{\Delta t} \mQ^{tok}}(\vx_{t-\epsilon}^i, \vy^i)\right)s^i_\theta(\vx_t, t)[\vy^i],
\end{equation}
where $\sigma_t^{\Delta t}=\sigma_t-\sigma_{t-\epsilon}$.

Since we have an analytic expression of the matrix exponential $e^{\sigma_t^{\Delta t} \mQ^{tok}}$, we can easily derive the expression of $e^{-\sigma_t^{\Delta t} \mQ^{tok}}$, by simply substituting $\sigma_t^{\Delta t}$ with $-\sigma_t^{\Delta t}$ in each entry. One can efficiently calculate the sum above by following the strategy of \citep{lou2023discrete} as below. 
\newline
\newline
In the case of the \textbf{absorb diffusion}, $e^{-\sigma_t^{\Delta t} \mQ^{tok}}(j\neq n, j\neq n)=1-\bar{a}$,  $e^{-\sigma_t^{\Delta t} \mQ^{tok}}( n, n)=1$, $e^{-\sigma_t^{\Delta t} \mQ^{tok}}( n, j\neq n)=\bar{a}$, with all other entries being $0$, where $\bar{a}=1-e^{\sigma_t^{\Delta t}}$. As such if $\vx_{t-\epsilon}\neq n$
\begin{equation}
    \sum_{\vy^i=1}^n \left(e^{-\sigma_t^{\Delta t} \mQ^{tok}}(\vx_{t-\epsilon}^i, \vy^i)\right)s^i_\theta(\vx_t, t)[\vy^i]=(1-\bar{a})s^i_\theta(\vx_t, t)[\vx_{t-\epsilon}^i],
\end{equation}
and if $\vx_{t-\epsilon}= n$
\begin{equation}
    \sum_{\vy^i=1}^n \left(e^{-\sigma_t^{\Delta t} \mQ^{tok}}(\vx_{t-\epsilon}^i, \vy^i)\right)s^i_\theta(\vx_t, t)[\vy^i]=\bar{a}\sum_{\vy^i=1}^{n-1}s^i_\theta(\vx_t, t)[\vy^i]+s^i_\theta(\vx_t, t)[n]=
\end{equation}
\begin{equation}
    =\bar{a}\sum_{\vy^i=1}^{n}s^i_\theta(\vx_t, t)[\vy^i]+(1-\bar{a})s^i_\theta(\vx_t, t)[n]=\bar{a} S^i_\theta(\vx_t, t)+(1-\bar{a})s^i_\theta(\vx_t, t)[n],
\end{equation} where $S^i_\theta(\vx_t, t)=\sum_{\vy^i=1}^{n}s^i_\theta(\vx_t, t)[\vy^i]$.
\newline
In the case of the \textbf{uniform diffusion} $e^{-\sigma_t^{\Delta t} \mQ^{tok}}(j, j)=\bar{c}$ while the rest of the entries are $\bar{b}$, where  $\bar{b}=\frac{1}{n=V}(1-e^{\sigma_t^{\Delta t}})$ and $\bar{c}=1-(n-1)\bar{b}$. Therefore 
\begin{equation}
    \sum_{\vy^i=1}^n \left(e^{-\sigma_t^{\Delta t} \mQ^{tok}}(\vx_{t-\epsilon}^i, \vy^i)\right)s^i_\theta(\vx_t, t)[\vy^i]=(\bar{c}-\bar{b})s^i_\theta(\vx_t, t)[\vx_{t-\epsilon}^i]+\bar{b} S^i_\theta(\vx_t, t)=
\end{equation}
\begin{equation}
    =\frac{s^i_\theta(\vx_t, t)[\vx_{t-\epsilon}^i]}{ e^{-\sigma_t^{\Delta t}}}+\frac{e^{-\sigma_t^{\Delta t}}-1}{n e^{-\sigma_t^{\Delta t}}} S^i_\theta(\vx_t, t).
\end{equation}
\newline
Finally, in the case of the \textbf{roulette diffusion}, we have $e^{-\sigma_t^{\Delta t} \mQ^{tok}}(j\neq n, j\neq n)=\bar{c}$, $e^{-\sigma_t^{\Delta t} \mQ^{tok}}(j = n, j\neq n)=\bar{a}$, $e^{-\sigma_t^{\Delta t} \mQ^{tok}}(j = n, j = n)=1$, $e^{-\sigma_t^{\Delta t} \mQ^{tok}}(j \neq n, j = n)=0$ and the rest of the entries are $\bar{b}$, where $\bar{a}=1-e^{p_m\sigma_t^{\Delta t} }$, $\bar{b} =e^{p_m\sigma_t^{\Delta t} } \cdot \frac{1-e^{\sigma_t^{\Delta t} }}{n-1}$ and $\bar{c}=e^{p_m\sigma_t^{\Delta t} }\left[\frac{1}{n-1}+\left(1-\frac{1}{n-1}\right) e^{\sigma_t^{\Delta t} (1-p_m)} \right]$. As in the absorb case, there are two cases, the first one being $\vx_{t-\epsilon}\neq n$:
\begin{equation}
    \sum_{\vy^i=1}^n \left(e^{-\sigma_t^{\Delta t} \mQ^{tok}}(\vx_{t-\epsilon}^i, \vy^i)\right)s^i_\theta(\vx_t, t)[\vy^i]=(\bar{c}-\bar{b})s^i_\theta(\vx_t, t)[\vx_{t-\epsilon}^i]+\bar{b}S^i_\theta(\vx_t, t)-\bar{b}s^i_\theta(\vx_t, t)[n],
\end{equation}
while for $\vx_{t-\epsilon}= n$, one derives:
\begin{equation}
    \sum_{\vy^i=1}^n \left(e^{-\sigma_t^{\Delta t} \mQ^{tok}}(\vx_{t-\epsilon}^i, \vy^i)\right)s^i_\theta(\vx_t, t)[\vy^i]=\bar{a}S^i_\theta(\vx_t, t)+(1-\bar{a})s^i_\theta(\vx_t, t)[n]=
\end{equation}
\begin{equation}
(\bar{c}-\bar{b})s^i_\theta(\vx_t, t)[n]+\bar{b}S^i_\theta(\vx_t, t)-\bar{b}s^i_\theta(\vx_t, t)[n] +(\bar{a}-\bar{b})S^i_\theta(\vx_t, t) + s^i_\theta(\vx_t, t)[n] (-\bar{a}+2\bar{b}-\bar{c}+1).
\end{equation}

It is important to remark that in the case of the absorb and roulette diffusion, $s^i_\theta(\vx_t, t)[n]$ plays a role in the quantities above only when $\vx_t^i=n$. In the previous section, it is mentioned that $s^i_\theta(\vx_t, t)[n]$ is not learned, but in this case when $\vx_t^i=n$ is needed. However, this particular case is not problematic as $s^i_\theta(\vx_t, t)[n]=s^i_\theta(\vx_t, t)[\vx_t^i]=1$, thus we set $s^i_\theta(\vx_t, t)[n]$ to $1$ manually.

\subsection{Procedure for generating the plots in Figure \ref{absorb_ratios_fig} }\label{compare_proc_appendix}

 We sample a batch of $16$ test points, and we perturb each datapoint with respect to a different $t$. Then, for example in the absorb case, we calculate the expressions
\begin{equation}\label{J1_sum}
    \sum_{\vy \neq \vx_t} \mQ_t(\vx_t, \vy) \ell\left(\frac{p_{t|0}(\vy|\vx_0)}{p_{t|0}(\vx_t|\vx_0)}, s_\theta(\vx_t, t)_\vy\right),
\end{equation}
and 
\begin{equation}
    \sum_{\vy \neq \vx_t} \mQ_t(\vx_t, \vy) \bar{\ell}\left(\frac{p_{t|0}(\vy|\vx_0)}{p_{t|0}(\vx_t|\vx_0)}, s_\theta(\vx_t, t)_\vy\right)-1+\epsilon,
\end{equation} respectively for $J_1$ and $J_2$, where $\epsilon=10^{-4}$. In each case, this returns a single point, therefore we repeat this procedure $64N$ times, where $N$ is the number of testing points. This ensures that in all cases ($J_1$ and $J_2$), we are estimating the loss in the same number of perturbed points. However, when $N$ is small, we keep testing $J_1$ and $J_2$ (extending their plots) for more points in order to provide to the reader more information about their (limiting) behaviour.

\begin{table}[H]
\caption{Results comparing the roulette transition-matrices for different $p_m$ in terms of generative perplexity.}
\label{roulette_gen_tab}
\begin{center}
\begin{tabular}{lllll} 
& $p_m=0.95$ & $p_m=0.65$ & $p_m=0.35$  & $p_m=0.05$\\
\hline
Roulette & 72.31 & 85.72 & 124.19 & 284.55\\
\hline
\end{tabular}
\end{center}
\end{table}

\subsection{Per sequence estimation of $J_1$}\label{Naive-J_1}

Here we explain a different method of calculating $J_1$. When applying this method one first samples a point (sequence) from the test set, and perturbs it for (say) 1024 time values $t$. Then the 1024 perturbed points are separated into (e.g.) 64 batches of size 16, and Expression \ref{J1_sum} is computed for each batch which returns 64 values. This process is then repeated for the next test point (sequence). If the test set has N points then the process produces 64N values. These values can be averaged, divided by $L$ and exponentiated.

\subsection{Additional results}\label{additional_res_appendix}

\begin{figure}[H]
\centering

\begin{subfigure}[b]{0.25\textwidth}
    \centering
    \includegraphics[scale=0.24]{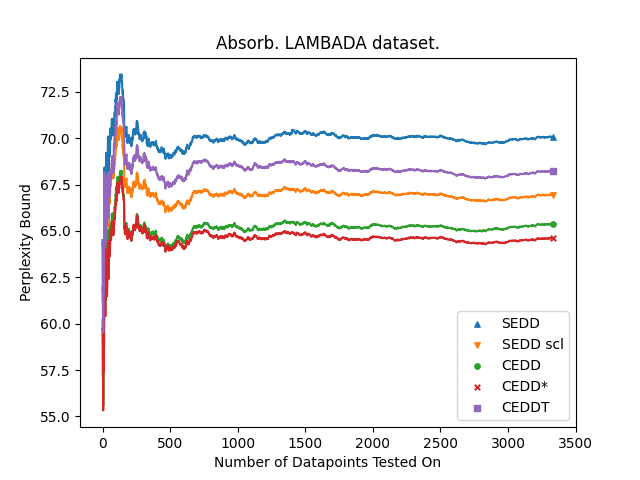}
\end{subfigure}%
\hfill%
\begin{subfigure}[b]{0.25\textwidth}
    \centering
    \includegraphics[scale=0.24]{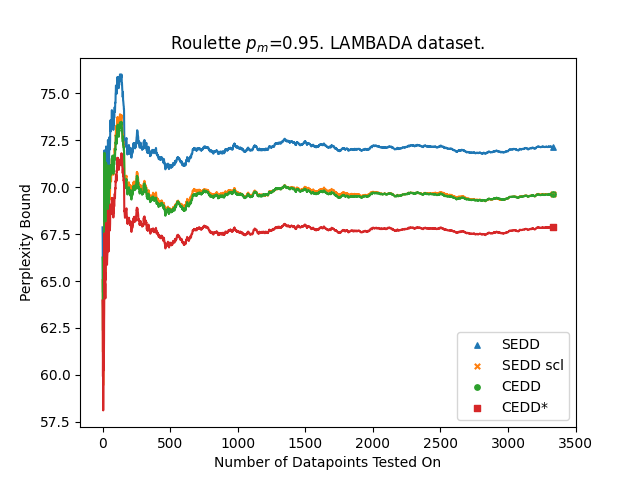}
\end{subfigure}%
\hfill%
\begin{subfigure}[b]{0.25\textwidth}
    \centering
    \includegraphics[scale=0.24]{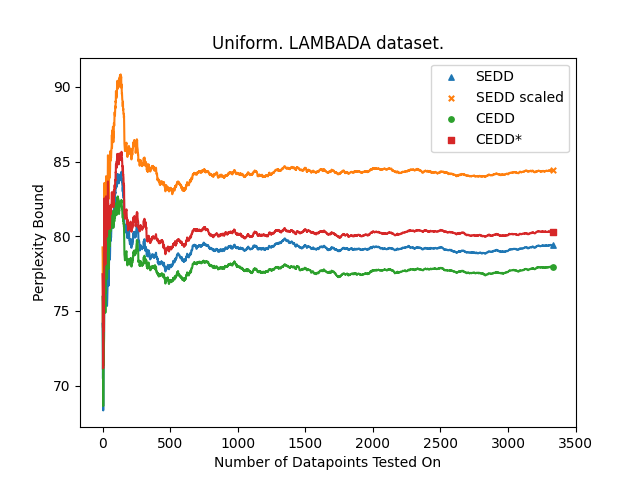}
\end{subfigure}%

\begin{subfigure}[b]{0.25\textwidth}
    \centering
    \includegraphics[scale=0.24]{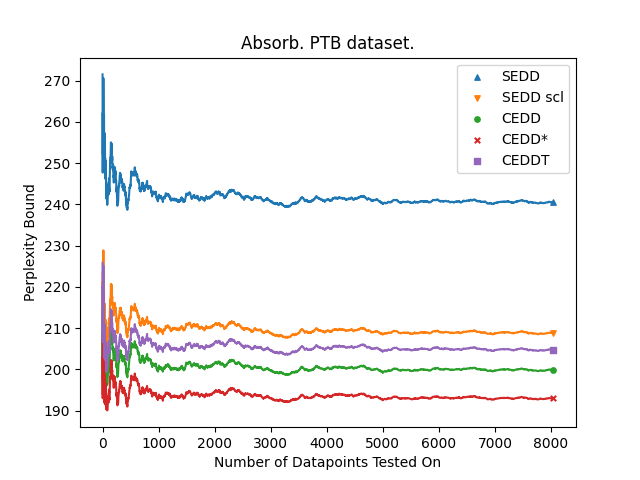}
\end{subfigure}%
\hfill%
\begin{subfigure}[b]{0.25\textwidth}
    \centering
    \includegraphics[scale=0.24]{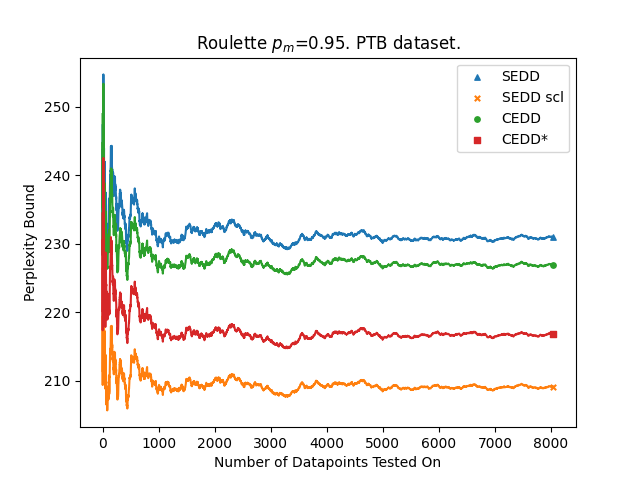}
\end{subfigure}%
\hfill%
\begin{subfigure}[b]{0.25\textwidth}
    \centering
    \includegraphics[scale=0.24]{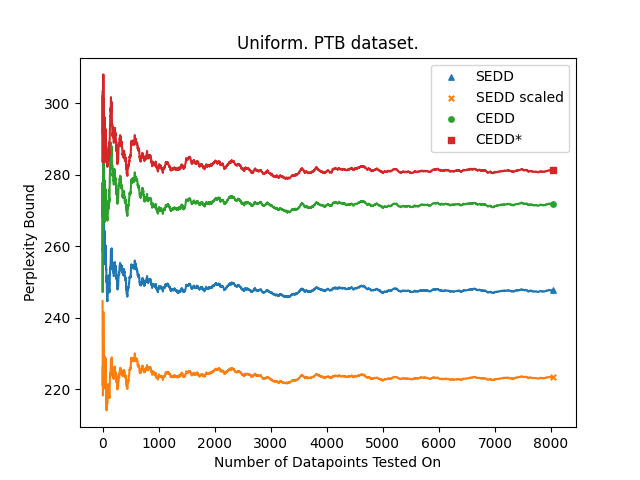}
\end{subfigure}%

\begin{subfigure}[b]{0.25\textwidth}
    \centering
    \includegraphics[scale=0.24]{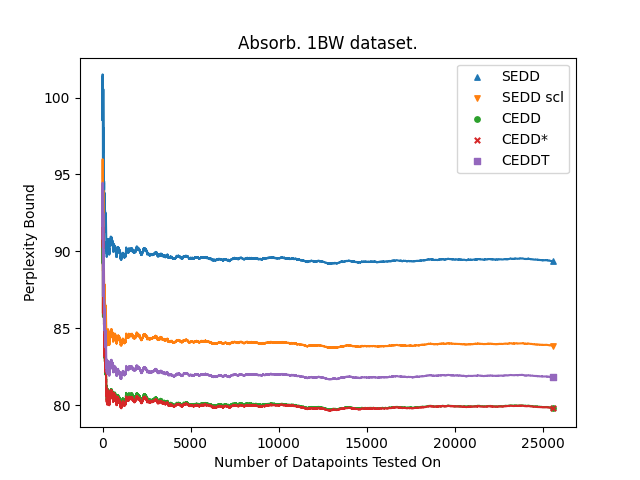}
\end{subfigure}%
\hfill%
\begin{subfigure}[b]{0.25\textwidth}
    \centering
    \includegraphics[scale=0.24]{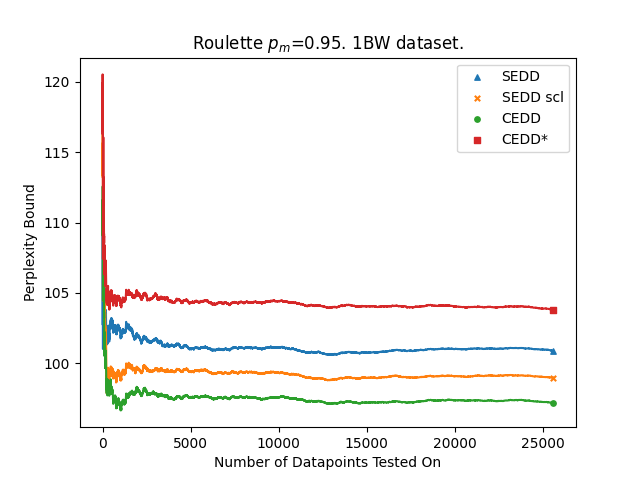}
\end{subfigure}%
\hfill%
\begin{subfigure}[b]{0.25\textwidth}
    \centering
    \includegraphics[scale=0.24]{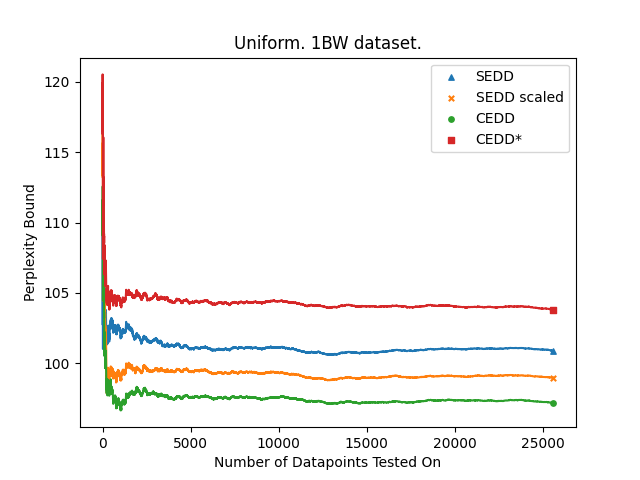}
\end{subfigure}%

\begin{subfigure}[b]{0.25\textwidth}
    \centering
    \includegraphics[scale=0.24]{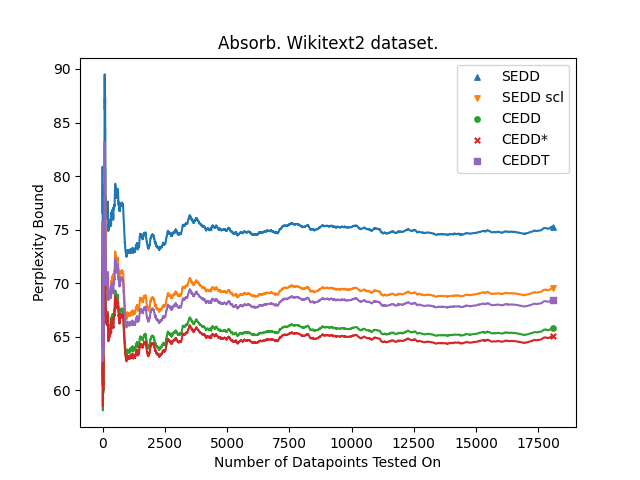}
\end{subfigure}%
\hfill%
\begin{subfigure}[b]{0.25\textwidth}
    \centering
    \includegraphics[scale=0.24]{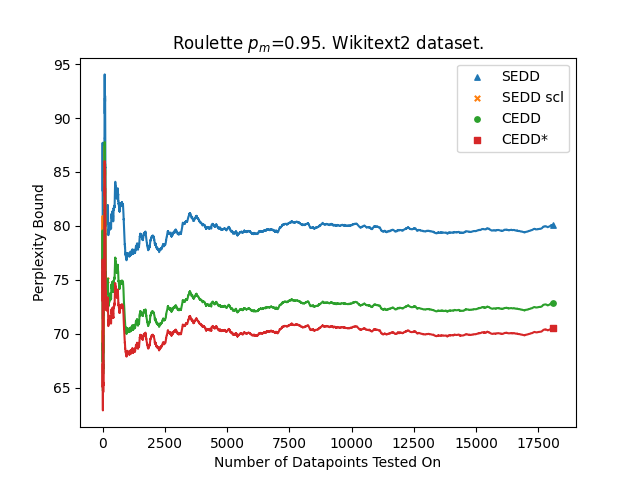}
\end{subfigure}%
\hfill%
\begin{subfigure}[b]{0.25\textwidth}
    \centering
    \includegraphics[scale=0.24]{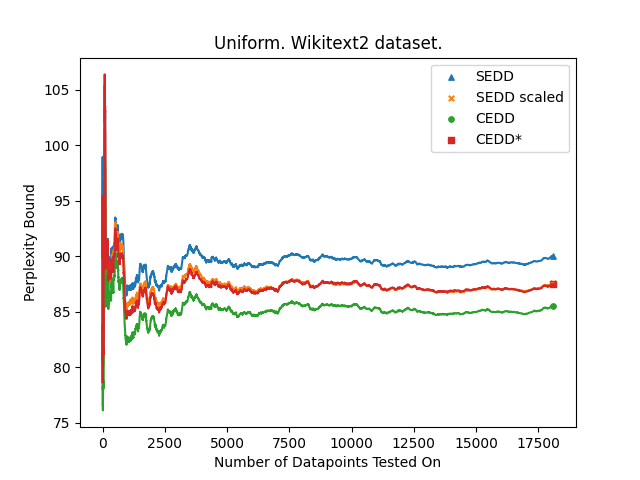}
\end{subfigure}%

\begin{subfigure}[b]{0.25\textwidth}
    \centering
    \includegraphics[scale=0.24]{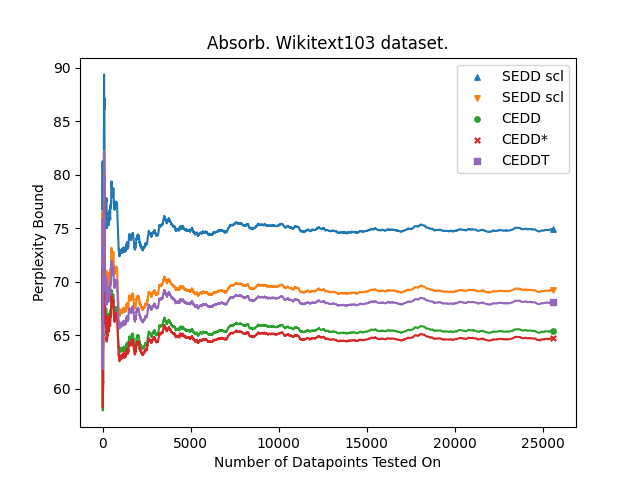}
\end{subfigure}%
\hfill%
\begin{subfigure}[b]{0.25\textwidth}
    \centering
    \includegraphics[scale=0.24]{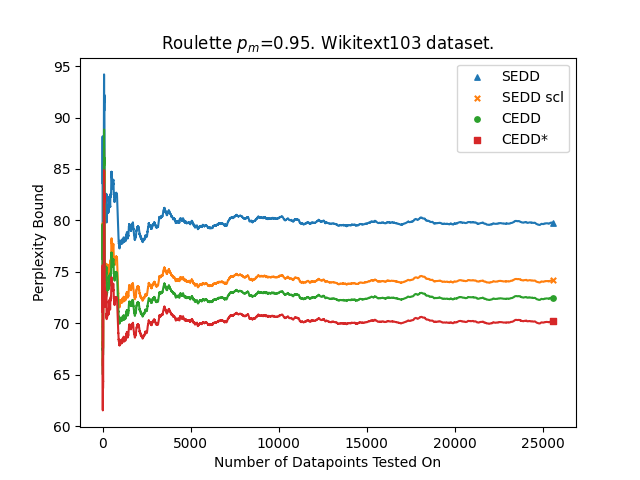}
\end{subfigure}%
\hfill%
\begin{subfigure}[b]{0.25\textwidth}
    \centering
    \includegraphics[scale=0.24]{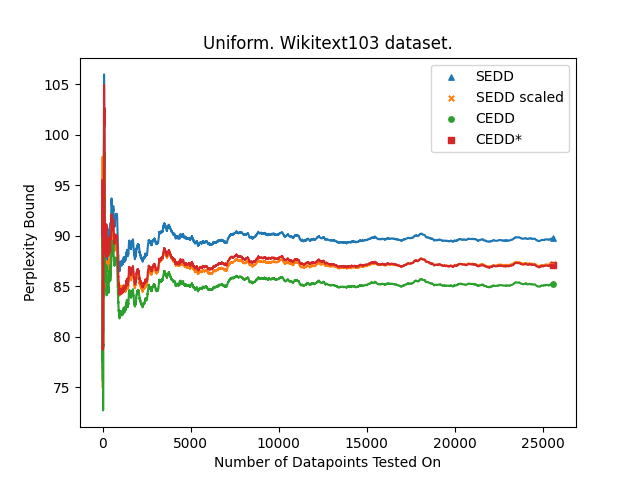}
\end{subfigure}%

\caption{The comparison of the testing history of various methods on various datasets. Each row represents a different dataset, and each column a different type of diffusion transition-rate matrix. The way we calculated $J_1$ to produce the plots is explained in Appendix \ref{Naive-J_1}. The curves presented are the exponentiated cummulative version of the array returned by the described process. Since we do not perturb the test set, this method enables a clearer visual comparison.}
\label{fig:grid5x3}
\end{figure}

\begin{table}[H]
\caption{Results comparing SEDD, SEDD scaled (SEDDs), CEDD and CEDD* using $J_2$. Lower is better.}
\label{small_perp_tab1_J2}
\begin{center}
\begin{tabular}{lllllll} 
\textbf{Model (L=128)} & \textbf{LAMBADA} & \textbf{WikiText2} & \textbf{PTB}  & \textbf{WikiText103} & \textbf{1BW}\\
\hline
SEDD Absorb & 69.33 & 74.58 & 238.35 & 74.27 & 88.35\\
SEDDs Absorb & 66.83 & 68.93 & 207.76 & 68.49 & 83.26\\
CEDD Absorb  & 64.87 & 65.07 & 198.00 & 64.97 & \textbf{78.93}\\
CEDD* Absorb & \textbf{64.11} & \textbf{64.54} & \textbf{191.38} & \textbf{64.30}& 79.17\\
\hline
SEDD Roulette & 71.37 & 79.21 & 227.98 & 78.67 & 92.21\\
SEDDs Roulette & 68.25 & 73.43 & 206.34 & 72.77 & 86.99\\
CEDD Roulette & 68.92 & 72.16 & 224.18 & 71.52 & \textbf{85.37}\\
CEDD* Roulette & \textbf{67.27} & \textbf{69.61} & \textbf{213.90} & \textbf{69.45}& 85.64\\
\hline
SEDD Uniform & 80.09 & 91.37 & 249.50 & 90.63 & 101.82\\
SEDDs Uniform & 80.29 & 88.48 & \textbf{226.03} & 87.73 & 99.73\\
CEDD Uniform & \textbf{79.46} &  \textbf{86.82} & 276.61 & \textbf{86.52} & \textbf{98.44}\\
CEDD* Uniform & 82.43 & 89.68 & 289.09 & 88.90 & 106.32\\
\hline

\end{tabular}
\end{center}
\end{table}

\begin{figure}[H]
    \centering
    \includegraphics[width=0.3\textwidth]{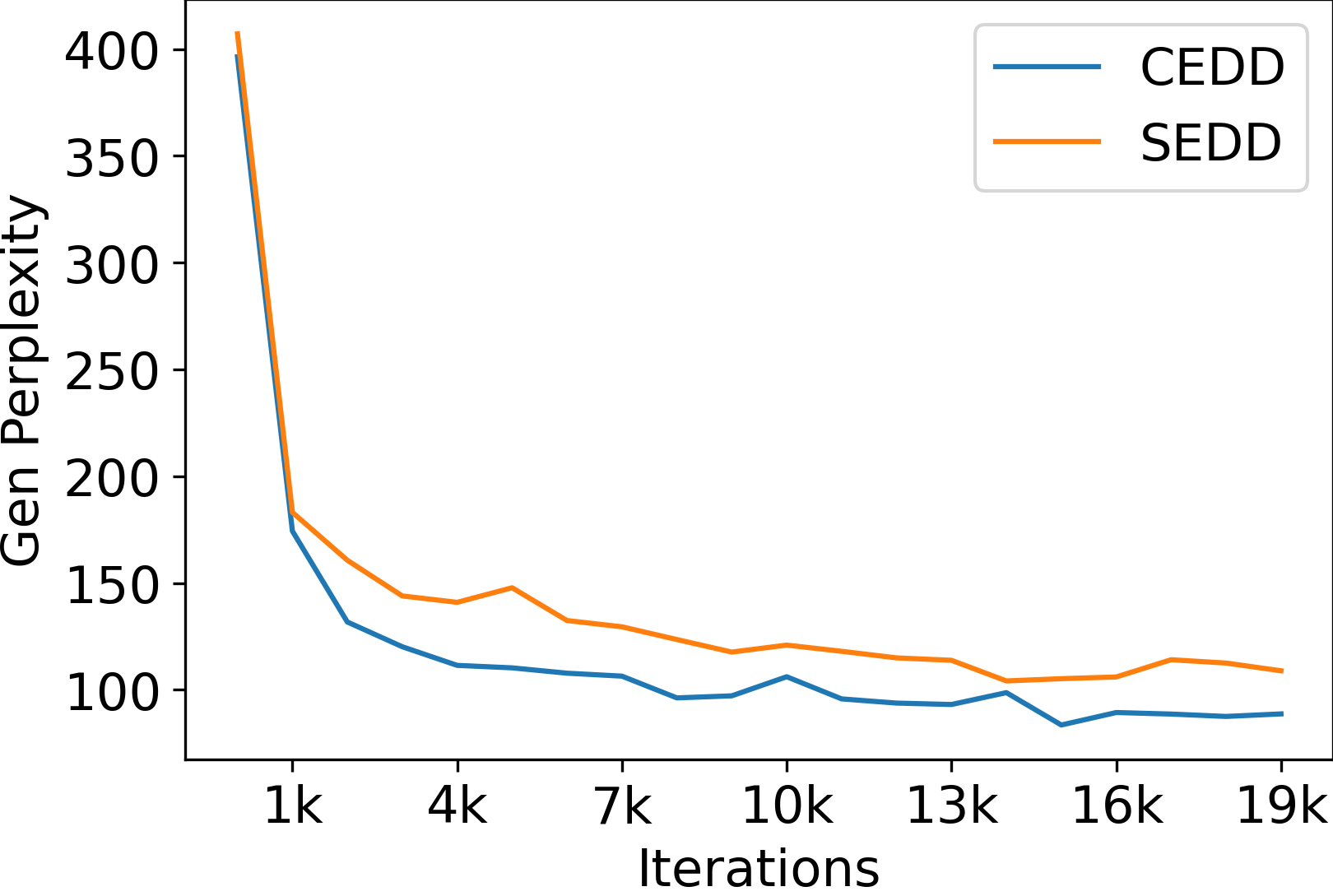}
    \caption{Comparison of CEDD* absorb L=1024 and SEDDs absorb 1024 in terms of generative perplexity (Euler, 128 steps), with respect to number of parameter updates.}
    \label{cvs20k}
\end{figure}

\begin{table}[H]
\caption{Results comparing SEDD (retrained), CEDD* trained for 20k parameter updates. For generation we use the analytic method with 1024 steps. Float 32 sampling.}
\label{seddvscedd_20k}
\begin{center}
\begin{tabular}{lllllll} 
\textbf{Model (Absorb)}  & \textbf{GenPerp} & \textbf{LAMBADA} & \textbf{WikiText2} & \textbf{PTB}  & \textbf{WikiText103} & \textbf{1BW}\\
\hline
SEDDs L=1024& 55.27 & 67.92 & 64.91 & 173.38 & 64.13 & 107.64\\
CEDD* L=1024& 48.29 & 63.75 & 57.58 & 157.71 & 57.26 & 98.68\\
\hline
\end{tabular}
\end{center}
\end{table}

\begin{table}[H]
\caption{Results comparing SEDD, SEDD scaled (SEDDs), CEDD and CEDD* in terms of generative perplexity.}
\begin{center}
\begin{tabular}{llll} 
\textbf{Model (L=128)} & \textbf{GenPerp} Fl32-GPT2L& \textbf{GenPerp} Fl64-GPT2L & \textbf{GenPerp} Fl64-LLama8B\\
\hline
SEDD Absorb & 83.62 & 172.35 & 212.15\\
SEDDs Absorb & 79.74 & 166.35 & 206.34\\
CEDD Absorb  & 74.19 & 148.21 & 185.90\\
CEDD* Absorb & \textbf{72.13} & \textbf{143.86} & \textbf{183.74}\\
\hline
SEDD Roulette & 87.81 & 178.94 & 220.00\\
SEDDs Roulette & 83.36 & 172.93 & 212.07\\
CEDD Roulette & 76.71 & 167.67 & 208.84\\
CEDD* Roulette & \textbf{72.31} & \textbf{158.56} & \textbf{197.79}\\
\hline
SEDD Uniform &  175.49  & 169.66 & 206.91\\
SEDDs Uniform & 171.07 & 163.88& 200.90\\
CEDD Uniform  & \textbf{168.35} & \textbf{161.84} &  \textbf{200.09}\\
CEDD* Uniform & 179.30 & 175.42 & 213.99\\
\hline
\end{tabular}
\end{center}
\end{table}

\begin{table}[H]
\caption{Results comparing SEDD, SEDD scaled (SEDDs), CEDD and CEDD* in terms of generative perplexity.}
\begin{center}
\begin{tabular}{lllllll} 
\textbf{Model (Absorb)} & \textbf{GenPerp} Fl32-GPT2L& \textbf{GenPerp} Fl64-GPT2L & \textbf{GenPerp} Fl64-LLama8B\\
\hline
SEDDs L=128 & 79.74 & 166.35 & 206.34\\
CEDD* L=128 & \textbf{72.13} & \textbf{143.86} & \textbf{183.74}\\
CEDDT L=128 & 74.07 & 154.04 & 195.41\\
\hline
SEDDs L=1024 & 40.95 & 105.27 & 111.87\\
CEDD* L=1024 & \textbf{40.93*} & \textbf{101.83} & \textbf{107.32}\\
CEDDT L=1024 & 42.18 & 108.88 & 115.60\\

\hline
GPT-2 L=1024& 41.02 & 41.02* & 50.25*\\
\hline
\end{tabular}
\end{center}
\end{table}

\begin{table}[H]
\caption{Correction accuracy percentages. 50k training iterations, batch size of 32 and $L=128$.}
\begin{center}
\begin{tabular}{lllll} 
\textbf{Model (L=128)} & CEDD* Uniform & SEDD Uniform & CEDD* Roulette & SEDD Roulette\\
\hline
PAP & 90.7 & 89.9 & $\boldsymbol{90.8}$ & 87.9\\
CAP & 91.2 & 90.3 & $\boldsymbol{91.3}$ & 88.6\\
\hline
\end{tabular}
\end{center}
\end{table}

\begin{table}[H]
\caption{Percentage of mistakes corrected for additional models. 25k training iterations, batch size of 32 and $L=128$.}
\begin{center}
\begin{tabular}{lll|lll} 
\textbf{Model (L=128)} & PAP & CAP & \textbf{Model (L=128)} & PAP & CAP\\
\hline
SEDD Roulette & 85.1 & 85.8 & SEDD Uniform & 86.9 & 87.5\\
SEDDs Roulette & 86.5 & 87.2 & SEDDs Uniform & 86.4 & 87.1\\
CEDD Roulette & 88.9 & 89.7 & CEDD Uniform & 88.5 & 89.2\\
CEDD* Roulette & $\boldsymbol{89.7}^*$ & $\boldsymbol{90.3}^*$ & CEDD* Uniform & $\boldsymbol{89.5}$ & $\boldsymbol{89.5}$\\
\hline
\end{tabular}
\end{center}
\end{table}

Finally, we provide results (Table \ref{DFMC}) when comparing CEDD* and Discrete Flow Matching that utilizes corrector sampling with the probability velocity $v_t^i(x^i, x_t) = \alpha_t \hat{u} - \beta_t \check{u}$, where $\check{u}(x^i, x_t) = \frac{\dot{k_t}}{k_t}(\delta_{x_t^i}(x^i) - \delta_m(x^i))$. We used $\alpha_t=1+\alpha t^a (1-t)^b$ for $\alpha=1$ and $a=b=0.5$. To compute the perplexity bound $e^{B}$ we utilized Equations (20) and (21) from \citet{haxholli2024minibatch}, and derived:

\begin{equation}
B = \frac{1}{L}\int_0^1 \sum_{x_t} p_{t|1}(x_t|x_1) \sum_{i=1}^L \Big[ -\delta_{x_t^i \neq x_1^i} \frac{\dot{k_t}}{1-k_t} \left(\log \alpha_t p_{1|t}^i(x_1^i|x_t; \theta) + 1\right) + 
\end{equation}
\begin{equation}
\alpha_t \frac{\dot{k_t}}{1-k_t} (1 - p_{1|t}^i(x_t^i|x_t; \theta)) + \beta_t \frac{\dot{k_t}}{k_t} \delta_{m \neq x_i}  \Big] dt.
\end{equation}
We used $k_t = t$ and noticed that by increasing $\alpha$, generative perplexity improves but the quality of samples is reduced, to the point of generated sequences becoming simple repetitions of symbols and rare tokens at $\alpha=10000$. This implies that $\alpha$ might modify the temperature of sampling. In addition, the measured perplexity bound quickly diverges to infinity, as $\alpha$ is increased.

\begin{table}[H]
\caption{Generative perplexity was computed using GPT2 large, and sampling was performed at Float64 precision.}
\begin{center}
\begin{tabular}{lllllll}\label{DFMC}
Model (L=128) & GenPerp & LAMBADA & WT2 & PTB & WT103 & 1BW \\
\hline
       SEDDs Absorb      & 166.35             & 67.05   & 69.37      & 208.69 & 69.17    & 83.87        \\
        CEDD* Absorb      & 143.86            & $\boldsymbol{64.60}$    & $\boldsymbol{65.04}$       & $\boldsymbol{192.99}$  & $\boldsymbol{64.69}$     & $\boldsymbol{79.81}$         \\
        Discrete flow $k_t = t$  & 145.48             & 71.90   & 71.20      & 221.15 & 70.84    & 82.63        \\
        Discrete flow $k_t = t^2$ & 152.70             & 72.31   & 72.87      & 215.30 & 72.55    & 85.82        \\
Discrete Flow Correct & $\boldsymbol{143.38}$ & 114.40   & 113.05   & 352.32 & 112.87         & 130.41 \\
\hline
\end{tabular}
\end{center}
\end{table}

\newpage
\section{Generated Examples}

\subsection{Generated text}\label{gentext_samples_appendix}
\begin{lstlisting}[caption={Generated texts from Absorb diffusion trained using CEDD, L=128.}]
Any version of the bundle allows corruption. The campaign was delayed a year from testing ground, and because of other deficiencies. Perhaps this can be used to inform all the Xbox One games come with the same package. Whilst there will be no problems, there will probably be no corruption price-wise and we can expect the same with one Xbox One game package at the same time. My opinion is a solid one. Let's assume Japan as the entry point in the entire country, and PS3 probably supports only one of these. It is effectively competing (see video) against hardware arcade titles such as PS4 or XPS4 as well
======================================================================
 his Achilles. I was ecstatic to be able to watch every one of them and keep them healthy tonight. You can look at the building blocks in college basketball, see that they play very hard and maybe feel superior to somebody. In my gut, however, you can argue that the one-year early good contracts come back and have a point. There has been a slight step-off with John Wall, which was a bit short on my mind. I now know Wall will be listed on the waiver, but I believe with Vic Beasley and Louren Maye this season, those two, with some really good shooters on the floor,
======================================================================
 times when it is vital for Scotland’s health and welfare The HRH Statement of Purpose and Development (SSTO) provides a framework for the government’s creation of the Future of Europe as a single body independent of Scotland’s national welfare system. The government’s own intention is to actively develop its vision for the UK’s welfare system based on its very nature and how the set of policies and key principals of the forum was established.

The Future of Europe Working Group (STO) Omissions of the UK Council on Europe’s Audit Office (œsivishig
======================================================================
 they go along with so many top players.

LAM: The lot of people in the parts they lose over are now than Jose Lopez, but if Affaar buys you deep and cheats Howard, Charo Colave, Bradley Donaldson, along with the rest of the opponents, we can see some of it.

LAM: With Santiago Ronaldo in South Florida. Vazal and Pedro Morales playing well, they don’t help us.

LAIDHUL: And you can actually tell which sides are being switched. Of three moving sides can be quite accurate.

A pawn,
\end{lstlisting}
\newpage
\begin{lstlisting}[caption={Generated texts from Uniform diffusion trained using CEDD, L=128.}]
 is to facilitate market conditions in jurisdictions. Also, not unlike that elsewhere where in terms of lack of in-house attorneys with no background qualified, or are looking to become in England and the whole country in the U.S., Arizona, New England, California, Connecticut, New Hampshire, New York, Massachusetts, Baltimore, Buffalo and Philadelphia, there are price measures set to rise to the model. And they come in handy as there are extraordinary levels of price in commodities, like alcohol and other drugs. Bankruptcycy approached, savings rate rigging now make even more dangerous. With broadest institutional licibility assets to sell and demand a
======================================================================
 aircraft, under the Mohammed Boland project.[28][30][31] In 2010, the government collaborated with Oquantasch states hoping to build the first pilots and aircraft carriers,[33][34] announced that on 16 October they decommissioned 25 acres of the Mozambique plantation leased by a Boeing 737, including several African pilots whom have been imparted on Brazilian soil and ended in. A military expedition – along with the arms embargo and tax evasion – has been delayed until federal Unabloprevalence halted part of the construction.[22]

Engineer units, and airport infrastructure [ edit ]

The
======================================================================
 scene. Though a woman called, saying that whether it was being on the suspect's video is not the answer, Desiling says an undisclosed number were present at the mall, and that the police were ready.

In the end, two shoppers, Florimore, 27, sprinted outside the convenience store after struggling with downpour to get inside the shop.

When Florimore left the restroom at the store, Zimmerman entered the vendor's cell phone on Sunday and secured her keys Saturday, Desiling said. He helped he lock a key before being used to contract it. After the taped encounter with a current Illinois Bureau of Police
======================================================================
, some of which were then scrutinized by the Bernanke Advisory Committee led by Michael Corm.

The U.K. Fed, eager.

Aboard all his most extreme concoctions careening the populace soon queued a shaky living net following weeks of the disastrous fiscal mess caused by the housing crash and the impending collapse in mortgage payments.

Can't anyone manage to usher it all the way. A savvy government could build “a serious living net — and that is much less than immediate relief.” Instead of achieving a 50 percent reduction in mortgages, anyone at the Fed’s table did it:
\end{lstlisting}
\newpage
\begin{lstlisting}[caption={Generated texts from Roulette (0.95) diffusion trained using CEDD, L=128.}]
 deeper problems caused by America’s growing deficit, and by our rapidly growing economy. The same was true today when the only way to undermine the nation’s welfare was in a virtuous cycle of austerity, stimulus and lavish budgets. And the scorching reality, especially for millions, was that “every man is the child of his parents’ mother”—not only abhorrent problem was the working age in Washington, which only saw 44% of the population working in the same household as network leaders and 79% of the workforce in their industry. We tapped these network leaders in polls on Election Day, causing favorable results
======================================================================
 other federal government agencies had far greater control over water industry regulation. Today, drilling wasn’t the equivalent of the Environmental Protection Agency championed since former president Ronald Reagan put Sen. Scott Desmond (CDE) on top of the Environmental Defense Agency, with a more funding area and rural programs under permit.

Like Moore, who spent so many a million dollars on his calls for the president to institute environmental protections on public lands and national parks, Trump himself had a nice shot on drilling, saying: “As passionate about the fight for property owners as we can be right here for Trump.”

Canada XL – No
======================================================================
Mistress: Next day next step, next step!" He's reply must not be translated. "What is - Guy's Your pronunciation? You still don't understand. Why am I saying that, don't mind: I've moved on to full set). Fitch word for move. Hey, that must have been long ago. What did you do going into our meeting? Everything changed, the rules changed... and you all admitted that Grand Auto Auto wasn't going to work, but you did. I need to know it. You've been in reference for the book, Notice."

Don't ignore him. He can't
======================================================================
 like low income earners, raising the tax on below.

All Part D credits forgiven by at least twenty-one percent paying income tax, and no state interest payments to help the health insurance guarantee. If owners pay the credits, they spend money they owe on each payment, either. A fifth of incomes will pay the credits.

While it's perfectly up to the type of hardship, people can apply state or federal assistance to the credit program. It requires people to sell health insurance. The credits qualify and vary from individual person to enrollee.

Having decided to pay their state income taxes on the bill, a single
\end{lstlisting}
\newpage
\begin{lstlisting}[caption={Generated texts from Absorb diffusion trained using CEDD*, L=128.}]
boat men: drift, and boom. Built, three 9th sailors found themselves losing an average of 12 inches of the underwater 50 feet in the delta base of S. M. Gielker Jr, U.S. captain (and pilot for the 11th Brigade), sailed deep into the canal on a watchdeck with his bow of thumb and lost at his anchor near shore.

Importantly, future "adillists should consider the alarming consequences that[t]here in the future the current rapprochement of the 9th would be had on the children of all areas of the coast," re-ferenc
======================================================================
the world’s private space program.” (This mission, in James’s words—rust is laced and revisited to the down-space fall and the commercial foundation—and eventually—software portable to what Lotus’s proprietary operating system does now—and also new hardware—serves, itself, as a collective organization rather than a fragmented, coordinated copal: “NASA loves the space community at large.”)

James retired as CEO on the company board in 2004. And longer, Planetary Resources’s judoided business model is intimately linked to his own organizational record and by
======================================================================
 engaged in hacking American government and stealing cyber defense secrets from its political enemies.

“We know today that the celebrity and his partner in the leaks incident to WikiLeaks produce the true story about our government, content and the data shared by our intelligence services, and deserve accountability,” said Mark Warner, Assistant Attorney General, at the key E.C.C conference on Wednesday in Washington, R.I. “In this case, the battle #WikiLeaks against WikiLeaks has touched both at a new level for the country, and finally has turning historic date.

“Our government has been without the power to hack
======================================================================
 non-committal after the third GOP debate.

Gia had already just commented on the upcoming status of her project on CNN, after Trump repeatedly asserted it would discuss politics for the first lady.

Trump said, "Look at this, on my part of Hillary Clinton, the mad country," Garcia continued.

"I took my position in a way President Trump didn't have before speaking :], he said it. I admitted some of those things, finally understanding that I stood for Oh Hillary as a businesswoman," she continued.

Garcia made the comment after training herself in Melania Trump's Daily
\end{lstlisting}
\newpage
\begin{lstlisting}[caption={Generated texts from Uniform diffusion trained using CEDD*, L=128.}]
 the Registration of the Bank of Development despite the skills of Arasan Sance. unemployable novice is in high regard for Swatch Up Worker as there are large scale projects bylaws that have already been difficult in the sector.

Caught diving too deeper, Arasan Sance had a career at IIT built on Year Up in Science at the Wood,, Theatre during an instructor class at the university in Banance in 2017. The guide says that some jobs may have been finalised and it was understood well that, if efficacious, McCarley could decide to be invited to a perhaps larger as academic class.

======================================================================
 strength of subsequent fitness genes from mitochondria after maintenance of. This type of transition is called exalted mitochondrial, and although extended to long a feat is indeed amolecule strategy that requires stringent tests in selective binders.

In particular, one that abundant, modified mitochondria can enable recovery; first they produce mitochondria which is attached to the fuel cell, is rapidly moving essentially back to normal form, transitioning from the atrophy to a more muscular phenotype. Mitria such as L1L protein A (C1A and transporter R (Anaris-in proteinase), as NAD is used as an adaptation strategy, like many
======================================================================
 how things were in the later years. The truck was littered with a black pepper. It was overflowing with halal green and mo yellow corn pepperoni and a charcoal plate.

Pictured used to have a baseballs encased in the front driveway with the sports hooders. Guy's black ores tip, black bean bag, silver, cherry, and granny banger was ricocheted, and when the Appomeess the Chinese bullfly rolled me over to the Morrissey. O 44, everything fresh and natural were just different from those in earlier, inscriptions. This guy's big white hooded or blue jacket
======================================================================
 idea of increasing momentum, however, is to slow down a quarterback’s strength. His is in the midst of a slow renaissance with the team’s ensuing 17 games, but it is aging one who is getting the effectiveness of missing Greg Jennings and injuries. That this crucial area comes as the 49ers failed progressual Mac is forcing his fades down instead of sticking fully — unless he is getting better.

The Jets is approaching the top-10 in front but, as soon as they fill it in, an underdog-Dish seems to benefit from this slight fatigue in their upcoming game.

Negvious Matt Miller
\end{lstlisting}
\newpage
\begin{lstlisting}[caption={Generated texts from Roulette (0.95) diffusion trained using CEDD*, L=128.}]
 such fraudulent statements and denials as, “the truth lies,” and “white lies” could never be trusted. More than one among you have just known Williams about.

Many are the fictional experiences Williams shaped, in ways that embodied him in a generation. The Tielemin figure in James Joyce’s white character largely was based on male gaze and information and, from Williams’ perspective, reflecting to the lens on men from a different time was disappointing.

Williams’s material images of how, in essence, we had occupied a not present world: beaten, destroyed, marginalized or
======================================================================
 an attempt to fully realize the potential for pedestrian safety where there is no dearticulate conduct for the traffic.

His own current plans, for which Lokxelli is the secretary of the planning drove changes with little progress through traditional detours as of late. “Not highways, but others are up quite a bit,”Akka forecasted “to make sure that one imp is not going to get torn. How to do that away from what it is becomes when much faster roads occur. There is all going to be of no serious consideration, at present, either in design or perspective.”


======================================================================
 mad over and watching him laugh loudly enough to say so, and my point is that it's not just that I am not displeased with a politician or a CEO, that I think people looked in their shoes and paid well."

To see countries that see relevance to the show's multimillionaires, including so frustrating the small smile in the face of the world, Thomas Archifelius applied, but wanted to represent the image of man keeping ignorant out his protectionist mission.

ARTICLE CONTINUES BELOW

GEORGE STEPHANRAD: In a few weeks, Papacuzzi and so did Paul
======================================================================
 and made fun of the tactic. They acknowledged was just silly. Indeed, shortly after the books, a real Western painter and fellow broach type, said, to another wood miller, “he’s straightened a nose. Skeptic got a hard chin and flies in the back.” The winced to me is that you’re a courtesan. You’ll know what's up for you<|endoftext|>Get it done. Your friends will never see it

Parents of kids and old owners dogs yesterday said they are going to miss president-elect Donald Trump's look, 'Loving' for
\end{lstlisting}
\newpage
\begin{lstlisting}[caption={Generated texts from Absorb diffusion trained using SEDD, L=128.}]
 likely to be the fact that Nephi calls God's beloved "The Synod Word," a difference when it translates to an accident.

As viewers know, an Andesad man and his wife will never meet in Nephi’s time God begins to connect with others, Le Burkert said.

Source: Bonuses de jeunee included<|endoftext|>The concept of attunational fit occurs from the presence of perceptual experiences on the cortex of the brain, as exposed recently by research by senior Prof. Linda Perriene.[1] Since personal dynamic patterns don't exist within human subjects, cortical memory is often cognitive
======================================================================
 will look like a target.

97) J.O. Do you think of as a White House national coordinator as a Clinton political operative? Kitty|O. I sort of're still up-the-table the White House. P.O. That person who refuses to have the same meek with a client with concerns/lives just pretty boycotted by pulling off in the relationship. Glad let's get that out of the way tonight and point out his ability and his job promise.

98) Next to the President - David Grimsnek | In at least 100 years George W. Bush has
======================================================================
 last month invested in Ria Louise, New York through a wealth of $7 billion this year.

Meanwhile, the committee is a part of Katzman & Girle Service Conference, the firm that licenses the annual facility. Frisco visits to CraftEarlier this year.

“We newg CraftyTags for philanthropic, recovery, trust, and use of our services, with immense support from federal government, Mossody said.

Under the partnership, Neoneys also is working on expansion as a partner among other not-for-profit resource providers, with an increasing amount of resources and core services covering
======================================================================
 at least one, but nothing qualified from any perspective they could be offering them.

“The response was from a very recent comment made in blue material by the public but it was decided ahead of time,” Smith writes of the public’s prurient. So we’ve found that a long and full response excludes the ones who didn’t believe the policy would meet the government’s standards for proof and hold.”

Sure enough, their apologists could do what they want on hard science. But that being said, they’d like some insights into scientific ethics and
\end{lstlisting}
\newpage
\begin{lstlisting}[caption={Generated texts from Uniform diffusion trained using SEDD, L=128.}]
 can generate a json of the library based on which one exists for the scenario instance.

<?php G require JetlinkViewer :: create ( ePub : 'Whu )' d. add(): c. app (). head () //Outputs from Sub_select from Google = str__ (http://www.google.com ); //Add it while that function starts down the path and results.Usingjson */ }

In order, Node handles repos(), several usable implementations in Mesh.js have been tested.

The JavaScript library actually has dire explanation for itls: it does not, but unlike
======================================================================
 using you or they use you taking people out to catch you or whatever. That was wonderful because of the ways I look. I had gone from anywhere from 20/22 here to 40 years old to 50/50. I also grew up feeling. 'True Blood Baby' has gave me a lot of body movements... that's got me right there and I used to pop their first transformation. That show."

My sister really grew up as a nairety boy. Thrones in a while? No, it's only sometimes, and since older people might even see us for 15 days or 30 days per night, how long they
======================================================================
, but the ban on kidnapping remains the only difference.

Seyer has closed the case in a court letter seeking a response.

For more on the arrest video visit your Confused TMZ app page.

Updated Sept. 24, 2013. Tribute to Lawyer R Harrison.<|endoftext|>Gils Dahman is starting his season with the Red Bulls Crystal Palace. The midfielder rounded out his fifth season in the United States, missing three league meetings at U-a.

You can watch the interview with the media below. You can see an image for the camera below.

Gils Dahman has started
======================================================================
 Girl”). The earnestness of the ensemble is random, cheerful though frustrated with both the actions of themselves and the different perspectives on their journey. This classic masterpiece—writing in sequels and his imagined form of movies—is the tale of cinema that still has had aesthetics and is able to grow. But it also sees the fall of the film genre—though very serious yet complex—so we see a live show that can also find places lost by leading the way to work around intersectionality, where the lineage of Reed and Dormedy (“Iron Man” though) is applicable—and that she’s starting
\end{lstlisting}
\newpage
\begin{lstlisting}[caption={Generated texts from Roulette (0.95) diffusion trained using SEDD, L=128.}]
 its own territory. All the European Union as well as the U.S. government, including the British East Company, were issued coins to the Federal Reserve and a few days were in the open to support a foreign buyer. President Lee had adopted law rejecting donations from the country’s tax code, preserving the spirit of the Constitution of the 1870 Act of self-rule in South America. In the caddition of the British banking system, the American Mint provided the first ammunition and smuggled guns to Europe. Nonetheless, a serious problem: the Union was secretly armed anti-communists with a formidable internal authority. The union government was
======================================================================
 to blame?

It’s full of wrongheaded sheen and bears expressions and carries a similar light ending. I’m not turned off on that Craig’s Edge of Truth NYC guide if I don’t see something wrong going on and do the right thing.

What’s different about another story that can somehow tell the tale unfolds in Chicago so well with the ppl being approved by the General Service.

Like an Independent. Right? All right.

What Happens?

A Licensed Reader keeps all relationships private. They have no precedents. By all things
======================================================================
 result it couldn't publish it.

No, my work is not the best way that all of us can be at home with the poll results. I acknowledge you believing it, and doing this poll is going to have to happen to make the difference.

But again, we are honestly happy with the results - let's see. We could not vote from all the experts - but that still could happen. What's your opinion on that debate?

The other most annoying - but the most upset is #MSNBC. In the Wall Street Journal, people called the polls slow Gore down, saying that even the ABC thought
======================================================================
, 400 times the base capacity there has been priced to four. And now another station can do the same for capacity with two other lines to 1 per cent of the power meter factory circuit which also costs to $56 million for cancellation.

The next debate is on the schedule, but may not happen in late 2017 without major delays.<|endoftext|>The French media reports that Chinese technology firm EtApple has today started claiming that it has all given back, in its software, a return to earlier this month’s SmartEasy payment processing that will simply payment it for most Chinese users. And based on recent developments, the same company is sending
\end{lstlisting}
\newpage
\begin{lstlisting}[caption={Generated text from Absorb diffusion trained using CEDD*, trained L=1024.}]
 two days later meeting with the Governmental Affairs Committee next Tuesday, replacing the committee's naps with an informal discussion.

In his Thursday hearing, Sessions said he had been "used" into criticizing the CIA.

"There has been some degree of transparency in looking at but that's not going to come cheap for himself or the American people, or really to undermine his independence, which I think, despite being in various positions, I've been very, very clear on the matter of that. I would be angry or very upset about that," he told senators at the meeting. "But I think as we go into Wednesday, and the United States Cabinet are finding ways to make very clear to citizens that they will respect that system."

Several high-level officials also called the appointment, saying it was "a bit of a quirk" since a revolving door of offices such as the CGA is already running rampant and has yet to hold any hearings.

"We've seen people make fairly much of their career life from the Justice Department out of this," said one GOP source. "And I don't think that choice is appropriate when that's a government entity that has given up a lot of accountability claims."

Speaking after noon another Republican member of the Senate Commerce Committee told. Todd Risch (R-Idaho) however, how he thought about Sessions, "and honestly, I didn't see a whole lot of discussion at one point suggesting a tone — a tone. There has been a lot of change since the removal of some five of these people ... I mean, I want people to find out that respect goes to the leadership of their department, and I'm not sure it's nice to move on down that precipice."<|endoftext|>When you try to put any pressure, you put more than the Pacers into you love watching. Before the Boston Celtics come to their games against Washington with.500 records on Tuesday night in Pittsburgh, you will likely watch a five-point game between Brad Stevens and Antetioun Gouden. Hollis have fourth-most competition in the game overall (49 points per game) but before games start their season of over 14.5 points per game has already passed.

That led Boston to an early 2-1 deficit.

No, not really. In addition to last season, however, the Celtics really struggled during the Wizards’ first year of the NBA. For only the first 24 minutes, Washington had 18 minutes their D-man shooting guard who averaged at least four points, and as he chipped in another 11 points on at least 10 fourth quarter quarters Boston was outscored in the paint on five four-point shots, including two-point turnover margin.

If I included Stevens and Horford scoring 18 points in the only fourth quarter quarter from the paint, it would have probably been 2-1, and in the second-half-first quarter at half-time it would have been a 9-2 game thanks to the disastrous fourth-quarter performances. Yes, the Celtics have been unable to pull the ball away from us in the paint over the fourth quarter, but their success was probably in spite of the bad 2s, otherwise, our defense was weak.
\end{lstlisting}
\newpage
\begin{lstlisting}[caption={Generated text from Absorb diffusion trained using SEDD, trained L=1024.}]
.

The principle of selling for a period of energy will make the most sense next year, as Tesla’s record rapid sales may be sluggish but demand lives on.

Tesla Motors Corp’s CTO Carl Conti said beginning a new conversation with Tesla about the future state of its workforce in 2014 may be a part of its strategy in which as the company races to reduce its price of luxury vehicle deliveries, the first units that get to market again in earnest are expected to begin to get “next level,” he said.

AP Photo

“I appreciate [the] ideas. These are not a fast break, but are ways to strengthen the margin at least knowing that if you get the right timeframe you make, its way ahead,” Musk told reporters at The New York Times over the past month.

“None of these people have looked toward [s]ourcing the cost or the increase of the vehicle, unless you see it as a downward spiral,” he said.

Although Musk and other executives point to one to two auto drivers and a net purchaser of gasoline or other, the company’s strategy is that a reduced price for vehicles will outstrip some of those buyers, Conti still said.

Conti said the price of Model S is projected to fall by at least a 33 percent decline rate in 2015, on average. Conti expects Tesla expects 58.5 million to 59.5 million in 2015 and also for an average 15 percent decline to rise to February.

For the Model S, by contrast, Tesla’s are expected to reach a write-down of about $500 million to 2,000,000 EVs in September and more than 700,000 currently, he said.

The first Model deliveries of the year is expected to nearly begin on March 22. That price is heavily influenced by dealer monitoring that includes a total inventory of 600,000, with 400,000 delivered next month, Conti said.

He noted Tesla expects its predicted upward trajectory to increase by 2.4 percent by 2015 as production in Australia and China accelerate to 4.6 percent next year, respectively.

Musk also noted in The Wall Street Journal that while Tesla expects more than three consecutive in production next year, it will production a total of $8 million in U.S. roll-backs in the fall.

He said he has settled on a key question about the car: “If a customer needs it and is willing to pay the rates that we can keep things moving, then we ought to consider adding incentives for it.”

“The price of Model S expectations for production is a driver’all,” Conti said. “When the profit is available, it means a car is not in its pocket. We are seeing this as a very good time after Jan. 31.”

\end{lstlisting}
\newpage
\begin{lstlisting}[caption={Generated text from GPT2, trained L=1024.}]
power well past the NPC versus NPC meter for each unit type. Their use is varied, and thus allows for a considerable amount of flexibility with The Banner Saga, but whether you prefer your parent's roots to stay part of the community -- with allies everyone can fight and help each other -- remains to be seen. For that, you can always pick the steam community tab (not required, that's how the game is named -- no one likes spamming themselves for news now, I know).

However, you agree to terms of service that you agree to disable your favorite kind of multiplayer games (by default) to prevent unwanted conflicts: factions are only listed once in a game and cannot be disbanded or changed just because of a non-friendlier fleet. So it's a pretty straightforward, fairly simple toggle to get used to.

Ugh, it sounds like the HUD might get progressively messier as the game ends.

Shields charging slowly: active shields last several minutes, recharge fully in seconds, and charge every second they charge.

Shields wearing'shield' else doubles their size, can clip high above their skull, and are invincible to all other shields around your ship that border the shield and can deploy wall-mounted shields. Other shields like this one can only be used to shield the next vessel. Hit someone on the Pustules to activate shield shields. In an expedited, non-lethal, coordinated confrontation, it is fairly easy to compare stats to each shield for a ship to use.

Shield minus shields compensation: up to five shields are held at rank 2 in your ship's shields. It's called 'delay,' but with a ton of fancy commands on screen, you don't have to know (or even see) their context or employment to make sure to compensate for some of the distortion inherent in that bit. Seconds before your shield starts charging (just seconds before it behaves like hiring another worker), the timer will start running indefinitely, which means shields can snap to the back of your ship if kept as low as possible. Avoid collisions caused by instantaneous cap.

Shield halves that match the length of the ship's hull: you can control the range between Shields, such as your Private Capital Signatures, where the ship's bulletproof liquor clock is set at half-space and your Heavy Construction Band Bump the Game's Game soundtrack to Fade let the corp count (if the corporation has sixty transports within one day of being prepped, it's probably already pretty close to ten minutes, but your mission selection loading speeds will be holding it back a little).

Surrounding explosions moving over the bow: it's deceivingly easy to line up targets like an Infallible Surface projection theater Explosion shield room high up in the sky: fire teams can run together and mean to demolish their own low-flying illions every once in a while, yet closely follow their targets like shadow warriors to the detriment of themselves. Covertly set to'square hit' for easier customization, Ubisoft's new ship weather HUD is the key to spreading effective damage throughout the entire ship using the standard meteorological updates.

Shaping your space: shape each bridge to create low-risk rockouts or pathogen/addage paths. This may be tedious, but it actually whizzes through by itself while simultaneously adding damage to giants mercenary ships.
\end{lstlisting}
\newpage
\section{CTMCs Preliminaries}\label{AppendixD}

\subsection{Discrete-Time Markov Chains Over Finite-State Spaces}
\label{appendix_gen_inst}

A discrete-time Markov Chain in a finite-state space is a stochastic process $X_1, X_2, \ldots, X_T$, where each state $X_t$ depends solely on the preceding state $X_{t-1}$. The states $X_t$ can take on any value from the set $\{1, 2, \ldots, S\}$, where $S$ denotes the total number of possible states, and $T$ represents the number of time steps. The probability of being in state $x$ at time $t$ is
\begin{equation}\label{appendix_first}
    p_t(X_t=x)= \sum_{y=1}^S p(X_t=x, X_{t-1}=y)=\sum_{y=1}^Sp_{t|t-1}(X_t=x|X_{t-1}=y)p_{t-1}(X_{t-1}=y).
\end{equation}
If we place all such probabilities $p_t(X_t=x)$ in a vector $\vs_t$ of shape $S\times 1$, such that $\vs_t(x)=p_t(X_t=x)$, then from above we can deduce that
\begin{equation}
    \vs_t= \mP\vs_{t-1},
\end{equation}
where $\mP(x, y) = p_{t|t-1}(X_t=x|X_{t-1}=y)$. Given an initial probability distribution $\vs_0$ over states, the equation above fully determines the evolution of the probability over states with respect to time. If it is known that the state at time $t-1$ is $y$, then $\vs_t$ is simply column $y$ of $\mP$. This implies that the sum of the elements of each column $y$ of $\mP$ is one.

\subsection{Continuous-Time Markov Chains Over Finite-State Spaces}
It is possible to define a stochastic process with the Markov property in finite-state spaces, for $t\in[0, T]$, \citep{anderson2012continuous}. As previously, we can define a discrete-time process, on time points $\{0, \epsilon, ..., T-\epsilon, T\}$, such that there is $\epsilon$ probability of activating the previous transition mechanism when progressing from time $t-\epsilon$ to $t$, otherwise we stay where we are with probability ($1-\epsilon$).
Removing the random variables to simplify notation, we have
\begin{equation}\label{appendix_transition_mechanism}
    p_t(x)= (1-\epsilon) p_{t-\epsilon}(x) +\epsilon \sum_{y=1}^Sp_{t|t-\epsilon}(x|y)p_{t-\epsilon}(y).
\end{equation}
We notice that when $\epsilon =1$ the equation above coincides with Equation (\ref{appendix_first}), and in addition as before we can write Equation (\ref{appendix_transition_mechanism}) in matrix form
\begin{equation}\label{appendix_disc_lin_ode}
    \vs_t= (1-\epsilon)\vs_{t-\epsilon} +\epsilon \mP\vs_{t-\epsilon} =\left(\mI+\epsilon(\mP-\mI)\right)\vs_{t-\epsilon}=\left(\mI+\epsilon\mQ\right)\vs_{t-\epsilon} \text{ , where } \mQ=\mP-\mI.
\end{equation}
 From Equation (\ref{appendix_disc_lin_ode}), we see that $\frac{\vs_t-\vs_{t-\epsilon}}{\epsilon}=\mQ\vs_{t-\epsilon}$, which when taking the limit $\epsilon\rightarrow 0$ becomes 
\begin{equation}\label{appendix_ode}
    \frac{d\vs_t}{dt}=\mQ\vs_t.
\end{equation}
Given an initial probability distribution $\vs_0$ over states, the equation above fully determines the evolution of the probability over states with respect to time. Indeed, the distribution over states at time $t$ is the solution of the linear ODE in (\ref{appendix_ode}): $\vs_t=e^{t\mQ}\vs_0$. Therefore, if it is known that the state at time $0$ is $j$, then $\vs_t$ is simply column $j$ of $e^{t\mQ}$. Naturally, one can rescale the time variable such that $t=\sigma(t)$, where $\sigma$ is monotonically increasing, $\sigma(0)=0$ and $\lim_{t\rightarrow 1}\sigma(t)=T$, so that for $\vz_0:=\vs_0$,  we get $e^{t\mQ}\vs_0=e^{\sigma(t)\mQ}\vz_0=\vz_t$, and 
\begin{equation}\label{appendix_proper_ode}
\frac{d\vz_t}{dt}=\sigma^{'}(t ) \mQ\vz_t=\mQ_t\vz_t \text{, where } \mQ_t=\sigma^{'}(t ) \mQ.
\end{equation}
 Matrices $\mQ_t=\sigma^{'}(t ) (\mP-\mI)$ clearly satisfy the properties of transition-rate matrices \citep{suhov2008probability}. Matrices $\mQ_t$ are chosen such that: a) the matrix exponential $e^{\sigma(t)\mQ}$ is easy to calculate, which is essential as $p_{t|0}(x|y)=e^{\sigma(t)\mQ}(x,y)$; and b) $\vz_1$ is an easy reference distribution to sample from \citep{austin2021structured, campbell2022continuous}.
 \newline
Finally, similar to diffusion processes in continuous spaces, the continuous-time Markov chain in Equation (\ref{appendix_proper_ode}) also admits a reverse process \citep{kelly1979reversibility, sun2022score}:
\begin{equation}\label{appendix_back_ode}
    \frac{d\vz_{1-t}}{dt}=\bar{\mQ}_{1-t}\vz_{1-t},
\end{equation}
where $\bar{\mQ}_t(x,y)={\mQ}_t(y,x)\frac{p_t(x)}{p_t(y)}$ for $x\neq y$, and $\bar{\mQ}_t(x,x)=-\sum_{y\neq x}\bar{\mQ}_t(y,x)$. Since we can easily sample from the reference distribution, the only unknowns preventing us from being able to run backwards are the ratios $\frac{p_t(x)}{p_t(y)}$ also known as concrete scores \citep{meng2022concrete}, which we desire to model using a neural network. Once such ratios are modeled we can generate samples from the learned data distribution $p_0^\theta$ by discretizing Equation (\ref{appendix_back_ode}) as follows:
\begin{equation}\label{appendix_gen_euler}
    p(x_{t-\epsilon} = y \mid x_t = x) = \delta_{x}(y) + \bar{\mQ}_t(y, x) \epsilon + O(\epsilon^2).
\end{equation}

\end{document}